%% file: main.tex
\documentclass{article}

\usepackage[margin=1in]{geometry}

\usepackage{natbib}

\usepackage[utf8]{inputenc} %
\usepackage[T1]{fontenc}    %
\usepackage{hyperref}       %
\usepackage{url}            %
\usepackage{booktabs}       %
\usepackage{amsfonts}       %
\usepackage{nicefrac}       %
\usepackage{microtype}      %
\usepackage{todonotes}

\usepackage{wrapfig}
\usepackage{subcaption}
\usepackage{amssymb,amsmath}
\usepackage{amsthm}
\usepackage{bbm}
\usepackage{tikz}
\usepackage[title]{appendix}
\usepackage{thmtools, thm-restate}
\usepackage{float}
\usepackage{enumitem}

\newtheorem{theorem}{Theorem}
\newtheorem{example}{Example}

\newtheorem{corollary}{Corollary}[theorem]

\newtheorem{proposition}{Proposition}
\newtheorem{lemma}{Lemma}
\newtheorem{definition}{Definition}

\definecolor{forestgreen}{rgb}{0.13, 0.55, 0.13}

\input{math_commands.tex}

\definecolor{myred}{rgb}{0.81, 0.06, 0.13}
\definecolor{myblue}{rgb}{0.15, 0.23, 0.89}

\usepackage{authblk}
\newcommand\CoAuthorMark{\footnotemark[\arabic{footnote}]} %

\title{Beyond variance reduction: Understanding \\the true impact of baselines on policy optimization}

\author[1]{Wesley Chung\thanks{Equal contribution.}}
\author[2,3]{Valentin Thomas\protect\CoAuthorMark}
\author[3,4]{Marlos C. Machado}
\author[1,2,3]{Nicolas Le Roux}

\affil[1]{\footnotesize Mila, McGill University}
\affil[2]{\footnotesize Mila, University of Montreal}
\affil[3]{\footnotesize Google Research, Brain Team, Montreal, Canada}
\affil[4]{\footnotesize Now at Google DeepMind, Edmonton, Canada}

\affil[ ]{\texttt {\{wesley.chung2,  vltn.thomas\}@gmail.com, marlosm@google.com, nicolas@le-roux.name}}

\date{}

\begin{document}

\maketitle

\begin{abstract}

Bandit and reinforcement learning (RL) problems can often be framed as optimization problems where the goal is to maximize average performance while having access only to stochastic estimates of the true gradient. Traditionally, stochastic optimization theory predicts that learning dynamics are governed by the curvature of the loss function and the noise of the gradient estimates. In this paper we demonstrate that this is not the case for bandit and RL problems. To allow our analysis to be interpreted in light of multi-step MDPs, we focus on techniques derived from stochastic optimization principles~(e.g., natural policy gradient and EXP3) and we show that some standard assumptions from optimization theory are violated in these problems. We present theoretical results showing that, at least for bandit problems, curvature and noise are not sufficient to explain the learning dynamics and that seemingly innocuous choices like the baseline can determine whether an algorithm converges. These theoretical findings match our empirical evaluation, which we extend to multi-state MDPs.
\end{abstract}

\section{Introduction}
In the standard multi-arm bandit setting~\citep{robbins1952some}, an agent needs to choose, at each timestep $t$, an arm $a_t \in \{1, ..., n\}$ to play, receiving a potentially stochastic reward $r_t$ with mean $\mu_{a_t}$. The goal of the agent is usually to maximize the total sum of rewards, $\sum_{i=1}^T r_t$, or to maximize the average performance at time $T$, $\E_{i\sim \pi} \mu_i$ with $\pi$ being the probability of the agent of drawing each arm~\citep{bubeck2012regret}. While the former measure is often used in the context of bandits,\footnote{The objective is usually presented as regret minimization.} $\E_{i\sim \pi} \mu_i$ is more common in the context of Markov Decision Processes (MDPs), which have multi-arm bandits as a special case.

In this paper we focus on techniques derived from stochastic optimization principles, such as EXP3~\citep{auer2002nonstochastic, seldin2013evaluation}. %
Despite the fact that they have higher regret in the non-adversarial setting than techniques explicitly tailored to minimize regret in bandit problems, like UCB~\citep{agrawal1995sample} or Thompson sampling~\citep{russo2017tutorial}, they naturally extend to the MDP setting, where they are known as \emph{policy gradient} methods.

We analyze the problem of learning to maximize the average reward, $J$, by gradient ascent:
\begin{align}
    \theta^\ast &= \arg\max_\theta J(\theta)
    = \arg\max_\theta \sum_{a} \pi_\theta(a) \mu_a \; , \label{eq:bandit_loss} 
\end{align} 
with $\mu_a$ being the average reward of arm $a$.
In this case, we are mainly interested in outputting an effective policy at the end of the optimization process, without explicitly considering the performance of intermediary policies.%

Optimization theory predicts that the convergence speed of stochastic gradient methods will be affected by the variance of the gradient estimates and by the geometry of the function $J$, represented by its curvature. Roughly speaking, the geometry dictates how effective true gradient ascent is at optimizing $J(\theta)$ while the variance can be viewed as a penalty, capturing how much slower the optimization process is by using noisy versions of this true gradient. More concretely, doing one gradient step with stepsize $\alpha$, using a stochastic estimate $g_t$ of the gradient, leads to \citep{bottou2018optimization}:
\begin{align*}\E[J(\theta_{t+1})] - J(\theta_t) &\geq (\alpha - \tfrac{L \alpha^2}{2}) \|\E [g_t] \|^2_2  - \tfrac{ L \alpha^2}{2} \text{Var}[g_t ],\end{align*}
when $J$ is $L$-smooth, i.e. its gradients are $L$-Lipschitz.

As large variance has been identified as an issue for policy gradient (PG) methods, many works have focused on reducing the noise of the updates. One common technique is the use of control variates~\citep{greensmith2004variance, hofmann2015variance}, referred to as \emph{baselines} in the context of RL. These baselines $b$ are subtracted from the observed returns to obtain shifted returns, $r(a_i) - b$, and do not change the expectation of the gradient. In MDPs, they are typically state-dependent.
While the value function is a common choice, previous work showed that the minimum-variance baseline for the REINFORCE \citep{williams1992simple} estimator is different and involves the norm of the gradient~\citep{peters2008reinforcement}.
Reducing variance has been the main motivation for many previous works on baselines~\citep[e.g.,][]{gu2016q, liu2017action, grathwohl2017backpropagation, wu2018variance, cheng2020trajectory}, but the influence of baselines on other aspects of the optimization process has hardly been studied. We take a deeper look at baselines and their effects on optimization.

\subsubsection*{Contributions}
We show that baselines can impact the optimization process 
beyond variance reduction and lead to qualitatively different learning curves, even when the variance of the gradients is the same. 
For instance, given two baselines with the same variance, the more negative baseline promotes \textit{committal} behaviour where a policy quickly tends towards a deterministic one, while the more positive baseline leads to \textit{non-committal} behaviour, where the policy retains higher entropy for a longer period.

Furthermore, we show that \textbf{the choice of baseline can even impact the convergence of natural policy gradient} (NPG), something variance cannot explain. In particular, we construct a three-armed bandit where using the baseline minimizing the variance can lead to convergence to a deterministic, sub-optimal policy for any positive stepsize, while another baseline, with larger variance, guarantees convergence to the optimal policy. As such a behaviour is impossible under the standard assumptions in optimization, this result shows how these assumptions may be violated in practice. It also provides a counterexample to the convergence of NPG algorithms in general, a popular variant with much faster convergence rates than vanilla PG when using the true gradient in tabular MDPs~\citep{agarwal2019optimality}. 

Further, \textbf{we identify on-policy sampling as a key factor to these convergence issues} as it induces a vicious cycle where making bad updates can lead to worse policies, in turn leading to worse updates. A natural solution is to break the dependency between the sampling distribution and the updates through off-policy sampling. We show that ensuring all actions are sampled with sufficiently large probability at each step is enough to guarantee convergence in probability. Note that this form of convergence is stronger than convergence of the expected iterates, a more common type of result \citep[e.g.,][]{mei2020global,agarwal2019optimality}.

We also perform an empirical evaluation on multi-step MDPs, showing that baselines have a similar impact in that setting. We observe \textbf{a significant impact on the empirical performance} of agents when using two different sets of baselines yielding the same variance, once again suggesting that learning dynamics in MDPs are governed by more than the curvature of the loss and the variance of the gradients.

\begin{figure*}[t]
\label{fig:npg_simplex}
\centering
  \begin{subfigure}[b]{0.24\linewidth}
    \includegraphics[width=\textwidth]{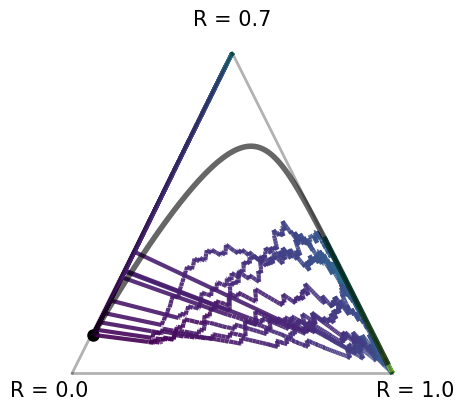}
    \caption{$b^-_\theta = b^*_\theta - \nicefrac{1}{2} $}
    \label{fig:-1}
  \end{subfigure}
  \hfill
    \begin{subfigure}[b]{0.24\linewidth}
    \includegraphics[width=\textwidth]{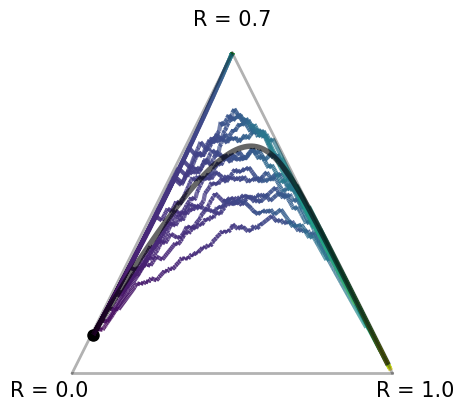}
    \caption{$b_\theta = b^*_\theta$}
    \label{fig:0}
  \end{subfigure}
  \hfill
  \begin{subfigure}[b]{0.24\linewidth}
    \includegraphics[width=\textwidth]{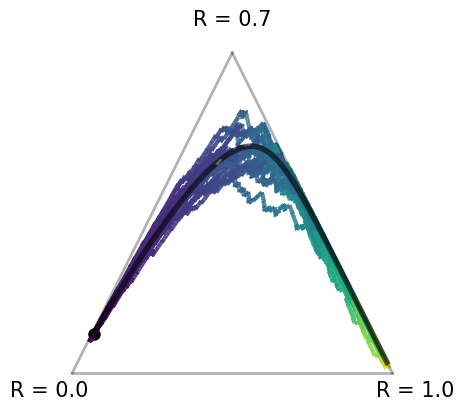}
    \caption{$b^+_\theta = b^*_\theta + \nicefrac{1}{2}$}
    \label{fig:1}
  \end{subfigure}
  \hfill
  \begin{subfigure}[b]{0.24\linewidth}
    \includegraphics[width=\textwidth]{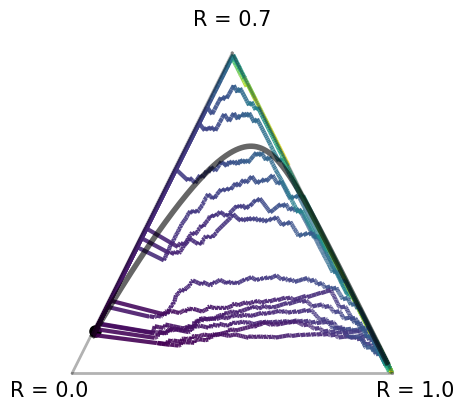}
    \caption{$b_\theta=V^{\pi_\theta}$}
    \label{fig:v}
  \end{subfigure}
  \caption{ \label{fig:trajectories}  We plot 15 different trajectories of natural policy gradient with softmax parameterization, when using various baselines, on a 3-arm bandit problem with rewards $(1,0.7,0)$ and stepsize $\alpha = 0.025$ and $\theta_0 = (0, 3, 5)$. The black dot is the initial policy and colors represent time, from purple to yellow. The black line is the trajectory when following the true gradient (which is unaffected by the baseline). Different values of $\epsilon$ denote different perturbations to the minimum-variance baseline. We see some cases of convergence to a suboptimal policy for both $\epsilon = -\nicefrac{1}{2}$ and $\epsilon = 0$. This does not happen for the larger baseline $\epsilon = \nicefrac{1}{2}$ or the value function as baseline. Figure made with Ternary~\citep{pythonternary}.}

\end{figure*}

\section{Baselines, learning dynamics \& exploration}
\label{sec:committal}

The problem defined in Eq.~\ref{eq:bandit_loss} can be solved by gradient ascent. Given access only to samples, the true gradient cannot generally be computed and the true update is replaced with a stochastic one, resulting in the following update:
\begin{align}
    \theta_{t+1} &= \theta_t + \frac{\alpha}{N} \sum_i r(a_i) \nabla_\theta \log \pi_\theta(a_i) \label{eq:mc_update} \; ,
\end{align}
where $a_i$ are actions drawn according to the agent's current policy $\pi_\theta$, $\alpha$ is the stepsize, and $N$, which can be 1, is the number of samples used to compute the update.
To reduce the variance of this estimate without introducing bias, we can introduce a baseline $b$, resulting in the gradient estimate $(r(a_i) - b) \nabla_\theta \log \pi_\theta(a_i)$.

While the choice of baseline is known to affect the variance, we show that baselines can also lead to qualitatively different behaviour of the optimization process, even when the variance is the same. This difference cannot be explained by the expectation or variance, quantities which govern the usual bounds for convergence rates~\citep{bottou2018optimization}.

\subsection{Committal and non-committal behaviours}

To provide a complete picture of the optimization process, we analyze the evolution of the policy during optimization.
We start in a simple setting, a deterministic three-armed bandit, where it is easier to produce informative visualizations. 

To eliminate variance as a potential confounding factor, we consider different baselines with the same variance. We start by computing the baseline leading to the minimum-variance of the gradients for the algorithm we use. For vanilla policy gradient, we have $b^*_\theta = \frac{\E[r(a_i) \|\nabla \log \pi_\theta (a_i)\|^2_2 ] }{\E[ \|\nabla \log \pi_\theta (a_i) \|^2_2] } $ ~\citep{peters2008reinforcement, greensmith2004variance} (see Appendix~\ref{app:optimal_baseline} for details and the NPG version). Note that this baseline depends on the current policy and changes throughout the optimization. As the variance is a quadratic function of the baseline, the two baselines $b^+_\theta = b^*_\theta + \epsilon$ and $b^-_\theta = b^*_\theta - \epsilon$ result in gradients with the same variance (see Appendix \ref{app:var_perturbed_baseline} for details). Thus, we use these two perturbed baselines to demonstrate that there are phenomena in the optimization process that variance~cannot~explain.

Fig.~\ref{fig:trajectories} presents fifteen learning curves on the probability simplex representing the space of possible policies for the three-arm bandit, when using NPG and a softmax parameterization. We choose $\epsilon = \nicefrac{1}{2}$ to obtain two baselines with the same variance: $b^+_\theta = b^*_\theta + \nicefrac{1}{2}$ and $b^-_\theta = b^*_\theta - \nicefrac{1}{2}$.

Inspecting the plots, the learning curves for $\epsilon = -\nicefrac{1}{2}$ and $\epsilon = \nicefrac{1}{2}$ are qualitatively different, even though the gradient estimates have the same variance. For $\epsilon = -\nicefrac{1}{2}$, the policies quickly reach a deterministic policy (i.e., a neighborhood of a corner of the probability simplex), which can be suboptimal, as indicated by the curves ending up at the policy choosing action 2. On the other hand, for $\epsilon=\nicefrac{1}{2}$, every learning curve ends up at the optimal policy, although the convergence might be slower. The learning curves also do not deviate much from the curve for the true gradient. 
Again, these differences cannot be explained by the variance since the baselines result in identical variances.

Additionally, for $b_\theta=b^*_\theta$, the learning curves spread out further. Compared to $\epsilon=\nicefrac{1}{2}$, some get closer to the top corner of the simplex, leading to convergence to a suboptimal solution, suggesting that the minimum-variance baseline may be worse than other, larger baselines. In the next section, we theoretically substantiate this and show that, for NPG, it is possible to converge to a suboptimal policy with the minimum-variance baseline; but there are larger baselines that guarantee convergence to an optimal~policy.

We look at the update rules to explain these different behaviours. When using a baseline $b$ with NPG, sampling $a_i$ results in the update
\begin{align*}
    \theta_{t+1} &= \theta_t + \alpha  [r(a_i)-b] F_\theta^{-1} \nabla_\theta \log \pi_\theta(a_i) \label{eq:npg_bandit_update} \\
    &= \theta_t + \alpha \frac{r(a_i)-b}{\pi_\theta(a_i)} \mathbbm{1}_{a_i} + \alpha\lambda e
\end{align*}
where $F_\theta^{-1} = \E_{a \sim \pi} [\nabla \log \pi_\theta (a) \nabla \log \pi_\theta (a)^\top]$, $\mathbbm{1}_{a_i}$ is a one-hot vector with $1$ at index $i$, and $\lambda e$ is a vector containing $\lambda$ in each entry.  The second line follows for the softmax policy (see Appendix \ref{app:npg_softmax_bandit}) and $\lambda$ is arbitrary since shifting $\theta$ by a constant does not change the policy.

Thus, supposing we sample action $a_i$, if $r(a_i) - b$ is positive, which happens more often when the baseline $b$ is small (more negative), the update rule will increase the probability $\pi_\theta(a_i)$. This leads to an increase in the probability of taking the actions the agent took before, regardless of their quality (see Fig.\ref{fig:-1} for $\epsilon=-\nicefrac{1}{2}$). Because the agent is likely to choose the same actions again, we call this \textit{committal} behaviour.

While a smaller baseline leads to committal behaviour, a larger (more positive) baseline makes the agent second-guess itself. If $r(a_i) - b$ is negative, which happens more often when $b$ is large, the parameter update decreases the probability $\pi_\theta(a_i)$ of the sampled action $a_i$, reducing the probability the agent will re-take the actions it just took, while increasing the probability of other actions. This might slow down convergence but it also makes it harder for the agent to get stuck. This is reflected in the $\epsilon = \nicefrac{1}{2}$ case (Fig.\ref{fig:1}), as all the learning curves end up at the optimal policy. We call this \textit{non-committal} behaviour.

While the previous experiments used perturbed variants of the minimum-variance baseline to control for the variance, this baseline would usually be infeasible to compute in more complex MDPs. Instead, a more typical choice of baseline would be the value function~\citep[Ch.~13]{sutton18book}, which we evaluate in Fig. \ref{fig:v}.
Choosing the value function as a baseline generated trajectories converging to the optimal policy, even though their convergence may be slow, despite it not being the minimum variance baseline.
The reason becomes clearer when we write the value function as $\displaystyle V^\pi = b^\ast_\theta - \tfrac{\textrm{Cov}(r, \|\nabla \log \pi\|^2)}{\E[\|\nabla \log \pi\|^2]}$ (see Appendix \ref{app:value_minvar_baseline}). The term $\textrm{Cov}(r, \|\nabla \log \pi\|^2)$ typically becomes negative as the gradient becomes smaller on actions with high rewards during the optimization process, leading to the value function being an optimistic baseline, justifying a choice often made by practitioners. 

Additional empirical results can be found in Appendix \ref{app:exp_3armbandit} for natural policy gradient and vanilla policy gradient for the softmax parameterization. Furthermore, we explore the use of projected stochastic gradient ascent and directly optimizing the policy probabilities $\pi_\theta(a)$. We find qualitatively similar results in all three cases; baselines can induce \textit{committal} and \textit{non-committal} behaviour. 

\section{Convergence to suboptimal policies with natural policy gradient (NPG)}

We empirically showed that PG algorithms can reach suboptimal policies and that the choice of baseline can affect the likelihood of this occurring. 
In this section, we provide theoretical results proving that it is indeed possible to converge to a suboptimal policy when using NPG. 
We discuss how this finding fits with existing convergence results and why standard assumptions are not satisfied in this setting. 

\subsection{A simple example}\label{sec:divergence_example}

Standard convergence results assume access to the true gradient~\citep[e.g.,][]{agarwal2019optimality} or, in the stochastic case, assume that the variance of the updates is uniformly bounded for all parameter values~\citep[e.g.,][]{bottou2018optimization}.
These assumptions are in fact quite strong and are violated in a simple two-arm bandit problem with fixed rewards. 
Pulling the optimal arm gives a reward of $r_1 = +1$, while pulling the suboptimal arm leads to a reward of $r_0 = 0$. We use the sigmoid parameterization and call $p_t = \sigma(\theta_t)$ the probability of sampling the optimal arm at time $t$.

Our stochastic estimator of the natural gradient is %
\begin{align*}
g_t = \left\{
                \begin{array}{l}
                  \frac{1 - b}{p_t}, \text{with probability}\ p_t \\
                  \frac{b}{1-p_t}, \text{with probability}\ 1-p_t,
                \end{array}
              \right.
\end{align*}
where $b$ is a baseline that does not depend on the action sampled at time $t$ but may depend on $\theta_t$. 
By computing the variance of the updates, $\textrm{Var}[g_t] =   \tfrac{(1-p_t - b)^2}{p_t (1-p_t)}$, we notice it is unbounded when the policy becomes deterministic, i.e. $p_t \to 0$ or $p_t \to 1$, violating the assumption of uniformly bounded variance, unless $b = 1-p_t$, which is the optimal baseline.
Note that using vanilla (non-natural) PG would, on the contrary, yield a bounded variance. In fact, we prove a convergence result in its favour in Appendix \ref{app:theory_2arm} (Prop. \ref{prop_vpg_cv}).

For NPG, the proposition below establishes potential convergence to a suboptimal arm and we demonstrate this~empirically~in~Fig.~\ref{fig:divergence_2arm_bandit}.

\begin{restatable}[]{proposition}{divtwoarms}
Consider a two-arm bandit with rewards $1$ and $0$ for the optimal and suboptimal arms, respectively. Suppose we use natural policy gradient starting from $\theta_0$, with a fixed baseline $b < 0$, and fixed stepsize $\alpha > 0$. If the policy samples the optimal action with probability $\sigma(\theta)$, then the probability of picking the suboptimal action forever and having $\theta_t$ go to $-\infty$ is strictly positive.
Additionally, if $\theta_0 \le 0$, we have
\[ P(\textup{suboptimal action forever}) \ge (1-e^{\theta_0}) (1-e^{\theta_0 + \alpha b})^{-\frac{1}{\alpha b}}. \]
\label{proposition_divergence}
\end{restatable}
\vspace{-0.8cm}
\begin{proof}
All the proofs may be found in the appendix.
\end{proof}

\begin{figure}[t!]
\centering
\begin{subfigure}[b]{0.32\linewidth}   
    \includegraphics[width=\textwidth]{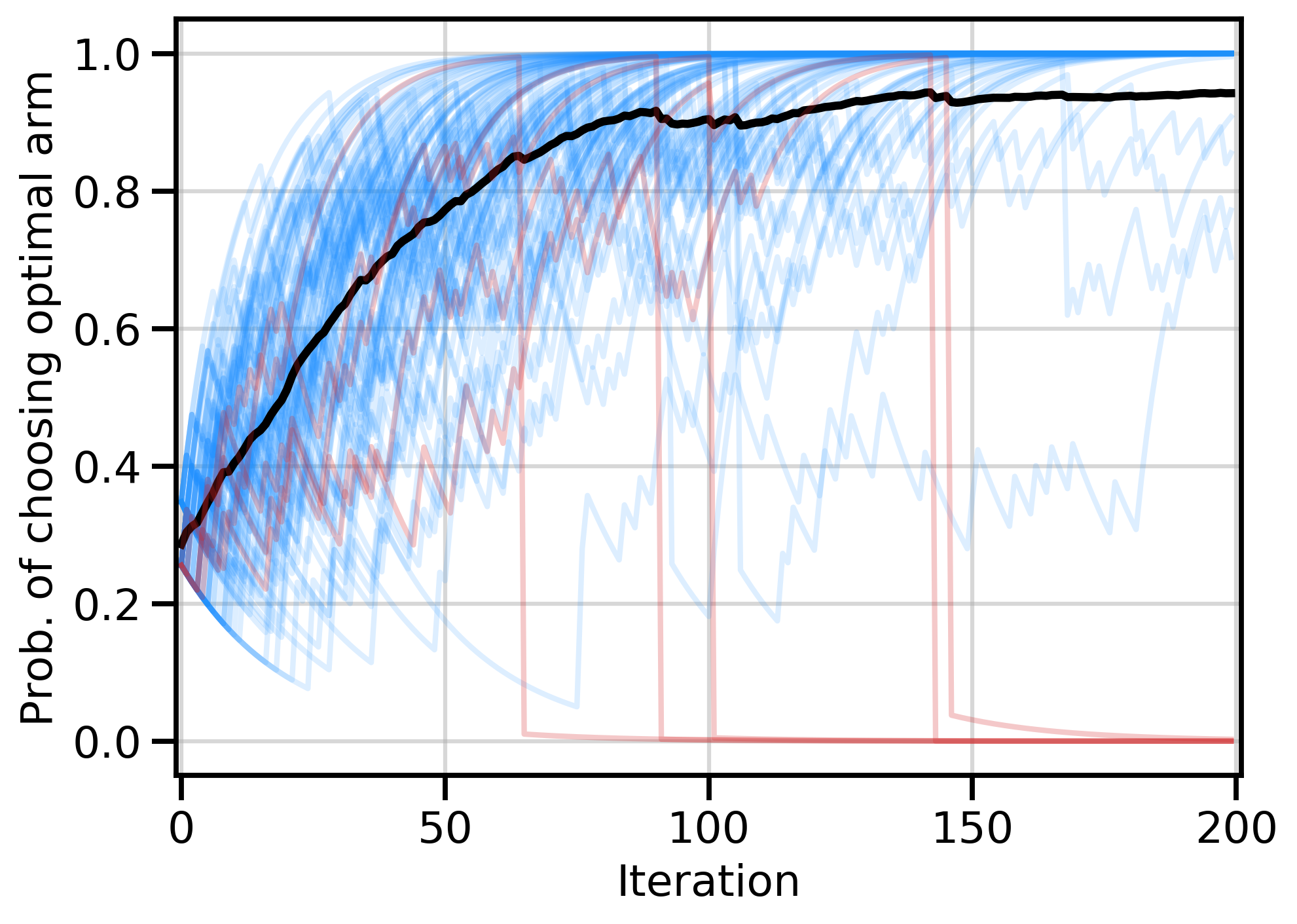}
    \caption{$\alpha=0.05$}
\end{subfigure}
\begin{subfigure}[b]{0.32\linewidth}   
    \includegraphics[width=\textwidth]{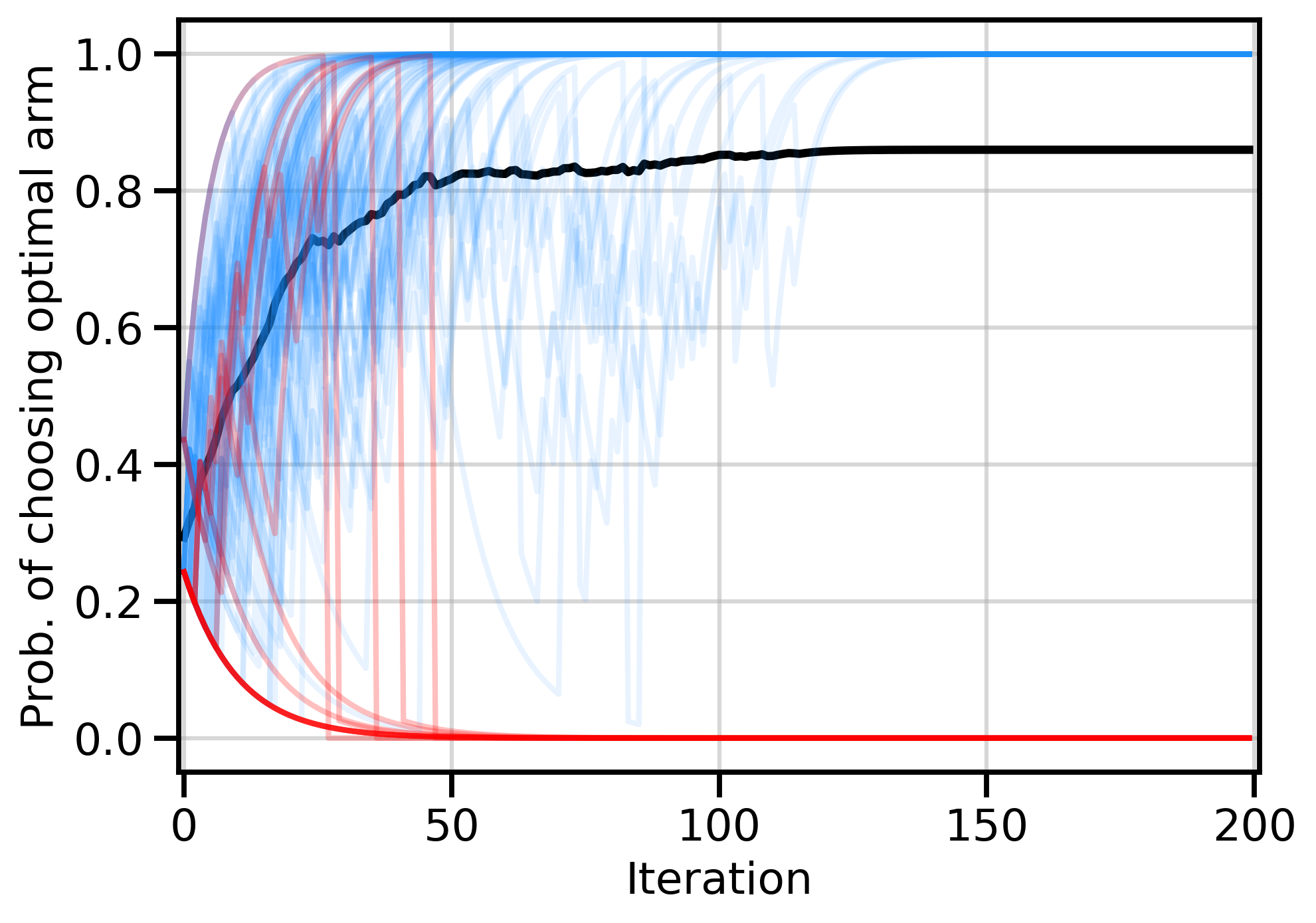}
    \caption{$\alpha=0.1$}
\end{subfigure}
\begin{subfigure}[b]{0.32\linewidth}   
    \includegraphics[width=\textwidth]{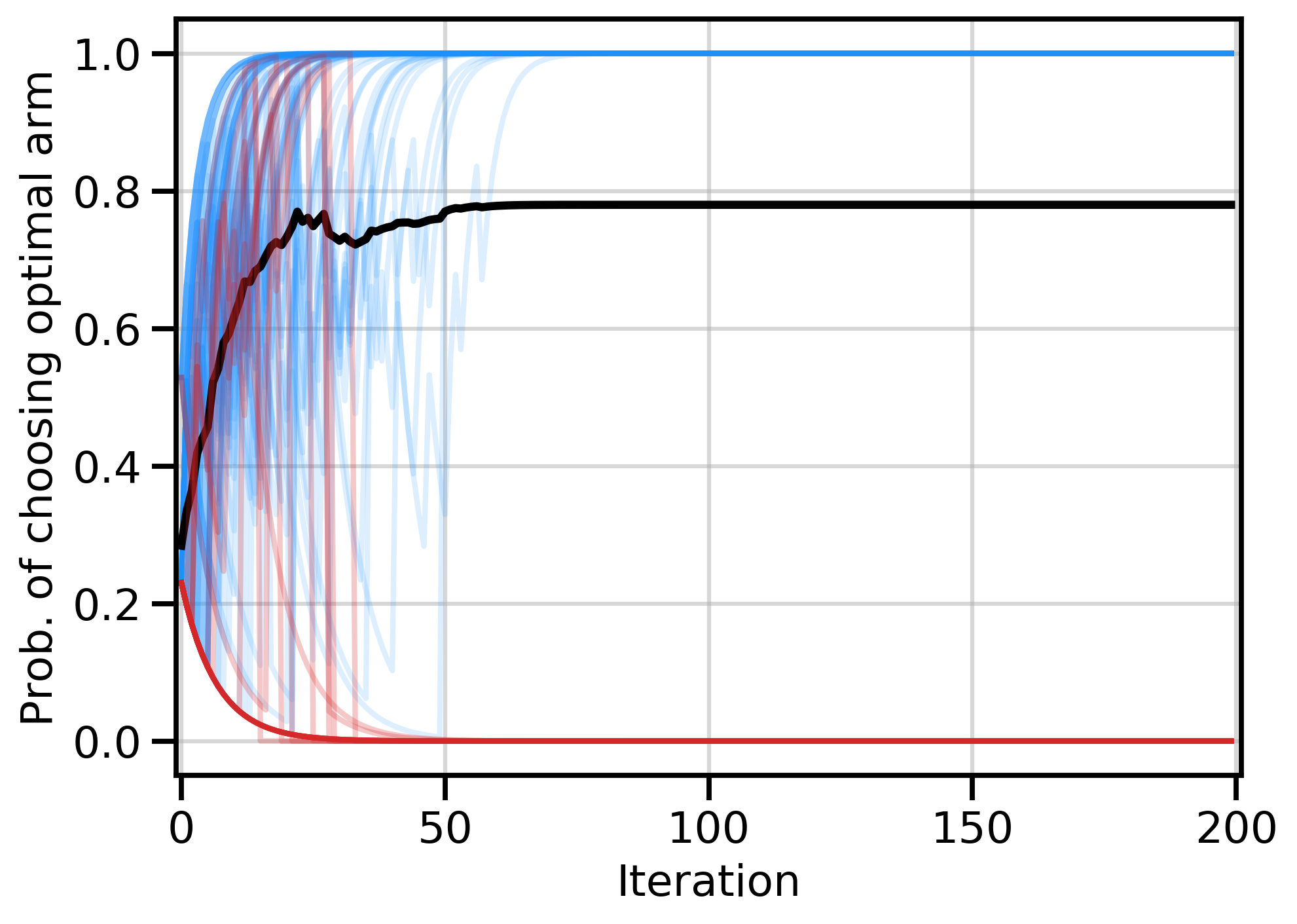}
    \caption{$\alpha=0.15$}
\end{subfigure}

\caption{Learning curves for 100 runs of 200 steps, on the two-arm bandit, with baseline $b\!=\!-1$ for three different stepsizes $\alpha$. \emph{Blue:} Curves converging to the optimal policy. \emph{Red:} Curves  converging to a suboptimal policy. \emph{Black:} Avg. performance. 
The number of runs that converged to the suboptimal solution are 5\%, 14\% and 22\% for the three $\alpha$'s. Larger $\alpha$'s are more prone to getting stuck at a suboptimal solution but settle on a deterministic policy more quickly.}
\label{fig:divergence_2arm_bandit}
\end{figure}

The updates provide some intuition as to why there is convergence to suboptimal policies. The issue is the \textit{committal} nature of the baseline. Choosing an action leads to an increase of that action's probability, even if it is a poor choice. Choosing the suboptimal arm leads to a decrease in $\theta$ by $\tfrac{\alpha b}{1-p_t}$, thus increasing the probability the same arm is drawn again and further decreasing $\theta$. By checking the probability of this occurring forever, $P(\text{suboptimal arm forever}) = \prod_{t=1}^\infty (1-p_t)$, we show that $1-p_t$ converges quickly enough to 1 that the infinite product is nonzero, showing it is possible to get trapped choosing the wrong arm forever~(Prop. \ref{proposition_divergence}), and $\theta_t \to -\infty$ as $t$ grows.

This issue could be solved by picking a baseline with lower variance. For instance, the minimum-variance baseline $b = 1 - p_t$ leads to $0$ variance and both possible updates are equal to $+\alpha$, guaranteeing that $\theta \to +\infty$, thus convergence. In fact, any baseline $b \in (0,1)$ suffices since both updates are positive and greater than $\alpha \min(b, 1-b)$. However, this is not always the case, as we show in the next section.

To decouple the impact of the variance with that of the committal nature of the baseline, Prop.~\ref{prop:2armed-perturbedminvar} analyzes the learning dynamics in the two-arm bandit case for perturbations of the optimal baseline, i.e. we study baselines of the form $b = b^* + \epsilon$ and show how $\epsilon$, and particularly its sign, affects learning. Note that, because the variance is a quadratic function with its minimum in $b^*$, both $+\epsilon$ and $-\epsilon$ have the same variance. Our findings can be summarized as follows:
\begin{proposition}
    \label{prop:2armed-perturbedminvar}
    For the two-armed bandit defined in Prop.~\ref{proposition_divergence}, when using a perturbed min-variance baseline $b = b^* + \epsilon$, the value of $\epsilon$ determines the learning dynamics as follows:
    \begin{itemize}[noitemsep, topsep=0pt]
        \item For $\epsilon < -1$, there is a positive probability of converging to the suboptimal arm.
        \item For $\epsilon \in (-1, 1)$, we have convergence in probability to the optimal policy.
        \item For $\epsilon \ge 1$, the supremum of the iterates goes to $+\infty$ in probability.
    \end{itemize}
\end{proposition}

While the proofs can be found in Appendix~\ref{sec:appendix_perturbed_minvar}, we provide here some intuition behind these results.

For $\epsilon < -1$, we reuse the same argument as for $b<0$ in Prop.~\ref{proposition_divergence}. The probability of drawing the correct arm can decrease quickly enough to lead to convergence to the suboptimal arm.

For $\epsilon \in (-1,1)$, the probability of drawing the correct arm cannot decrease too fast. Hence, although the updates, as well as the variance of the gradient estimate, are potentially unbounded, we still have convergence to the optimal solution in probability.

Finally, for $\eps \ge 1$, we can reuse an intermediate argument from the $\epsilon \in (0,1)$ case to argue that for any threshold $C$, the parameter will eventually exceed that threshold. For $\epsilon \in (0,1)$, once a certain threshold is crossed, the policy is guaranteed to improve at each step. However, with a large positive perturbation, updates are larger and we lose this additional guarantee, leading to the weaker result.

We want to emphasize that not only we get provably different dynamics for $\epsilon < -1$ and $\eps \ge 1$, showing the importance of the sign of the perturbation, but that there also is a sharp transition around $|\epsilon| = 1$, which cannot be captured solely by the variance.

\subsection{Reducing variance with baselines can be detrimental}
\label{sec:reducing_detrimental}
As we saw with the two-armed bandit, the direction of the updates is important in assessing convergence. More specifically, problems can arise when the choice of baseline induces committal behaviour.
We now show a different bandit setting where committal behaviour happens even when using the minimum-variance baseline, thus leading to convergence to a suboptimal policy. Furthermore, we design a better baseline which ensures all updates move the parameters towards the optimal policy. This cements the idea that the quality of parameter updates must not be analyzed in terms of variance but rather in terms of the probability of going in a bad direction, since a baseline that induces higher variance leads to convergence while the minimum-variance baseline does not.
The following theorem summarizes this.

\begin{restatable}[]{theorem}{threearmedbandit}
There exists a three-arm bandit where using the stochastic natural gradient on a softmax-parameterized policy with the minimum-variance baseline can lead to convergence to a suboptimal policy with probability $\rho > 0$, and there is a different baseline (with larger variance) which results in convergence to the optimal policy with probability~1.
\label{proposition_threearmedbandit}
\end{restatable}
The bandit used in this theorem is the one we used for the experiments depicted in Fig.~\ref{fig:trajectories}.
The key is that the minimum-variance baseline can be lower than the second best reward; so pulling the second arm will increase its probability and induce committal behaviour.
This can cause the agent to prematurely commit to the second arm and converge to the wrong policy.
On the other hand, using any baseline whose value is between the optimal reward and the second best reward, which we term a \textit{gap} baseline, will always increase the probability of the optimal action at every step, no matter which arm is drawn. Since the updates are sufficiently large at every step, this is enough to ensure convergence with probability 1, despite the higher variance compared to the minimum variance baseline.
The key is that whether a baseline underestimates or overestimates the second best reward can affect the algorithm convergence and this is more critical than the resulting variance of the gradient estimates.

As such, more than lower variance, good baselines are those that can assign positive effective returns to the good trajectories and negative effective returns to the others. These results cast doubt on whether finding baselines which minimize variance is a meaningful goal to pursue. The baseline can affect optimization in subtle ways, beyond variance, and further study is needed to identify the true causes of some improved empirical results observed in previous works. This importance of the sign of the returns, rather than their exact value, echoes with the cross-entropy method~\citep{de2005tutorial}, which maximizes the probability of the trajectories with the largest returns, regardless of their actual value.

\section{Off-policy sampling} 

So far, we have seen that \textit{committal} behaviour can be problematic as it can cause convergence to a suboptimal policy. This can be especially problematic when the agent follows a near-deterministic policy as it is unlikely to receive different samples which would move the policy away from the closest deterministic one, regardless of the quality of that policy.

Up to this point, we assumed that actions were sampled according to the current policy, a setting known as \emph{on-policy}. This setting couples the updates and the policy and is a root cause of the \textit{committal} behaviour: the update at the current step changes the policy, which affects the distribution of rewards obtained and hence the next updates. 
However, we know from the optimization literature that bounding the variance of the updates will lead to convergence \citep{bottou2018optimization}. As the variance becomes unbounded when the probability of drawing some actions goes to 0, a natural solution to avoid these issues is to sample actions from a behaviour policy that selects every action with sufficiently high probability.
Such a policy would make it impossible to choose the same, suboptimal action forever. 

\subsection{Convergence guarantees with IS}

Because the behaviour policy changed, we introduce importance sampling (IS) corrections to preserve the unbiased updates~\citep{kahn1951estimation, precup2000eligibility}. 
These changes are sufficient to guarantee convergence for any baseline:

\begin{restatable}[]{proposition}{is}
\label{lem:main_off-policy_IS}
Consider a $n$-armed bandit with stochastic rewards with bounded support and a unique optimal action. The behaviour policy $\mu_t$ selects action $i$ with probability $\mu_t(i)$ and let $\epsilon_t = \min_i \mu_t(i)$. When using NPG with importance sampling and a bounded baseline $b$, if $\lim_{t \to \infty} t \ \epsilon_t^2 = +\infty$ , then the target policy $\pi_t$ converges to the optimal policy in probability.
\end{restatable}
\begin{proof} 
\textit{(Sketch)}
Using Azuma-Hoeffding's inequality, we can show that for well chosen constants $\Delta_i, \delta$ and $C > 0$ ,
\begin{eqnarray*}
\sP\left( \theta_t^1 \ge  \theta_0^1 + \alpha \delta \Delta_1 t\right) %
&\ge& 1- \exp \left( -\frac{\delta^2 \Delta_1^2}{2 C^2}t \epsilon_t^2 \right)\\
\end{eqnarray*} 
where $\theta^1$ is the parameter associated to the optimal arm.
Thus if $\lim_{t\to\infty} t \epsilon_t^2 = +\infty$, the RHS goes to $1$. 
In a similar manner, we can upper bound $\sP\left( \theta_t^i \ge  \theta_0^i + \alpha \delta \Delta_i t\right)$ for all suboptimal arms, and 
applying an union bound, we get the desired result.
\end{proof}
\vspace{-3mm}

The condition on $\mu_t$ imposes a cap on how fast the behaviour policy can become deterministic: no faster than $t^{-1/2}$. Intuitively, this ensures each action is sampled sufficiently often and prevents premature convergence to a suboptimal policy.
The condition is satisfied for any sequence of behaviour policies which assign at least $\epsilon_t$ probability to each action at each step, such as $\epsilon$-greedy policies. It also holds if $\epsilon_t$ decreases over time at a sufficiently slow rate. By choosing as behaviour policy $\mu$ a linear interpolation between $\pi$ and the uniform policy, $\mu(a) = (1-\gamma) \pi(a) + \frac{\gamma}{K}, \gamma \in (0,1]$, where $K$ is the number of arms, we recover the classic EXP3 algorithm~\citep{auer2002nonstochastic, seldin2012evaluation}.

We can also confirm that this condition is not satisfied for the simple example we presented when discussing convergence to suboptimal policies. There, $p_t$ could decrease exponentially fast since the tails of the sigmoid function decay exponentially and the parameters move by at least a constant at every step. In this case, $\epsilon_t = \Omega (e^{-t})$, resulting in $\lim_{t\to \infty} t e^{-2 t} = 0$, so Proposition~\ref{lem:main_off-policy_IS} does not apply.

\subsection{Importance sampling, baselines \& variance}
As we have seen, using a separate behaviour policy that samples all actions sufficiently often may lead to stronger convergence guarantees, even if it increases the variance of the gradient estimates in most of the space, as what matters is what happens in the high variance regions, which are usually close to the boundaries. Fig.~\ref{fig:simplex_variance} shows the ratios of gradient variances between on-policy PG without baseline, on-policy PG with the minimum variance baseline, and off-policy PG using importance sampling~(IS) where the sampling distribution is $\mu(a) = \frac{1}{2} \pi(a) + \frac{1}{6}$, i.e. a mixture of the current policy $\pi$ and the uniform distribution. While using the minimum variance baseline decreases the variance on the entire space compared to not using a baseline, IS actually \emph{increases} the variance when the current policy is close to uniform. However, IS does a much better job at reducing the variance close to the boundaries of the simplex, where it actually matters to guarantee convergence.

\begin{figure}[t!]
\begin{center}
    \begin{subfigure}[b]{0.08\linewidth}
    \includegraphics[trim={2cm 1.5cm 2cm 0},clip,width=\textwidth]{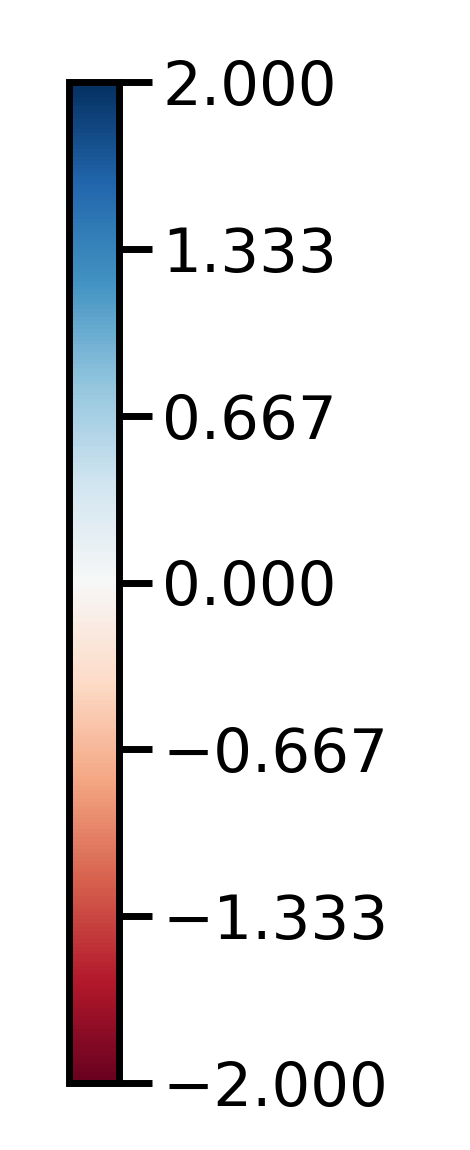}
    \label{fig:vary_corr}
  \end{subfigure}
  \begin{subfigure}[b]{.27\linewidth}
    \includegraphics[trim={2cm 1.5cm 2cm 0mm},clip, width=\textwidth]{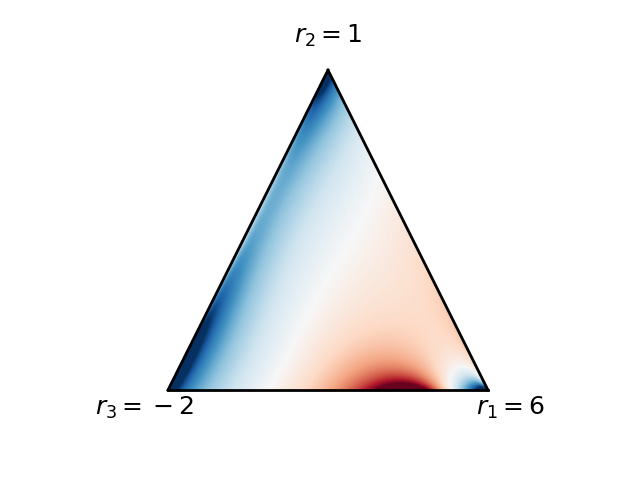}
    \caption{{\color{myred}$b=0$} / {\color{myblue}IS}.}
  \end{subfigure}
  \begin{subfigure}[b]{0.27\linewidth}
    \includegraphics[trim={2cm 1.5cm 2cm 0},clip,width=\textwidth]{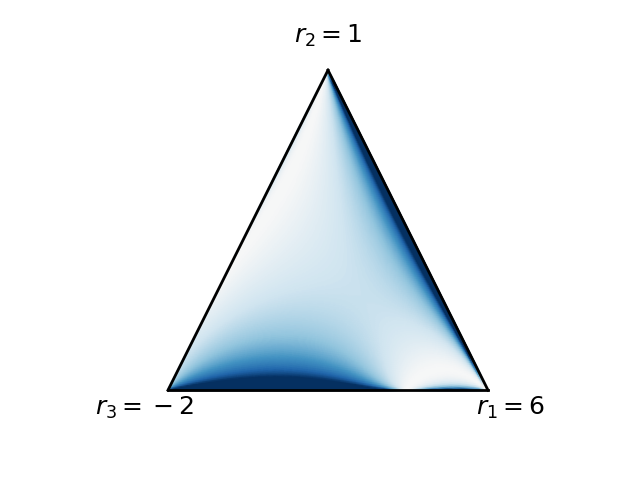}
    \caption{{\color{myred}$b=0$} / {\color{myblue}$b^\ast$}.}
  \end{subfigure}
  \begin{subfigure}[b]{0.27\linewidth}
    \includegraphics[trim={2cm 1.5cm 2cm 0},clip,width=\textwidth]{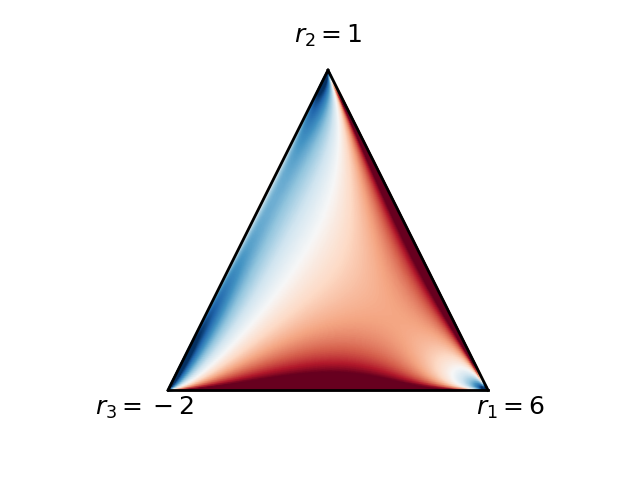}
    \caption{{\color{myred}$b^\ast$} / {\color{myblue}IS}.}
  \end{subfigure}
\caption{Comparison between the variance of different methods on a 3-arm bandit. Each plot depicts the $\log$ of the ratio between the variance of two approaches. For example, Fig. (a) depicts $\log \tfrac{\textrm{Var}[g_{b=0}]}{\textrm{Var}[g_{\text{IS}}]}$, the $\log$ of the ratio between the variance of the gradients of PG without a baseline and PG with IS. The triangle represents the probability simplex with each corner representing a deterministic policy on a specific arm. The method written in blue (resp. red) in each figure has lower variance in blue (resp. red) regions of the simplex. The sampling policy $\mu$, used in the PG method with IS, is a linear interpolation between $\pi$ and the uniform distribution, $\mu(a) = \frac{1}{2} \pi(a) + \frac{1}{6}$. Note that this is not the min. variance sampling distribution and it leads to higher variance than PG without a baseline in some parts of the simplex.} \label{fig:simplex_variance}
\end{center}
\end{figure}

This suggests that convergence of PG methods is not so much governed by the variance of the gradient estimates in general, but by the variance in the worst regions, usually near the boundary. While baselines can reduce the variance, they generally cannot prevent the variance in those regions from exploding, leading to the policy getting stuck. Thus, good baselines are not the ones reducing the variance across the space but rather those that can prevent the learning from reaching these regions altogether. Large values of $b$, such that $r(a_i) - b$ is negative for most actions, achieve precisely that. On the other hand, due to the increased flexibility of sampling distributions, IS can limit the nefariousness of these critical regions, offering better convergence guarantees despite not reducing variance~everywhere.

Importantly, although IS is usually used in RL to correct for the distribution of past samples~\citep[e.g.,][]{munos16retrace}, we advocate here for expanding the research on designing appropriate sampling distributions as done by~\citet{hanna2017data, hanna2018importance} and \citet{parmas2019unified}. This line of work has a long history in statistics~\citep[c.f.,][]{liu2008monte}. 

\subsection{Other mitigating strategies}
We conclude this section by discussing alternative strategies to mitigate the convergence issues. While they might be effective, and some are indeed used in practice, they are not without pitfalls.

First, one could consider reducing the stepsizes, with the hope that the policy would not converge as quickly towards a suboptimal deterministic policy and would eventually leave that bad region. Indeed, if we are to use vanilla PG in the two-arm bandit example, instead of NPG, this effectively reduces the stepsize by a factor of $\sigma(\theta)(1-\sigma(\theta))$ (the Fisher information). In this case, we are able to show convergence in probability to the optimal policy. See Proposition~\ref{prop_vpg_cv} in Appendix \ref{app:theory_2arm}.

Empirically, we find that, when using vanilla PG, the policy may still remain stuck near a suboptimal policy when using a negative baseline, similar to Fig.~\ref{fig:divergence_2arm_bandit}. While the previous proposition guarantees convergence eventually, the rate may be very slow, which remains problematic in practice. There is theoretical evidence that following even the true vanilla PG may result in slow convergence~\citep{schaul2019ray}, suggesting that the problem is not necessarily due to noise.

An alternative solution would be to add entropy regularization to the objective. By doing so, the policy would be prevented from getting too close to deterministic policies. While this might prevent convergence to a suboptimal policy, it would also exclude the possibility of fully converging to the optimal policy, though the policy may remain near it.

In bandits, EXP3 has been found not to enjoy high-probability guarantees on its regret so variants have been developed to address this deficiency \citep[c.f.][]{lattimore_szepesvari_2020}. For example, by introducing bias in the updates, their variance can be reduced significantly \cite{auer2002nonstochastic, neu2015explore}.
Finally, other works have also developed provably convergent policy gradient algorithms using different mechanisms, such as exploration bonuses or ensembles of policies~\citep{cai2019provably, efroni2020optimistic, agarwal2020pc}.

\section{Extension to multi-step MDPs}
\begin{figure*}
\centering
 \begin{subfigure}[b]{0.21\linewidth}   
    \includegraphics[trim={0cm -0.8cm 0cm 0},clip,width=\textwidth]{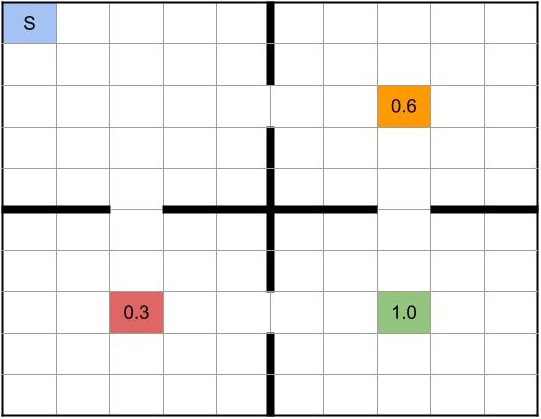}
    \caption{MDP}
    \label{fig:4rooms_mdp}
  \end{subfigure}
  \begin{subfigure}[b]{0.25\linewidth}
    \includegraphics[width=\textwidth]{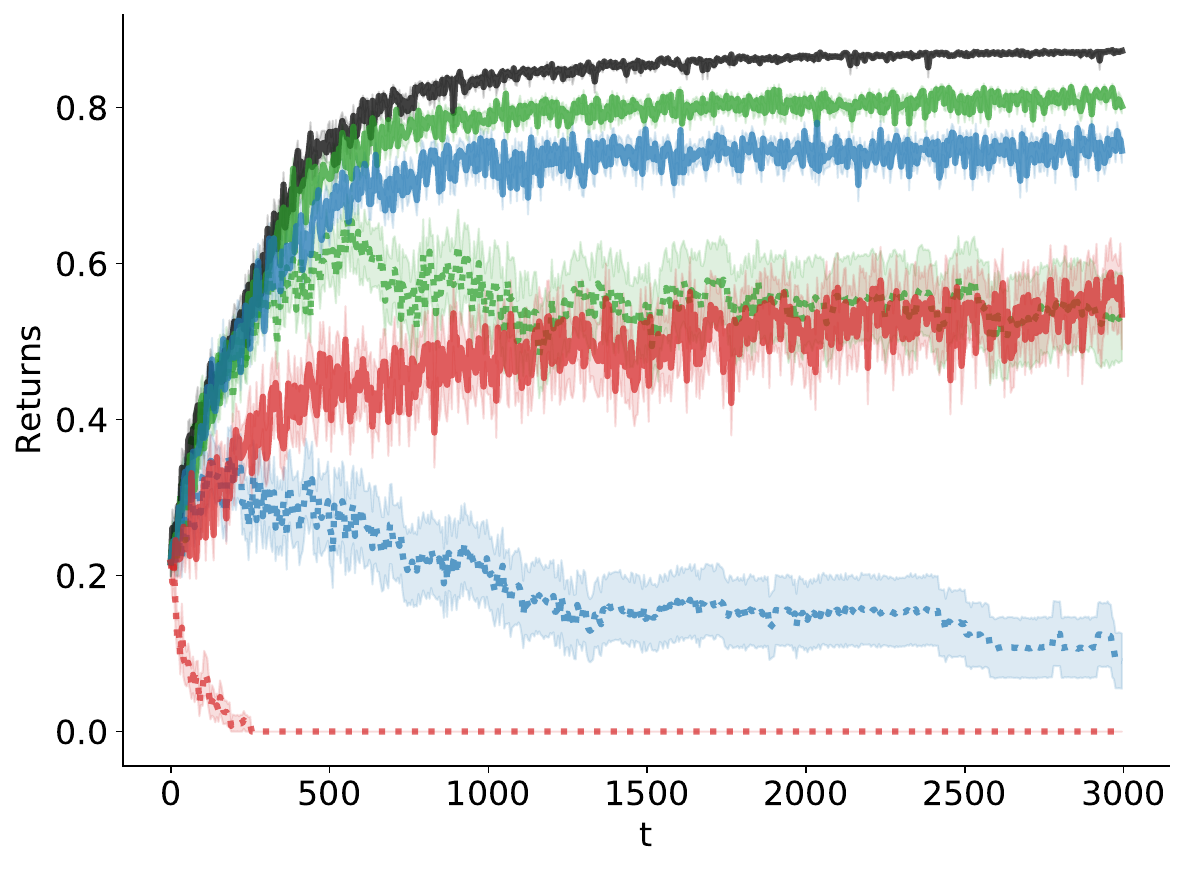}
    \caption{Returns}
    \label{fig:4rooms_return}
  \end{subfigure}
    \begin{subfigure}[b]{0.25\linewidth}
    \includegraphics[width=\textwidth]{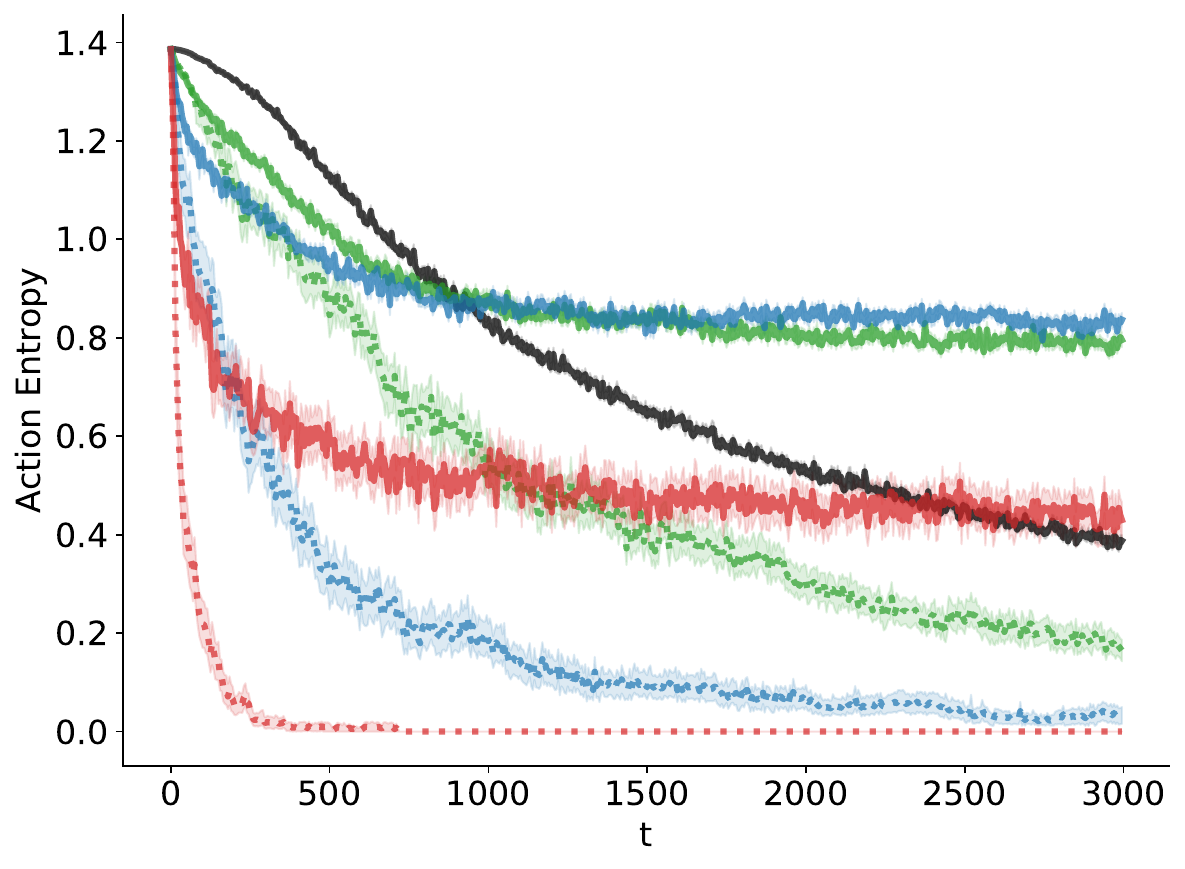}
    \caption{Entropy (A)}
    \label{fig:4rooms_action}
  \end{subfigure}
  \begin{subfigure}[b]{0.25\linewidth}
    \includegraphics[width=\textwidth]{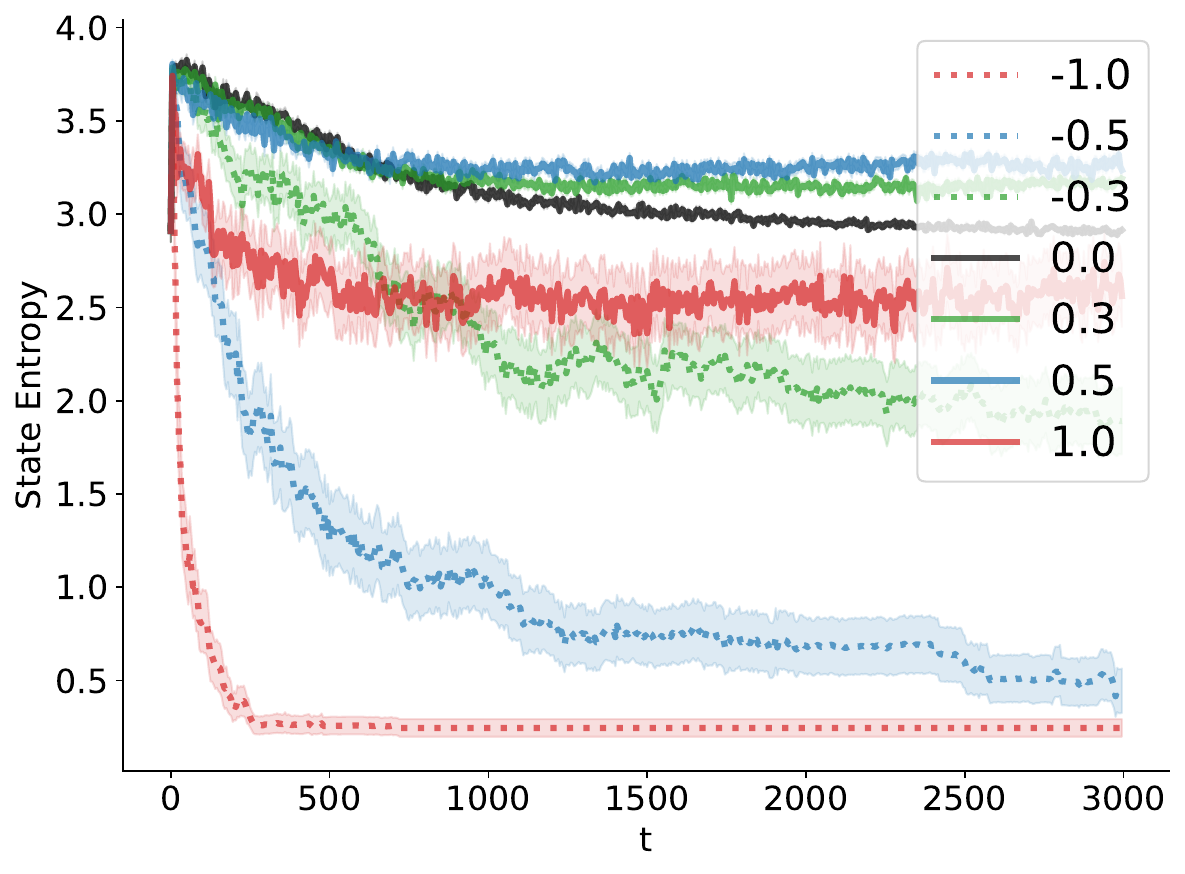}
    \caption{Entropy (S)}
    \label{fig:4rooms_state}
  \end{subfigure}
  \caption{We plot the discounted returns, the entropy of the policy over the states visited in each trajectory, and the entropy of the state visitation distribution, averaged over 50 runs, for multiple baselines. The baselines are of the form $b(s) = b^*(s) + \epsilon$, perturbations of the minimum-variance baseline, with $\epsilon$ indicated in the legend. The shaded regions denote one standard error. Note that the policy entropy of lower baselines tends to decay faster than for larger baselines. Also, smaller baselines tend to get stuck on suboptimal policies, as indicated by the returns plot. See text for additional details.~\label{fig:stats_4rooms}}
\end{figure*}
We focused our theoretical analyses on multi-arm bandits so far. However, we are also interested in more general environments where gradient-based methods are commonplace. 
We now turn our attention to the Markov Decision Process (MDP) framework~\citep{puterman2014markov}. An MDP is a set $\{\gS, \gA, P, r, \gamma, \rho\}$ where $\gS$ and $\gA$ are the set of states and actions, $P$ is the environment transition function,
$r$ is the reward function, $\gamma \in [0, 1)$ the discount factor,
and $\rho$ is the initial state distribution. The goal of RL algorithms is to find a policy $\pi_\theta$, parameterized by $\theta$, which maximizes the (discounted) expected return; i.e. Eq.~\ref{eq:bandit_loss} becomes
\begin{align*}
    \arg\max_\theta J(\theta)
    = \arg\max_\theta \sum_s d^{\pi_\theta}_\gamma(s)\sum_{a} \pi_\theta(a | s) r(s, a), %
\end{align*}
where there is now a discounted distribution over states induced by $\pi_\theta$. Although that distribution depends on $\pi_\theta$ in a potentially complex way, the parameter updates are similar to Eq.~\ref{eq:mc_update}:
\begin{align*}
    \theta_{t+1} &= \theta_t + \frac{\alpha}{N} \sum_i [Q(s_i, a_i) - b(s_i)]\nabla_\theta \log \pi_\theta(a_i | s_i) \label{eq:mc_update_mdp} \; ,
\end{align*}
where $(a_i, s_i)$ pairs are drawn according to the discounted state-visitation distribution induced by $\pi_\theta$ and $Q$ is the state-action value function induced by $\pi_\theta$~\citep[c.f.][]{sutton18book}. To match the bandit setting and common practice, we made the baseline state dependent. 

Although our theoretical analyses do not easily extend to multi-step MDPs, we empirically investigated if the similarity between these formulations leads to similar differences in learning dynamics when changing the baseline. We consider a 10x10 gridworld consisting of 4 rooms as depicted on Fig.~\ref{fig:4rooms_mdp}. We use a discount factor $\gamma=0.99$. The agent starts in the upper left room and two adjacent rooms contain a goal state of value $0.6$ or $0.3$. The best  goal (even discounted), with a value of $1$, lies in the furthest room, so that the agent must learn to cross the sub-optimal rooms and reach the furthest one. 

Similar to the bandit setting, for a state $s$, we can derive the minimum-variance baseline $b^*(s)$ assuming access to state-action values $Q(s,a)$ for $\pi_\theta$ and consider perturbations to it. Again, we use baselines $b(s) = b^*(s) + \epsilon$ and $b(s) = b^*(s) - \epsilon$, since they result in identical variances. 
We use a natural policy gradient estimate, which substitutes $\nabla \log \pi (a_i|s_i)$ by $F^{-1}_{s_i} \nabla \log \pi(a_i|s_i)$ in the update rule, where $F_{s_i}$ is the Fisher information matrix for state $s_i$ and solve for the exact $Q(s,a)$ values using dynamic programming for all updates (see Appendix \ref{app:npg_mdp_estimate} for details).

In order to identify the committal vs. non-committal behaviour of the agent depending on the baseline, we monitor the entropy of the policy and the entropy of the stationary state distribution over time.
Fig.\ref{fig:4rooms_return} shows the average returns over time and Fig.\ref{fig:4rooms_action} and \ref{fig:4rooms_state} show the entropy of the policy in two ways. The first is the average entropy of the action distribution along the states visited in each trajectory, and the second is the entropy of the distribution of the number of times each state is visited up to that point in training.

The action entropy for smaller baselines tends to decay faster compared to larger ones, indicating convergence to a deterministic policy. This quick convergence is premature in some cases since the returns are not as high for the lower baselines.
In fact for $\epsilon=-1$, we see that the agent gets stuck on a policy that is unable to reach any goal within the time limit, as indicated by the returns of $0$.  
On the other hand, the larger baselines tend to achieve larger returns with larger entropy policies, but do not fully converge to the optimal policy as evidenced by the gap in the returns plot.

Since committal and non-committal behaviour can be directly inferred from the PG and the sign of the effective rewards $R(\tau) - b$, we posit that these effects extend to all MDPs.
In particular, in complex MDPs, the first trajectories explored are likely to be suboptimal and a low baseline will increase their probability of being sampled again, requiring the use of techniques such as entropy regularization to prevent the policy~from~getting~stuck~too~quickly. 

\section{Conclusion} 

We presented results that dispute common beliefs about baselines, variance, and policy gradient methods in general. 
As opposed to the common belief that baselines only provide benefits through variance reduction, we showed that they can significantly affect the optimization process in ways that cannot be explained by the variance and that lower variance can even sometimes be detrimental.

Different baselines can give rise to very different learning dynamics, even when they reduce the variance of the gradients equally. They do that by either making a policy quickly tend towards a deterministic one (\emph{committal} behaviour) or by maintaining high-entropy for a longer period of time (\emph{non-committal} behaviour). We showed that \textit{committal} behaviour can be problematic and lead to convergence to a suboptimal policy. Specifically, we showed that stochastic natural policy gradient does not always converge to the optimal solution due to the unusual situation in which the iterates converge to the optimal policy in expectation but not almost surely. 
Moreover, we showed that baselines that lead to lower-variance can sometimes be detrimental to optimization, highlighting the limitations of using variance to analyze the convergence properties of these methods. 
We also showed that standard convergence guarantees for PG methods do not apply to some settings because the assumption of bounded variance of the updates is violated.

The aforementioned convergence issues are also caused by the problematic coupling between the algorithm's updates and its sampling distribution since one directly impacts the other. As a potential solution, we showed that off-policy sampling can sidestep these difficulties by ensuring we use a sampling distribution that is different than the one induced by the agent's current policy. This supports the hypothesis that on-policy learning can be problematic, as observed in previous work~\citep{schaul2019ray, hennes2020neural}. Nevertheless, importance sampling in RL is generally seen as problematic~\citep{vanHasselt2018deadly_triad} due to instabilities it introduces to the learning process. Moving from an imposed policy, using past trajectories, to a chosen sampling policy reduces the variance of the gradients for near-deterministic policies and can lead to much better behaviour. 

More broadly, this work suggests that treating bandit and reinforcement learning problems as a black-box optimization of a function $J(\theta)$ may be insufficient to perform well. As we have seen, the current parameter value can affect all future parameter values by influencing the data collection process and thus the updates performed. Theoretically, relying on immediately available quantities such as the gradient variance and ignoring the sequential nature of the optimization problem is not enough to discriminate between certain optimization algorithms. In essence, to design highly-effective policy optimization algorithms, it may be necessary to develop a better understanding of how the optimization process evolves over many steps.

\section*{Acknowledgements}
We would like to thank Kris de Asis, Alan Chan, Ofir Nachum, Doina Precup, Dale Schuurmans, and Ahmed Touati for helpful discussions. We also thank Courtney Paquette, Vincent Liu and Scott Fujimoto for reviewing an earlier version of this paper. Nicolas Le Roux is supported by a Canada CIFAR AI Chair.

\bibliography{full}
\bibliographystyle{plainnat}

\begin{appendix}

\input{appendix_techreport.tex}

\end{appendix}

\end{document}

%% file: math_commands.tex
\usepackage{amsmath,amsfonts,bm}

\def\eqref#1{equation~\ref{#1}}

\def\1{\bm{1}}

\def\eps{{\epsilon}}

\DeclareMathAlphabet{\mathsfit}{\encodingdefault}{\sfdefault}{m}{sl}
\SetMathAlphabet{\mathsfit}{bold}{\encodingdefault}{\sfdefault}{bx}{n}

\def\gA{{\mathcal{A}}}

\def\gS{{\mathcal{S}}}

\def\sN{{\mathbb{N}}}

\def\sP{{\mathbb{P}}}

\newcommand{\E}{\mathbb{E}}

\newcommand{\R}{\mathbb{R}}

%% file: appendix_techreport.tex
\newpage

\noindent \noindent \textbf{\LARGE Appendix}\\
\vspace{1em}

\section*{Organization of the appendix}

We organize the appendix into several thematic sections. 

The first one, section \ref{app:exp} contains additional experiments and figures on bandits and MDPs. 
We have further investigations into committal and non-committal behaviour with baselines. 
More precisely subsection~\ref{app:exp_3armbandit} contains additional experiments for the 3 arm bandits for vanilla policy gradient, natural policy gradient and policy gradient with direct parameterization and a discussion on the effect the hyperparameters have on the results. In all cases, we find evidence for committal and non-committal behaviours. 
In the rest of the section, we investigate this in MDPs, starting with a smaller MDP with 2 different goals in subsection~\ref{app:simple_mdp} and constant baselines. We also provide additional experiments on the 4 rooms environment in subsection~\ref{app:4rooms_mdp}, including the vanilla policy gradient and constant baselines with REINFORCE. 

Then, section~\ref{app:theory_2arm} contains theory for the two-armed bandit case, namely proofs of convergence to a suboptimal policy (Proposition~\ref{proposition_divergence} in Appendix~\ref{app:2arm_constant_baseline_div}) and an analysis of perturbed minimum-variance baselines (Proposition~\ref{prop:2armed-perturbedminvar} in Appendix~\ref{sec:appendix_perturbed_minvar}). For the latter, depending on the perturbation, we may have possible convergence to a suboptimal policy, convergence to the optimal policy in probability, or a weaker form of convergence to the optimal policy. Finally, we also show vanilla policy gradient converges to the optimal policy in probability regardless of the baseline in Appendix~\ref{app:2arm_vanilla_pg}. 

Section~\ref{app:theory_multiarm} contains the theory for multi-armed bandit, including the proof of theorem~\ref{proposition_threearmedbandit}. This theorem presents a counterexample to the idea that reducing variance always improves optimization. We show that there is baseline leading to reduced variance which may converge to a suboptimal policy with positive probability (see Appendix~\ref{app:3arm_minvar_baseline_div}) while there is another baseline with larger variance that converges to the optimal policy with probability $1$ (see Appendix~\ref{app:3arm_gap_baseline}).
We identify on-policy sampling as being a potential source of these convergence issues. We provide proofs of proposition~\ref{lem:main_off-policy_IS} in Appendix \ref{app:3arm_off_policy}, which shows convergence to the optimal policy in probability when using off-policy sampling with importance sampling.  

Finally, in section~\ref{app:other_results}, we provide derivations of miscellaneous, smaller results such as the calculation of the minimum-variance baseline (Appendix~\ref{app:optimal_baseline}), the natural policy gradient update for the softmax parameterization (Appendix~\ref{app:npg_softmax_bandit}) and the connection between the value function and the minimum-variance baseline  (Appendix~\ref{app:value_minvar_baseline}).

\section{Other experiments}
\label{app:exp}

\subsection{Three-armed bandit}
\label{app:exp_3armbandit}
In this subsection, we provide additional experiments on the three-armed bandit with natural and vanilla policy gradients for the softmax parameterization, varying the initializations.  Additionally, we present results for the direct parameterization and utilizing projected stochastic gradient ascent. 

The main takeaway is that the effect of the baselines appears more strongly when the initialization is unfavorable (for instance with a high probability of selecting a suboptimal action at first). The effect also are diminished when using small learning rates as in that case the effect of the noise on the optimization process lessens.

While the simplex visualization is very appealing, we mainly show here learning curves as we can showcase more seeds that way and show the effects are noticeable across many runs.

\subsubsection*{Natural policy gradient}
\begin{figure}[!ht]
\centering
  \begin{subfigure}[b]{0.245\linewidth}
    \includegraphics[width=\textwidth]{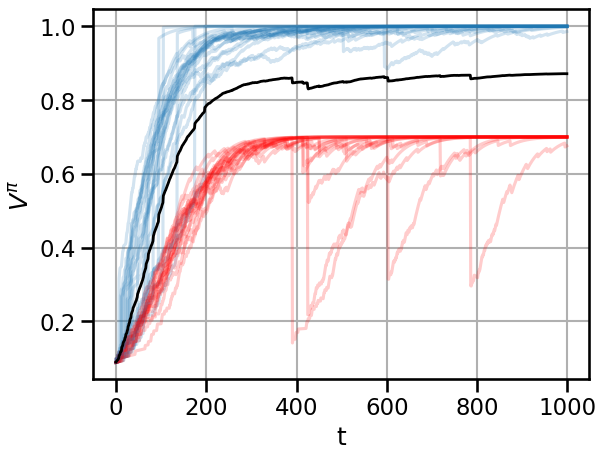}
    \caption{$b = b^\ast - \nicefrac{1}{2}$}
  \end{subfigure}
    \begin{subfigure}[b]{0.245\linewidth}
    \includegraphics[width=\textwidth]{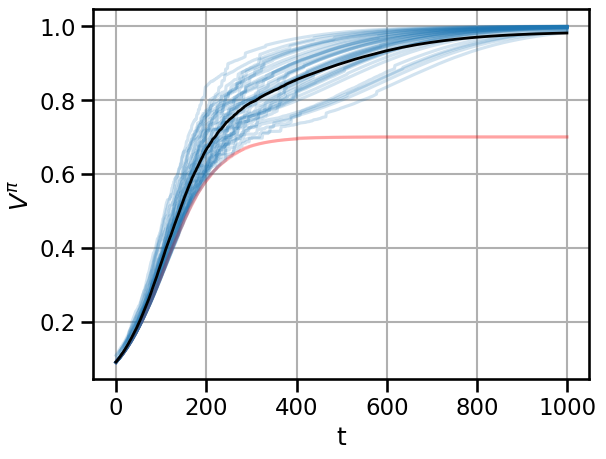}
    \caption{$b = b^\ast$}
  \end{subfigure}
  \begin{subfigure}[b]{0.245\linewidth}
    \includegraphics[width=\textwidth]{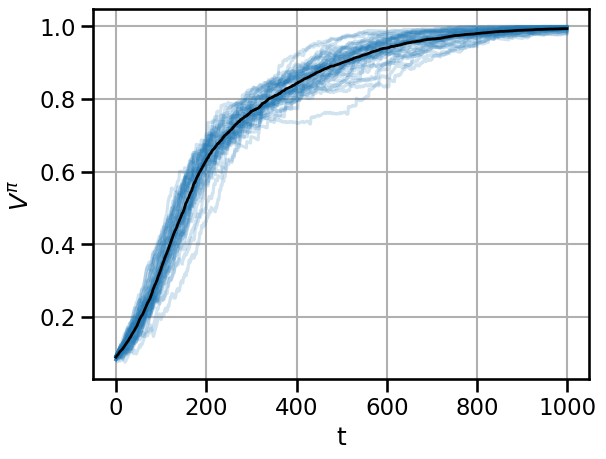}
    \caption{$b = b^\ast + \nicefrac{1}{2}$}
  \end{subfigure}
  \begin{subfigure}[b]{0.245\linewidth}
    \includegraphics[width=\textwidth]{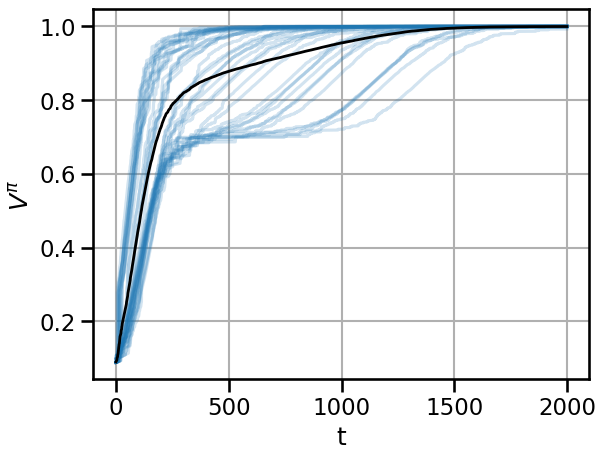}
    \caption{$b=V^\pi$}
  \end{subfigure}
  \caption{We plot 40 different learning curves (in blue and red) of natural policy gradient, when using various baselines, on a 3-arm bandit problem with rewards $(1, 0.7, 0)$, $\alpha = 0.025$ and $\theta_0 = (0, 3, 5)$. The black line is the average value over the 40 seeds for each setting. The red curves denote the seeds that did not reach a value of at least $0.9$ at the end of training. Note that the value function baseline convergence was slow and thus was trained for twice the number of time steps.}~\label{appfig:learning_curves_npg_035}
\end{figure}
Figure~\ref{appfig:learning_curves_npg_035} uses the same setting as Figure~\ref{fig:trajectories} with 40 trajectories instead of 15. We do once again observe many cases of convergence to the wrong arm for the negative baseline and some cases for the minimum variance baseline, while the positive baseline converges reliably. In this case the value function also converges to the optimal solution but is much slower.

\begin{figure}[!ht]
\centering
  \begin{subfigure}[b]{0.245\linewidth}
    \includegraphics[width=\textwidth]{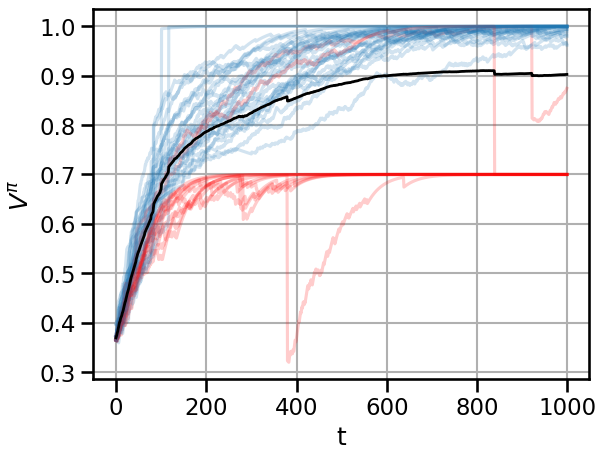}
    \caption{$b = b^\ast - \nicefrac{1}{2}$}
  \end{subfigure}
    \begin{subfigure}[b]{0.245\linewidth}
    \includegraphics[width=\textwidth]{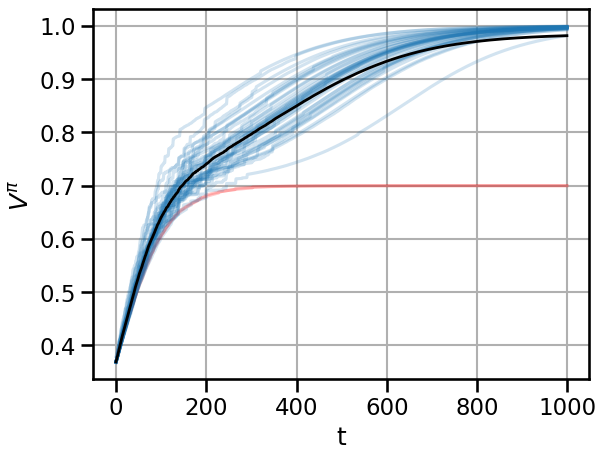}
    \caption{$b = b^\ast$}
  \end{subfigure}
  \begin{subfigure}[b]{0.245\linewidth}
    \includegraphics[width=\textwidth]{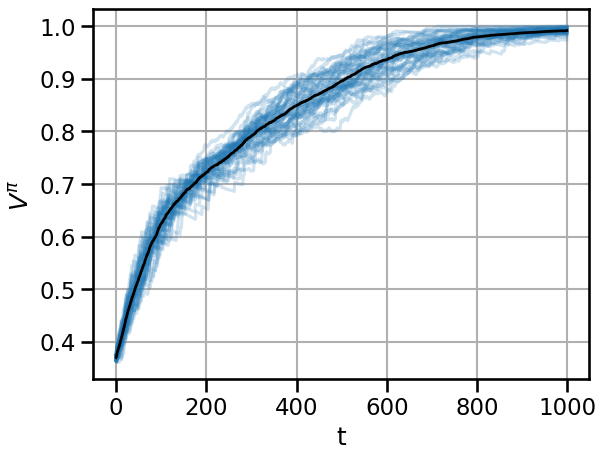}
    \caption{$b = b^\ast + \nicefrac{1}{2}$}
  \end{subfigure}
  \begin{subfigure}[b]{0.245\linewidth}
    \includegraphics[width=\textwidth]{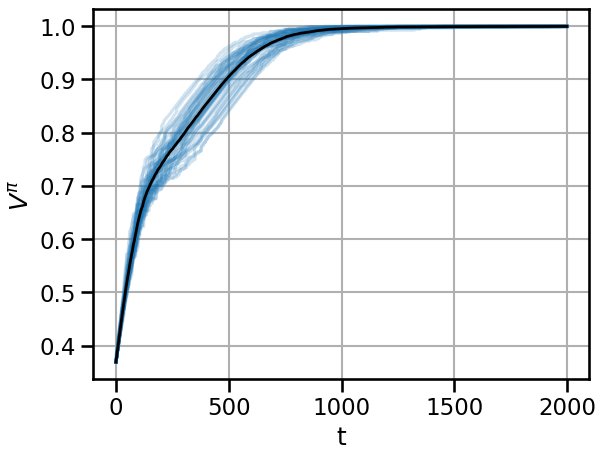}
    \caption{$b=V^\pi$}
  \end{subfigure}
  \caption{We plot 40 different learning curves (in blue and red) of natural policy gradient, when using various baselines, on a 3-arm bandit problem with rewards $(1, 0.7, 0)$, $\alpha = 0.025$ and $\theta_0 = (0, 3, 3)$. The black line is the average value over the 40 seeds for each setting. The red curves denote the seeds that did not reach a value of at least $0.9$ at the end of training.}~\label{appfig:learning_curves_npg_033}
\end{figure}
Figure~\ref{appfig:learning_curves_npg_033} shows a similar setting to Figure~\ref{appfig:learning_curves_npg_035} but where the initialization parameter is not as extreme. We observe the same type of behavior, but not as pronounced as before; fewer seeds converge to the wrong arm.

\begin{figure}[!ht]
\centering
  \begin{subfigure}[b]{0.245\linewidth}
    \includegraphics[width=\textwidth]{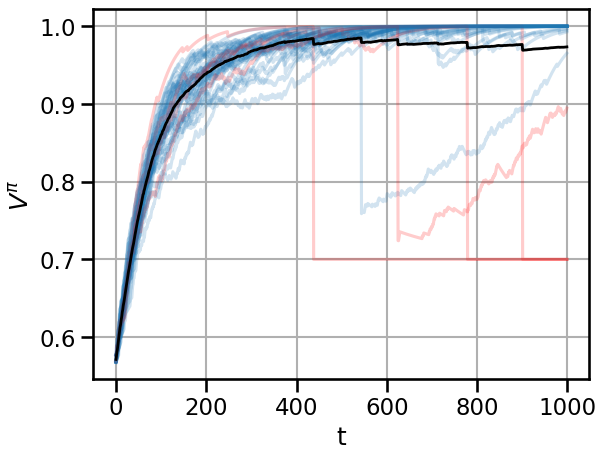}
    \caption{$b = b^\ast - \nicefrac{1}{2}$}
  \end{subfigure}
    \begin{subfigure}[b]{0.245\linewidth}
    \includegraphics[width=\textwidth]{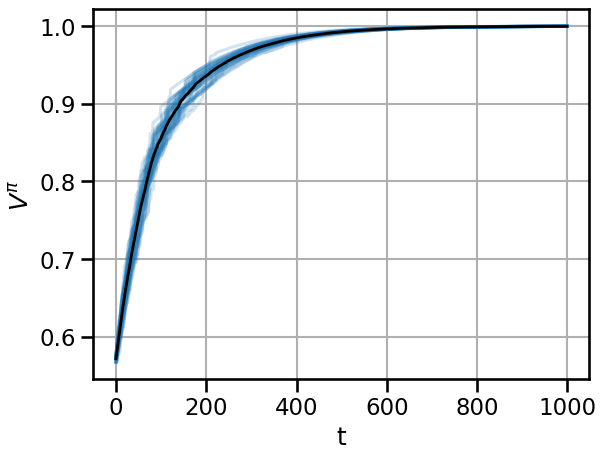}
    \caption{$b = b^\ast$}
  \end{subfigure}
  \begin{subfigure}[b]{0.245\linewidth}
    \includegraphics[width=\textwidth]{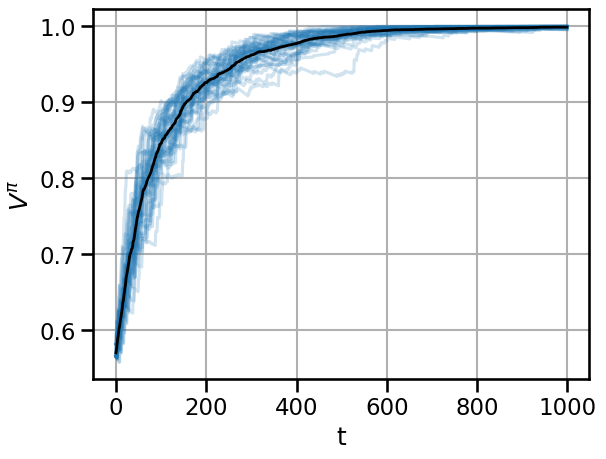}
    \caption{$b = b^\ast + \nicefrac{1}{2}$}
  \end{subfigure}
  \begin{subfigure}[b]{0.245\linewidth}
    \includegraphics[width=\textwidth]{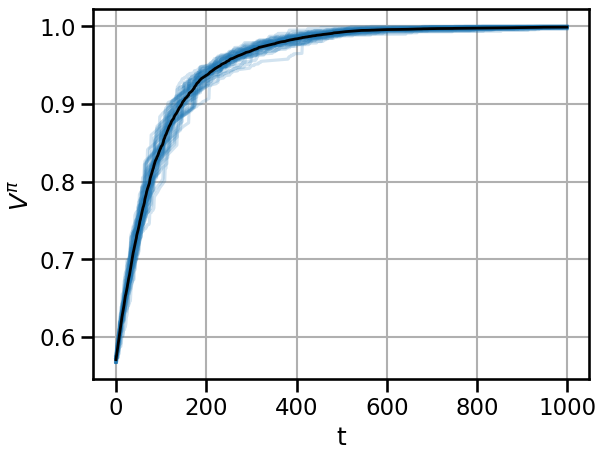}
    \caption{$b=V^\pi$}
  \end{subfigure}
  \caption{We plot 40 different learning curves (in blue and red) of natural policy gradient, when using various baselines, on a 3-arm bandit problem with rewards $(1, 0.7, 0)$, $\alpha = 0.025$ and $\theta_0 = (0, 0, 0)$ i.e the initial policy is uniform. The black line is the average value over the 40 seeds for each setting. The red curves denote the seeds that did not reach a value of at least $0.9$ at the end of training.}~\label{appfig:learning_curves_npg_000}
\end{figure}
In Figure~\ref{appfig:learning_curves_npg_000} whose initial policy is the uniform, we observe that the minimum variance baseline and the value function as baseline perform very well. On the other hand the committal baseline still has seeds that do not converge to the right arm. Interestingly, while all seeds for the non-committal baseline identify the optimal arm, the variance of the return is higher than for the optimal baseline, suggesting a case similar to the result presented in Proposition \ref{lem:prop_epsilon_1inf} where a positive baseline ensured we get close to the optimal arm but may not remain arbitrary close to it.

\subsubsection*{Vanilla policy gradient}
While we have no theory indicating that we may converge to a suboptimal arm with vanilla policy gradient, we can still observe some effect in terms of learning speed in practice (see Figures \ref{appfig:learning_curves_vpg_000_simplex} to \ref{appfig:learning_curves_vpg_033}). 

On Figures~\ref{appfig:learning_curves_vpg_000_simplex} and~\ref{appfig:learning_curves_vpg_000} we plot the simplex view and the learning curves for vanilla policy gradient initialized at the uniform policy. We do observe that some trajectories did not converge to the optimal arm in the imparted time for the committal baseline, while they converged in all other settings. The mininum variance baseline is slower to converge than the non-committal and the value function in this setting as can be seem both in the simplex plot and learning curves.

\begin{figure}[!ht]
\centering
  \begin{subfigure}[b]{0.245\linewidth}
    \includegraphics[width=\textwidth]{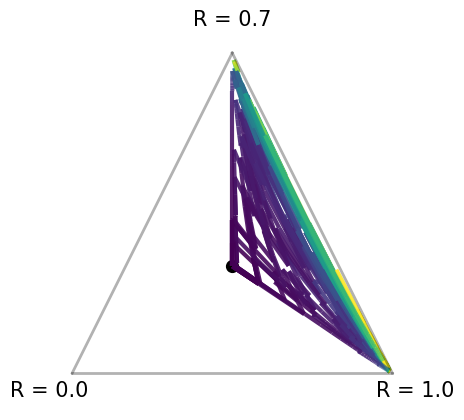}
    \caption{$b = b^\ast - \nicefrac{1}{2}$}
  \end{subfigure}
    \begin{subfigure}[b]{0.245\linewidth}
    \includegraphics[width=\textwidth]{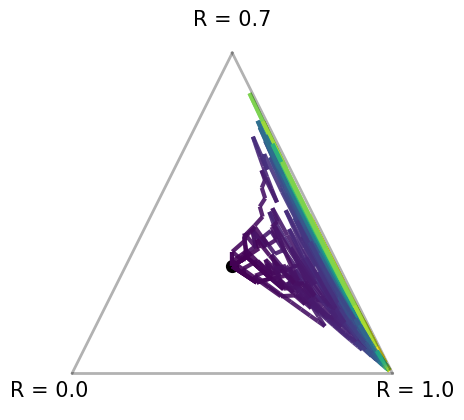}
    \caption{$b = b^\ast$}
  \end{subfigure}
  \begin{subfigure}[b]{0.245\linewidth}
    \includegraphics[width=\textwidth]{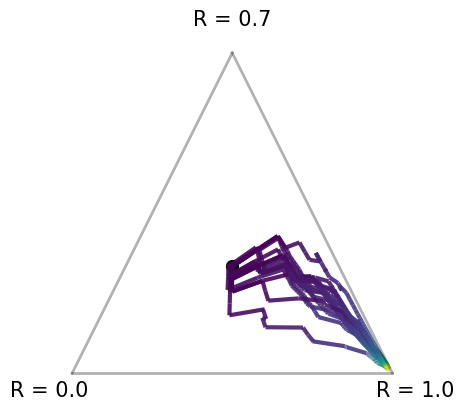}
    \caption{$b = b^\ast + \nicefrac{1}{2}$}
  \end{subfigure}
  \begin{subfigure}[b]{0.245\linewidth}
    \includegraphics[width=\textwidth]{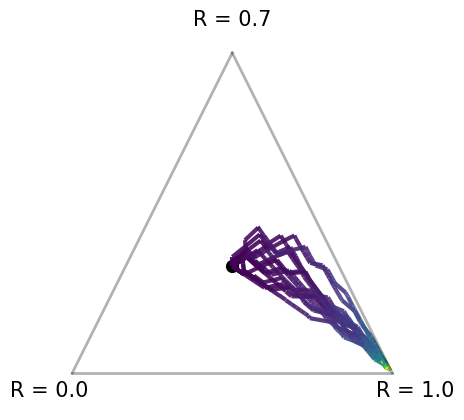}
    \caption{$b=V^\pi$}
  \end{subfigure}
  \caption{Simplex plot of 15 different learning curves for vanilla policy gradient, when using various baselines, on a 3-arm bandit problem with rewards $(1, 0.7, 0)$, $\alpha = 0.5$ and $\theta_0 = (0, 0, 0)$. Colors, from purple to yellow represent training steps.}~\label{appfig:learning_curves_vpg_000_simplex}
\end{figure}

\begin{figure}[!ht]
\centering
  \begin{subfigure}[b]{0.245\linewidth}
    \includegraphics[width=\textwidth]{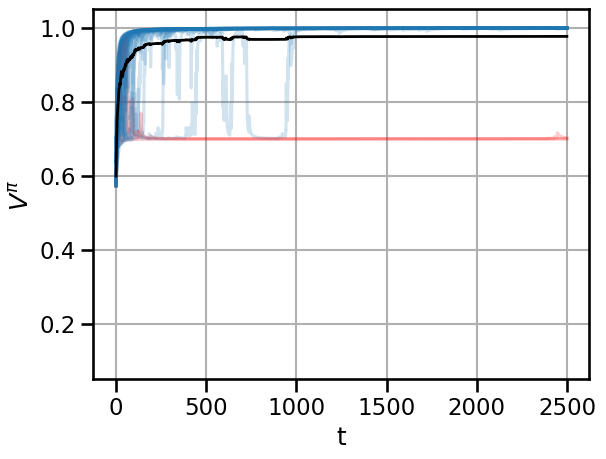}
    \caption{$b = b^\ast - \nicefrac{1}{2}$}
  \end{subfigure}
    \begin{subfigure}[b]{0.245\linewidth}
    \includegraphics[width=\textwidth]{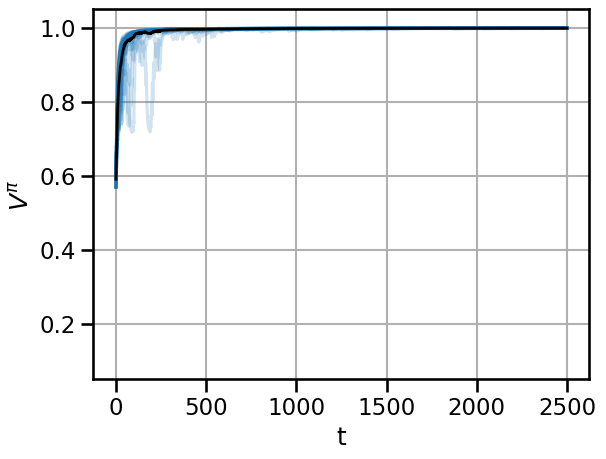}
    \caption{$b = b^\ast$}
  \end{subfigure}
  \begin{subfigure}[b]{0.245\linewidth}
    \includegraphics[width=\textwidth]{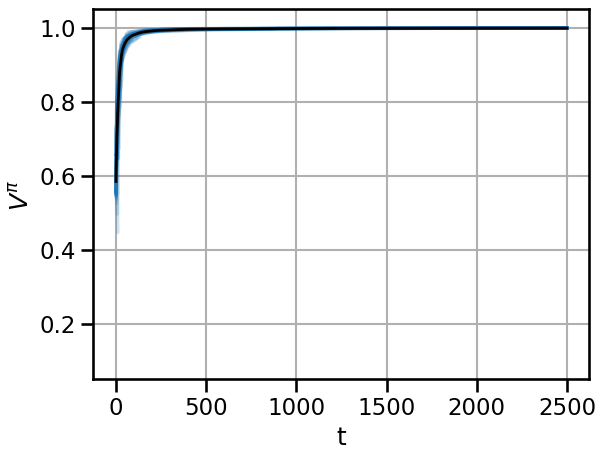}
    \caption{$b = b^\ast + \nicefrac{1}{2}$}
  \end{subfigure}
  \begin{subfigure}[b]{0.245\linewidth}
    \includegraphics[width=\textwidth]{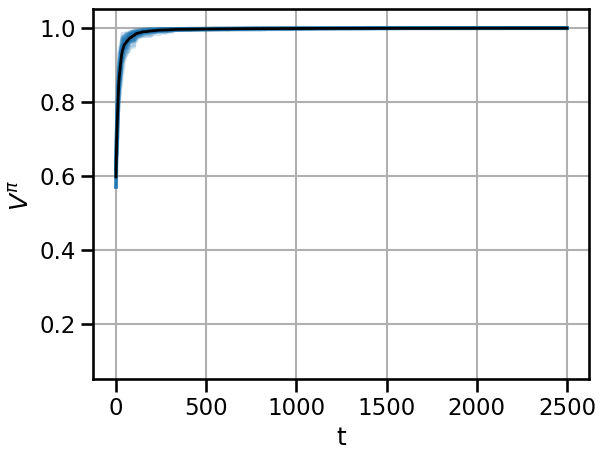}
    \caption{$b=V^\pi$}
  \end{subfigure}
  \caption{We plot 40 different learning curves (in blue and red) of vanilla policy gradient, when using various baselines, on a 3-arm bandit problem with rewards $(1, 0.7, 0)$, $\alpha = 0.5$ and $\theta_0 = (0, 0, 0)$. The black line is the average value over the 40 seeds for each setting. The red curves denote the seeds that did not reach a value of at least $0.9$ at the end of training.}~\label{appfig:learning_curves_vpg_000}
\end{figure}

On Figures~\ref{appfig:learning_curves_vpg_033_simplex} and~\ref{appfig:learning_curves_vpg_033} we plot the simplex view and the learning curves for vanilla policy gradient initialized at a policy yielding a very high probability of sampling the suboptimal actions, $48.7 \%$ for each. We do observe a similar behavior than for the previous plots with vanilla PG, but in this setting the minimum variance baseline is even slower to converge and a few seeds did not identify the optimal arm. As the gradient flow leads the solutions closer to the simplex edges, the simplex plot is not as helpful in this setting to understand the behavior of each baseline option.

\begin{figure}[!ht]
\centering
  \begin{subfigure}[b]{0.245\linewidth}
    \includegraphics[width=\textwidth]{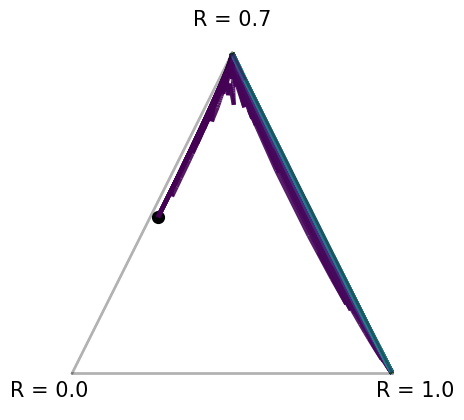}
    \caption{$b = b^\ast - \nicefrac{1}{2}$}
  \end{subfigure}
    \begin{subfigure}[b]{0.245\linewidth}
    \includegraphics[width=\textwidth]{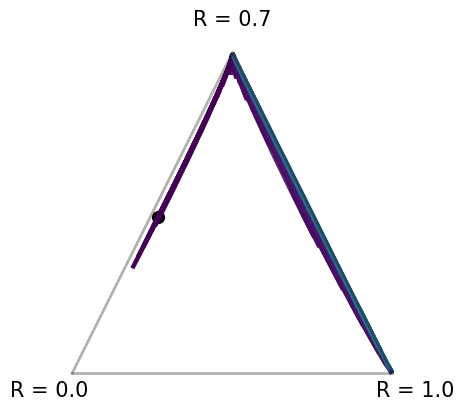}
    \caption{$b = b^\ast$}
  \end{subfigure}
  \begin{subfigure}[b]{0.245\linewidth}
    \includegraphics[width=\textwidth]{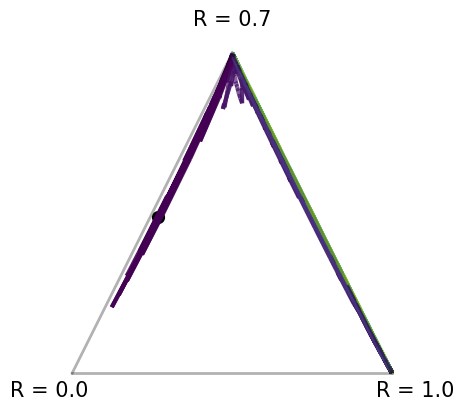}
    \caption{$b = b^\ast + \nicefrac{1}{2}$}
  \end{subfigure}
  \begin{subfigure}[b]{0.245\linewidth}
    \includegraphics[width=\textwidth]{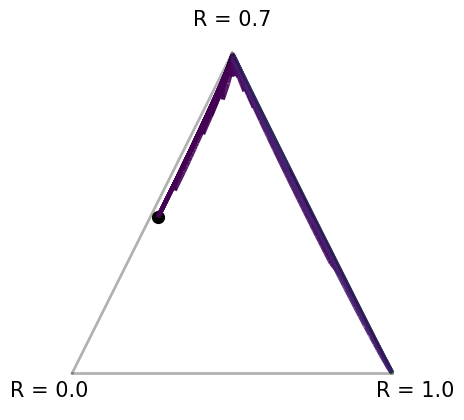}
    \caption{$b=V^\pi$}
  \end{subfigure}
  \caption{Simplex plot of 15 different learning curves for vanilla policy gradient, when using various baselines, on a 3-arm bandit problem with rewards $(1, 0.7, 0)$, $\alpha = 0.5$ and $\theta_0 = (0, 3, 3)$. Colors, from purple to yellow represent training steps.}~\label{appfig:learning_curves_vpg_033_simplex}
\end{figure}

\begin{figure}[!ht]
\centering
  \begin{subfigure}[b]{0.245\linewidth}
    \includegraphics[width=\textwidth]{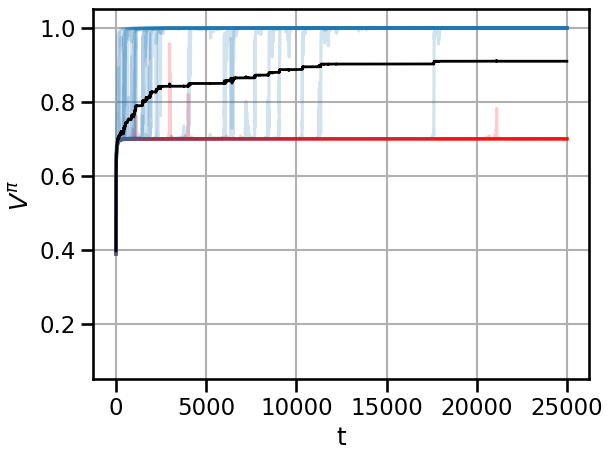}
    \caption{$b = b^\ast - \nicefrac{1}{2}$}
  \end{subfigure}
    \begin{subfigure}[b]{0.245\linewidth}
    \includegraphics[width=\textwidth]{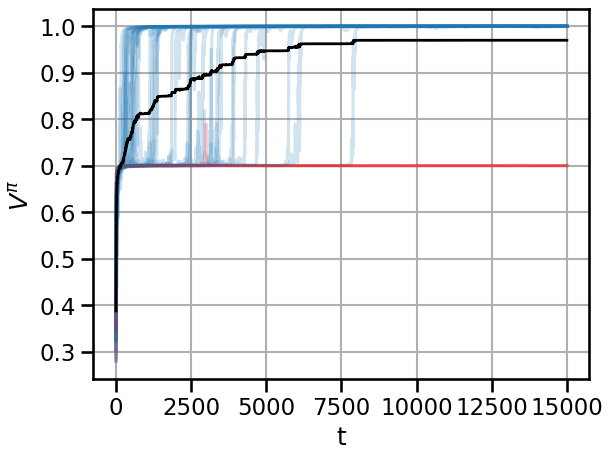}
    \caption{$b = b^\ast$}
  \end{subfigure}
  \begin{subfigure}[b]{0.245\linewidth}
    \includegraphics[width=\textwidth]{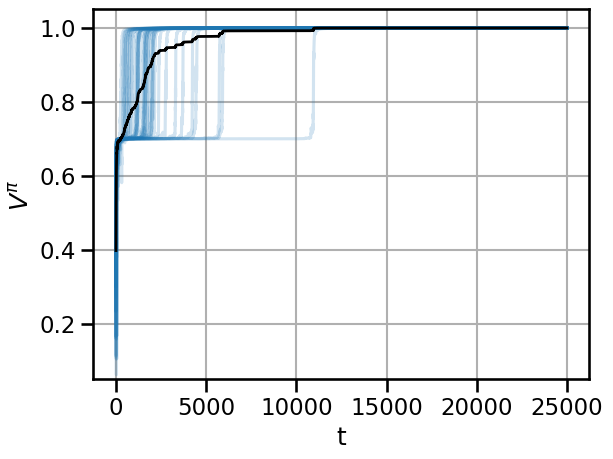}
    \caption{$b = b^\ast + \nicefrac{1}{2}$}
  \end{subfigure}
  \begin{subfigure}[b]{0.245\linewidth}
    \includegraphics[width=\textwidth]{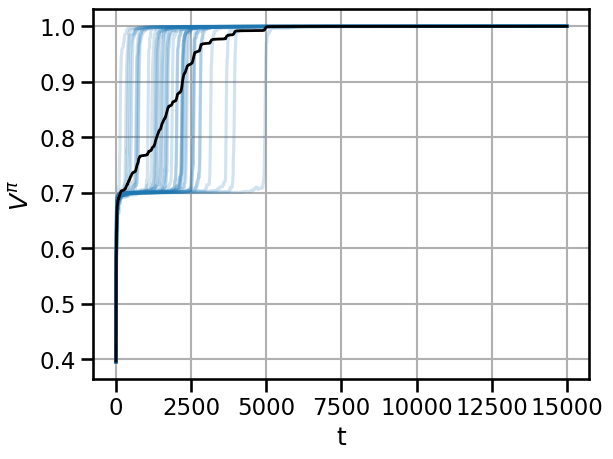}
    \caption{$b=V^\pi$}
  \end{subfigure}
  \caption{We plot 40 different learning curves (in blue and red) of vanilla policy gradient, when using various baselines, on a 3-arm bandit problem with rewards $(1, 0.7, 0)$, $\alpha = 0.5$ and $\theta_0 = (0, 3, 3)$. The black line is the average value over the 40 seeds for each setting. The red curves denote the seeds that did not reach a value of at least $0.9$ at the end of training.}~\label{appfig:learning_curves_vpg_033}
\end{figure}

\subsubsection*{Policy gradient with direct parameterization}

Here we present results with the direct parameterization, i.e where $\theta$ contains directly the probability of drawing each arm. In that case the gradient update is
\begin{align*}
    \theta_{t+1} &= \text{Proj}_{\Delta_3} \big[ \theta_t + \alpha \frac{r(a_i)-b}{\theta(a_i)} \mathbbm{1}_{a_i} \big]
\end{align*}
where $\Delta_3$ is the three dimensional simplex $\Delta_3 = \{u, v, w \ge 0, u+v+w = 1\}$. In this case, however, because the projection step is non trivial and doesn't have an easy explicit closed form solution (but we can express it as the output of an algorithm), we cannot explicitly write down the optimal baseline. Again, because of the projection step, baselines of this form are not guaranteed to preserve unbiasedness of the gradient estimate. For this reason, we only show experiments with fixed baselines, but keep in mind that these results are not as meaningful as the ones presented above. We present the results in Figures~\ref{appfig:learning_curves_dpg_035_simplex} and~\ref{appfig:learning_curves_dpg_035}.

\begin{figure}[!ht]
\centering
  \begin{subfigure}[b]{0.3\linewidth}
    \includegraphics[width=\textwidth]{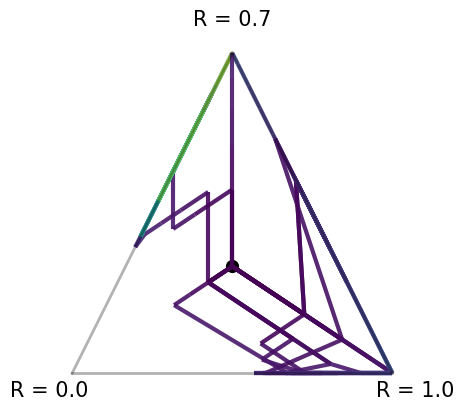}
    \caption{$b = - \nicefrac{1}{2}$}
  \end{subfigure}
    \begin{subfigure}[b]{0.3\linewidth}
    \includegraphics[width=\textwidth]{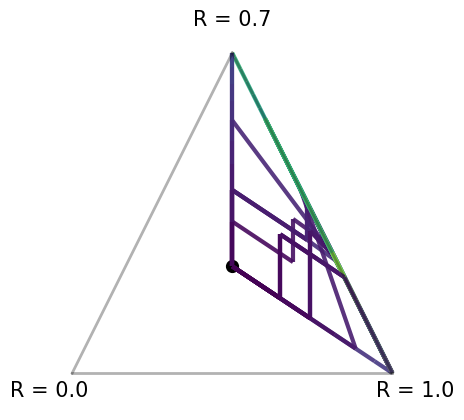}
    \caption{$b = 0$}
  \end{subfigure}
  \begin{subfigure}[b]{0.3\linewidth}
    \includegraphics[width=\textwidth]{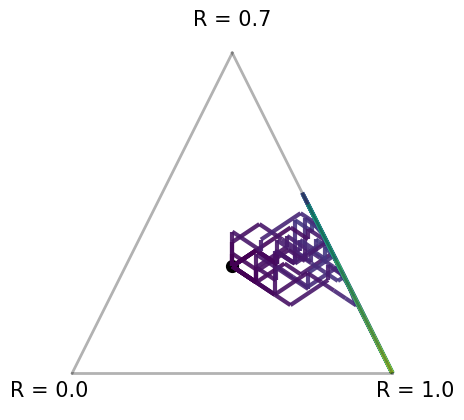}
    \caption{$b = \nicefrac{1}{2}$}
  \end{subfigure}
  \caption{We plot 15 different learning curves of vanilla policy gradient with direct parameterization, when using various baselines, on a 3-arm bandit problem with rewards $(1, 0.7, 0)$, $\alpha = 0.1$ and $\theta_0 = (\nicefrac{1}{3}, \nicefrac{1}{3}, \nicefrac{1}{3})$, the uniform policy on the simplex.}~\label{appfig:learning_curves_dpg_035_simplex}
\end{figure}

\begin{figure}[!ht]
\centering
  \begin{subfigure}[b]{0.326\linewidth}
    \includegraphics[width=\textwidth]{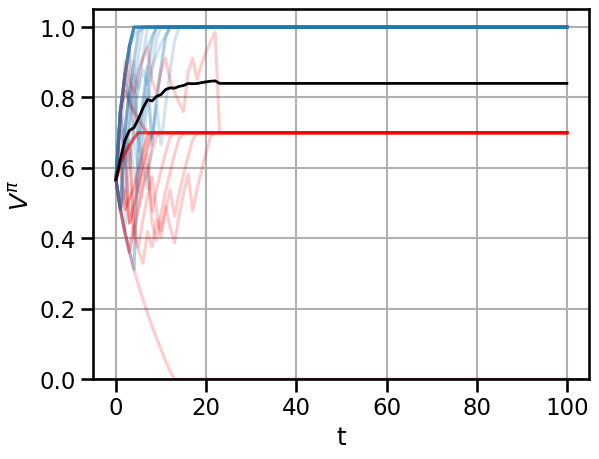}
    \caption{$b = - \nicefrac{1}{2}$}
  \end{subfigure}
    \begin{subfigure}[b]{0.326\linewidth}
    \includegraphics[width=\textwidth]{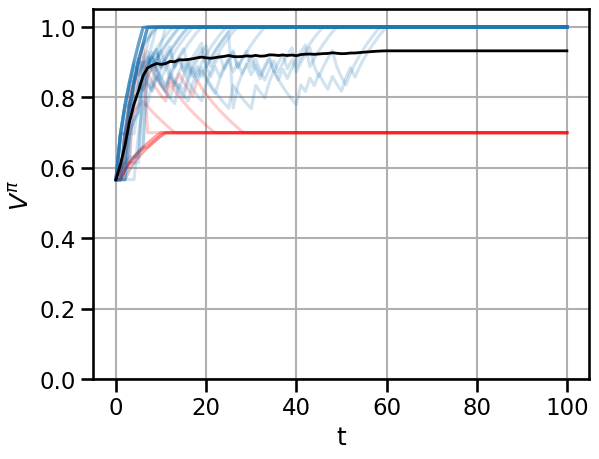}
    \caption{$b = 0$}
  \end{subfigure}
  \begin{subfigure}[b]{0.326\linewidth}
    \includegraphics[width=\textwidth]{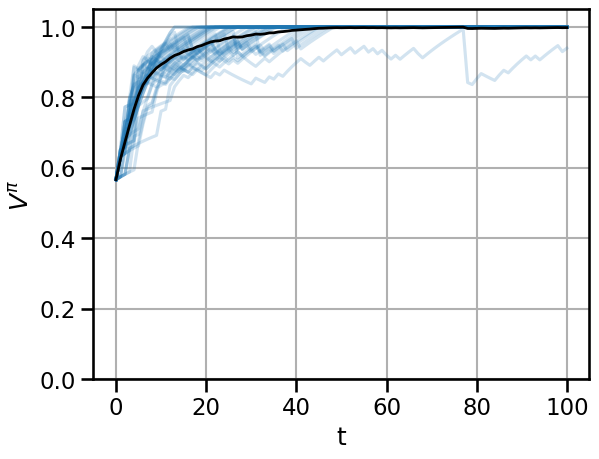}
    \caption{$b = \nicefrac{1}{2}$}
  \end{subfigure}
  \caption{We plot 40 different learning curves (in blue and red) of vanilla policy gradient with direct parameterization, when using various baselines, on a 3-arm bandit problem with rewards $(1, 0.7, 0)$, $\alpha = 0.1$ and $\theta_0 = (\nicefrac{1}{3}, \nicefrac{1}{3}, \nicefrac{1}{3})$, the uniform policy. The black line is the average value over the 40 seeds for each setting. The red curves denote the seeds that did not reach a value of at least $0.9$ at the end of training.}~\label{appfig:learning_curves_dpg_035}
\end{figure}

Once again in this setting we can see that negative baselines tend to encourage convergence to a suboptimal arm while positive baselines help converge to the optimal arm.

\newpage %
\subsection{Simple gridworld}
\label{app:simple_mdp}
As a simple MDP with more than one state, we experiment using a 5x5 gridworld with two goal states, the closer one giving a reward of 0.8 and the further one a reward of 1. We ran the vanilla  policy gradient with a fixed stepsize and discount factor of $0.99$ multiple times for several baselines. 
Fig.~\ref{fig:mdp5x5} displays individual learning curves with the index of the episode on the x-axis, and the fraction of episodes where the agent reached the reward of 1 up to that point on the  y-axis.
To match the experiments for the four rooms domain in the main text, Fig.~\ref{fig:5x5_stats} shows returns and the entropy of the actions and state visitation distributions for multiple settings of the baseline. 
Once again, we see a difference between the smaller and larger baselines. In fact, the difference is more striking in this example since some learning curves get stuck at suboptimal policies. Overall, we see two main trends in this experiment: a) The larger the baseline, the more likely the agent converges to the optimal policy, and b) Agents with negative baselines converge faster, albeit sometimes to a suboptimal behaviour.  We emphasize that a) is not universally true and large enough baselines will lead to an increase in variance and a decrease in performance.

\begin{figure}[ht]
\centering 
  \begin{subfigure}[b]{0.2\linewidth}
    \includegraphics[trim={0cm -0.75cm -0.1cm 0},clip,width=\textwidth]{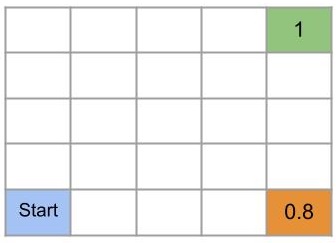}
    \caption{MDP used}
    \label{fig:gridworld}
  \end{subfigure}
  \begin{subfigure}[b]{0.26\linewidth}
    \includegraphics[trim={0cm 0cm 0cm 0},clip,width=\textwidth]{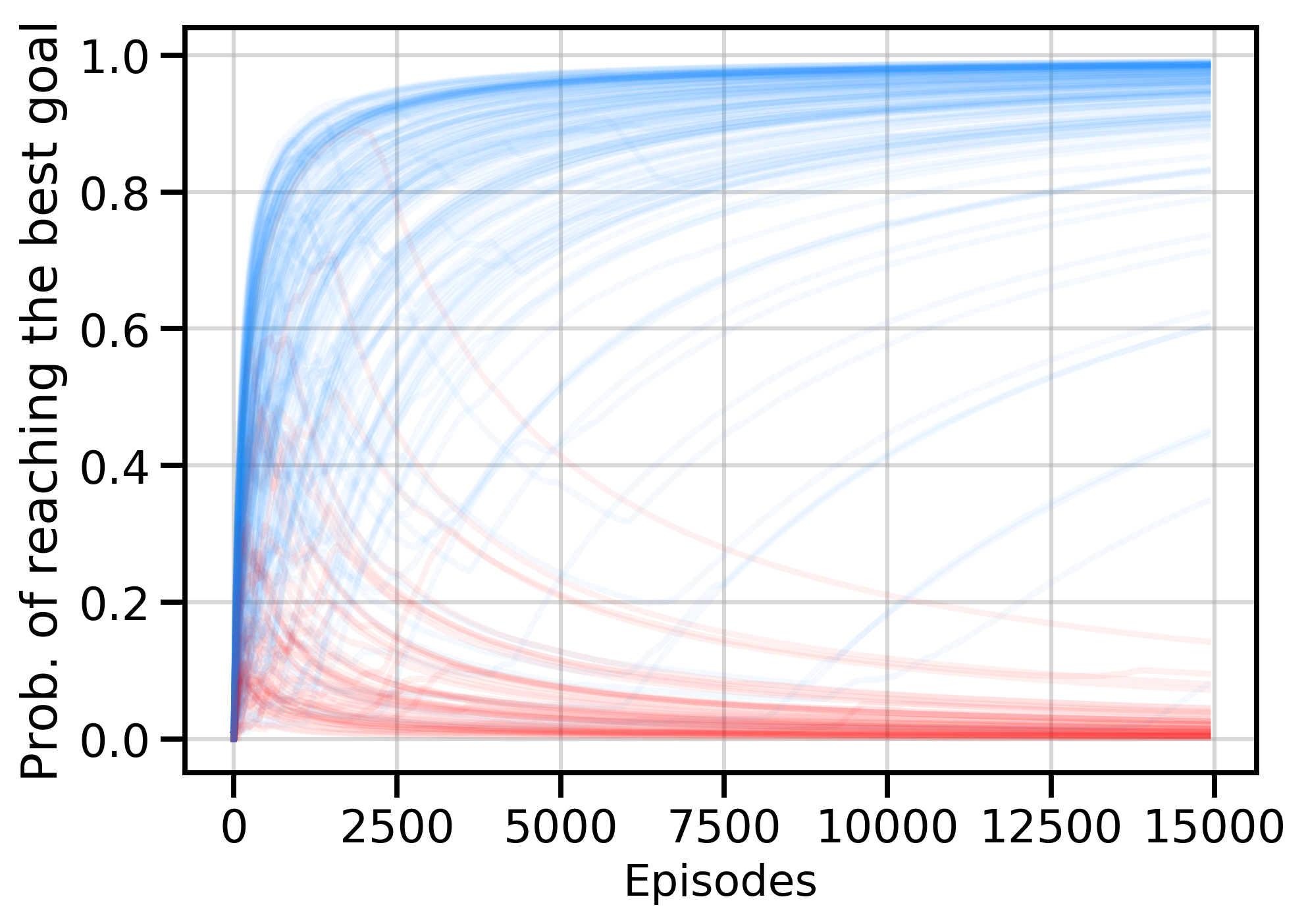}
    \caption{$b=-1$}
    \label{fig:mdp-1}
  \end{subfigure}
  \begin{subfigure}[b]{0.26\linewidth}
    \includegraphics[trim={0cm 0cm 0cm 0},clip,width=\textwidth]{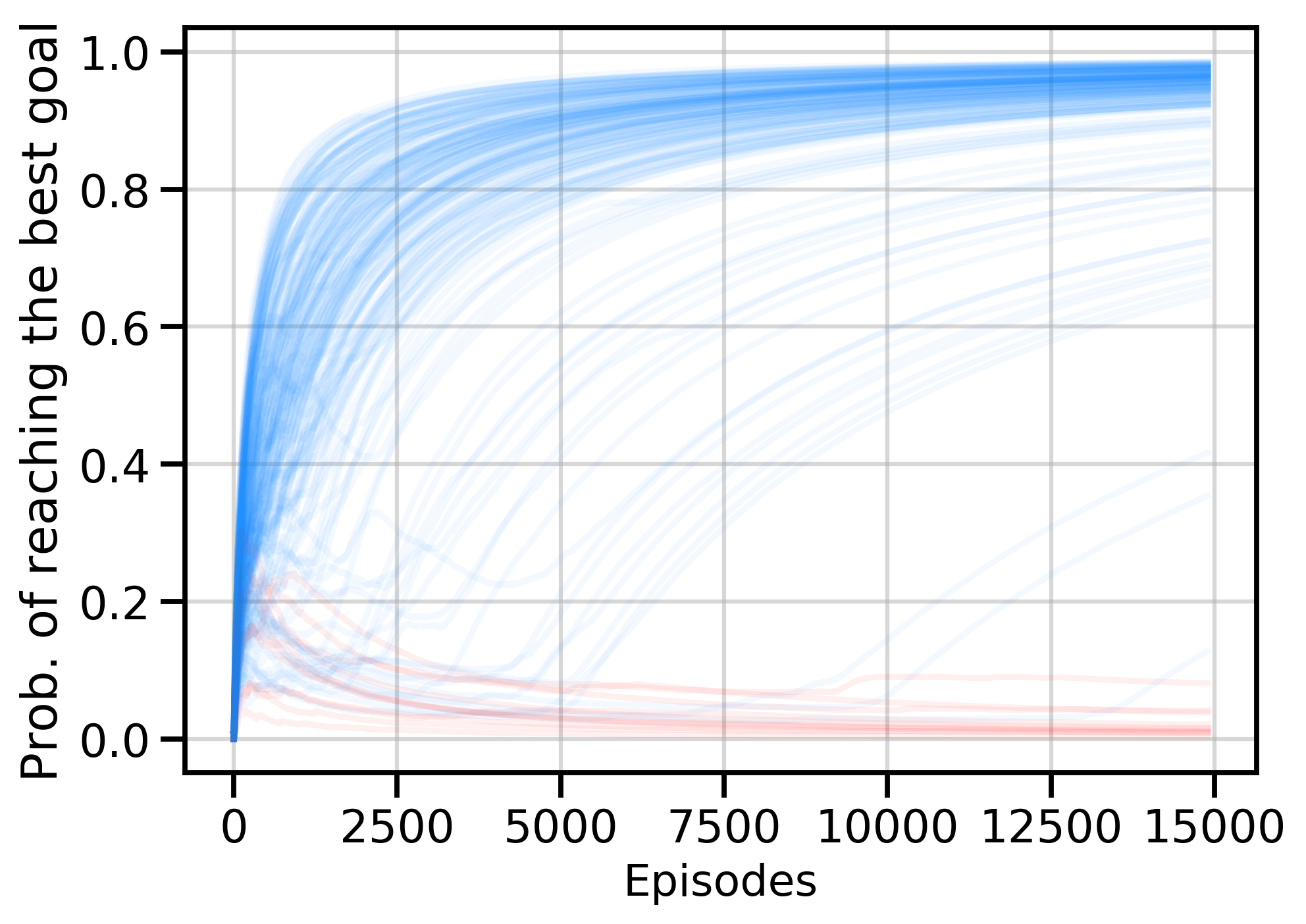}
    \caption{$b=0$}
    \label{fig:mdp0}
  \end{subfigure}
    \begin{subfigure}[b]{0.26\linewidth}
    \includegraphics[trim={0cm 0cm 0cm 0},clip,width=\textwidth]{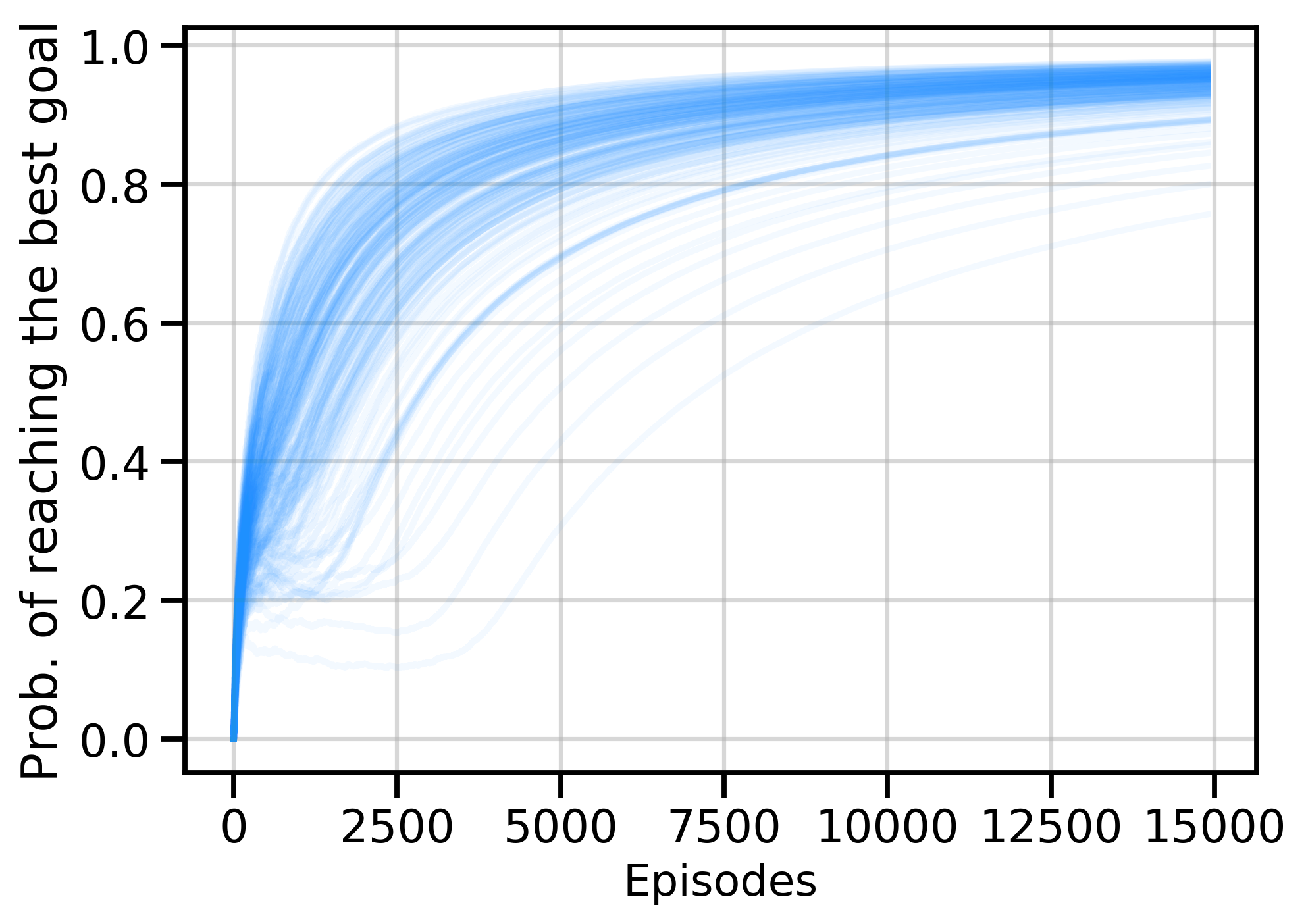}
    \caption{$b=1$}
    \label{fig:mdp1}
  \end{subfigure}

  \caption{Learning curves for a 5x5 gridworld with two goal states where the further goal is optimal.
  Trajectories in red do not converge to an optimal policy.
  }
  \label{fig:mdp5x5}
\end{figure}

\begin{figure}[!ht]
\centering
  \begin{subfigure}[b]{0.32\linewidth}
    \includegraphics[width=\textwidth]{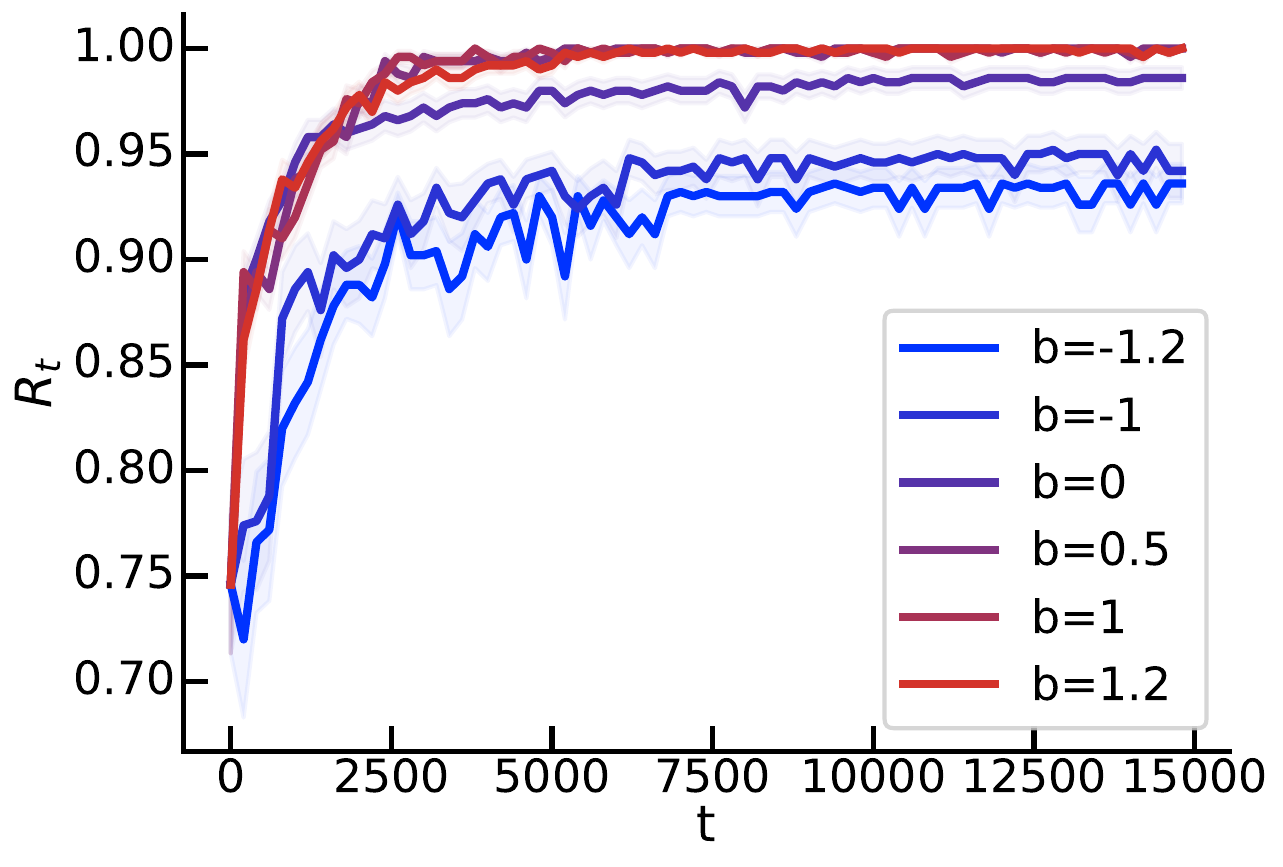}
    \caption{Returns}
    \label{fig:5x5_return}
  \end{subfigure}
    \begin{subfigure}[b]{0.32\linewidth}
    \includegraphics[width=\textwidth]{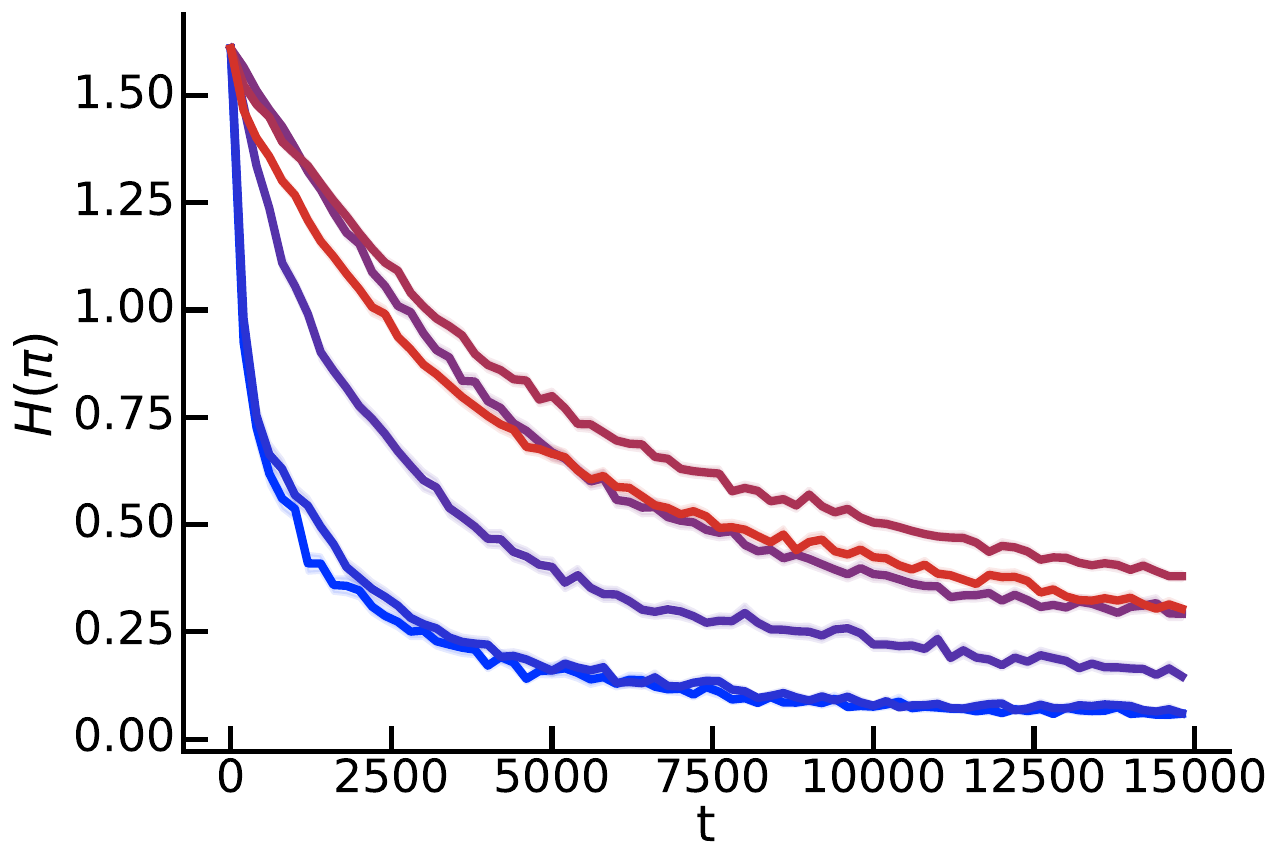}
    \caption{Action entropy}
    \label{fig:5x5_action}
  \end{subfigure}
  \begin{subfigure}[b]{0.32\linewidth}
    \includegraphics[width=\textwidth]{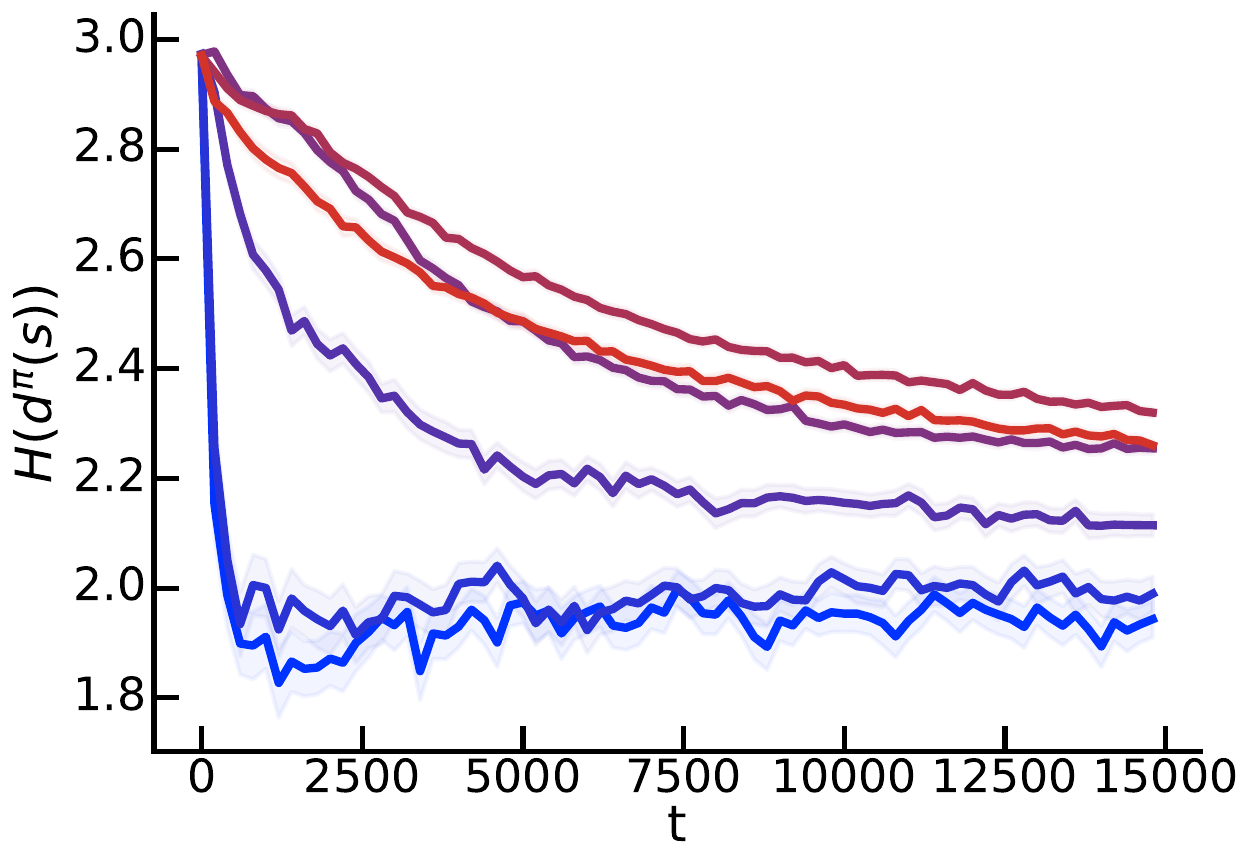}
    \caption{State visitation entropy}
    \label{fig:5x5_state}
  \end{subfigure}
 
  \caption{We plot the returns, the entropy of the policy over the states visited in each trajectory, and the entropy of the state visitation distribution averaged over 100 runs for multiple baselines for the 5x5 gridworld. The shaded regions denote one standard error and are close to the mean curve. Similar to the four rooms, the policy entropy of lower baselines tends to decay faster than for larger baselines, and smaller baselines tend to get stuck on suboptimal policies, as indicated by the returns~plot.~\label{fig:5x5_stats}}
\end{figure}

\subsection{Additional results on the 4 rooms environment}
\label{app:4rooms_mdp}
For the four-rooms gridworld discussed in the main text, we extend the experiments and provide additional details.
The environment is a 10x10 gridworld consisting of 4 rooms as depicted on Fig.~\ref{fig:4rooms_mdp} with a discount factor $\gamma=0.99$. The agent starts in the upper left room and two adjacent rooms contain a goal state of value $0.6$ (discounted, $\approx0.54$) or $0.3$ (discounted, $\approx0.27$). However, the best goal, with a value of $1$ (discounted, $\approx0.87$), lies in the furthest room, so that the agent must learn to cross the sub-optimal rooms and reach the furthest one. 

For the NPG algorithm used in the main text, we required solving for $Q_\pi(s,a)$ for the current policy $\pi$. This was done using dynamic programming on the true MDP, stopping when the change between successive approximations of the value function didn't differ more than $0.001$. 
Additionally, a more thorough derivation of the NPG estimate we use can be found in Appendix \ref{app:npg_mdp_estimate}.

We also experiment with using the vanilla policy gradient with the tabular softmax parameterization in the four-rooms environment. We use a similar estimator of the policy gradient which makes updates of the form: 
\[ \theta \xleftarrow{} \theta + \alpha (Q_{\pi_\theta} (s_i, a_i) - b) \nabla \log \pi_\theta (a_i|s_i) \]
for all observed $s_i, a_i$ in the sampled trajectory. 
As with the NPG estimator, we can find the minimum-variance baseline $b^*_\theta$ in closed-form and thus can choose baselines of the form $b^+ = b^*_\theta + \epsilon$ and $b^-_\theta = b^*_\theta - \epsilon$ to ensure equal variance as before. 
Fig. \ref{appfig:4rooms_vpg_absolute_perturb} plots the results. In this case, we find that there is not a large difference between the results for $+\epsilon$ and $-\epsilon$, unlike the results for NPG and those for vanilla PG in the bandit setting. 

The reason for this discrepancy may be due to the magnitudes of the perturbations $\epsilon$ relative to the size of the unperturbed update $Q_\pi(s_i, a_i) - b^*_\theta$. The magnitude of $Q_\pi(s_i, a_i) - b^*$ varies largely from the order of $0.001$ to $0.1$, even within an episode.  To investigate this further, we try another experiment using perturbations $\epsilon = c (\max_a Q_\pi(s_i,a) - b^*_\theta$ for various choices of $c > 0$. This would ensure that the magnitude of the perturbation is similar to the magnitude of $Q_\pi(s_i, a_i) - b^*$, while still controlling for the variance of the gradient estimates. 
In Fig. \ref{appfig:4rooms_vpg_relative_perturb}, we see that there is a difference between the $+\epsilon$ and $-\epsilon$ settings. As expected, the $+\epsilon$ baseline leads to larger action and state entropy although, in this case, this results in a reduction of performance. 
Overall, the differences between vanilla PG and natural PG are not fully understood and there may be many factors playing a role, possibly including the size of the updates, step sizes and the properties of the MDP. 

\begin{figure}[!ht]
\centering
  \begin{subfigure}[b]{0.32\linewidth}
    \includegraphics[width=\textwidth]{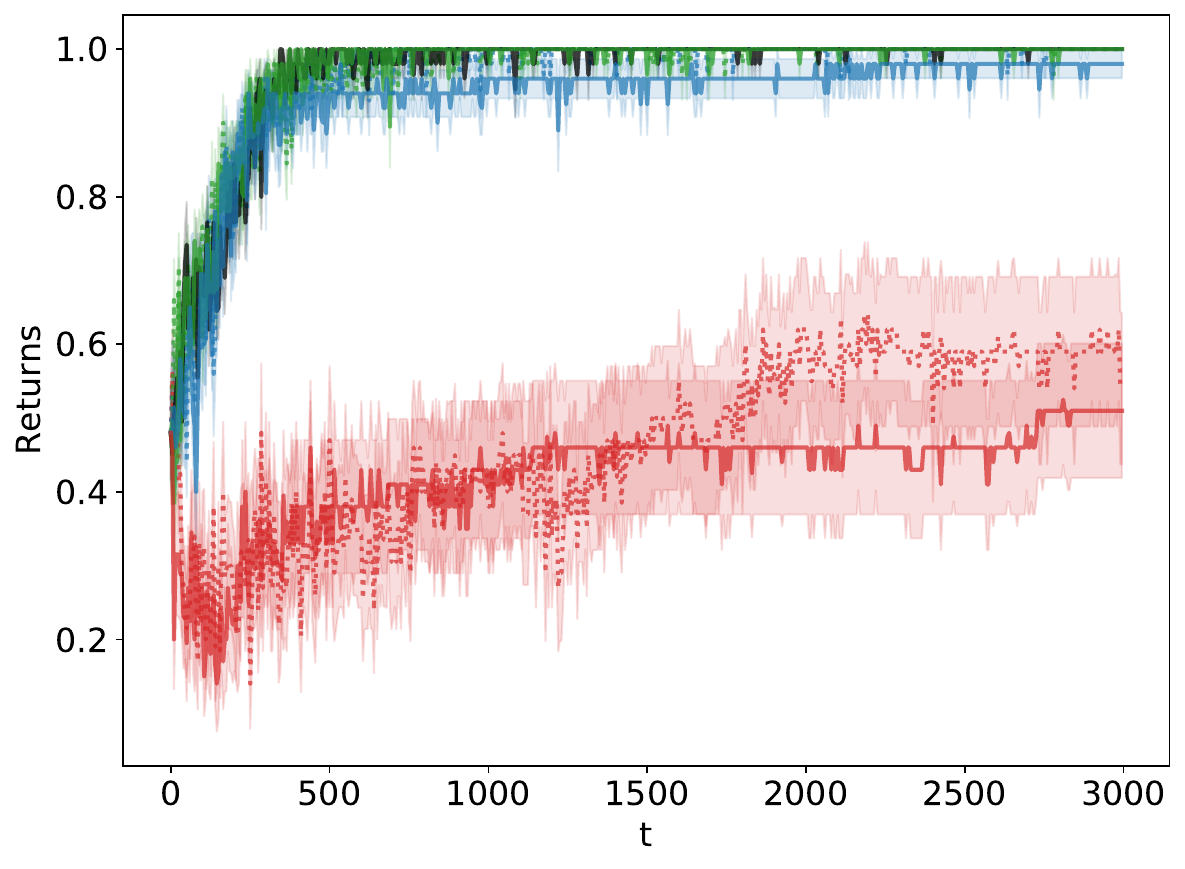}
    \caption{Returns}
  \end{subfigure}
    \begin{subfigure}[b]{0.32\linewidth}
    \includegraphics[width=\textwidth]{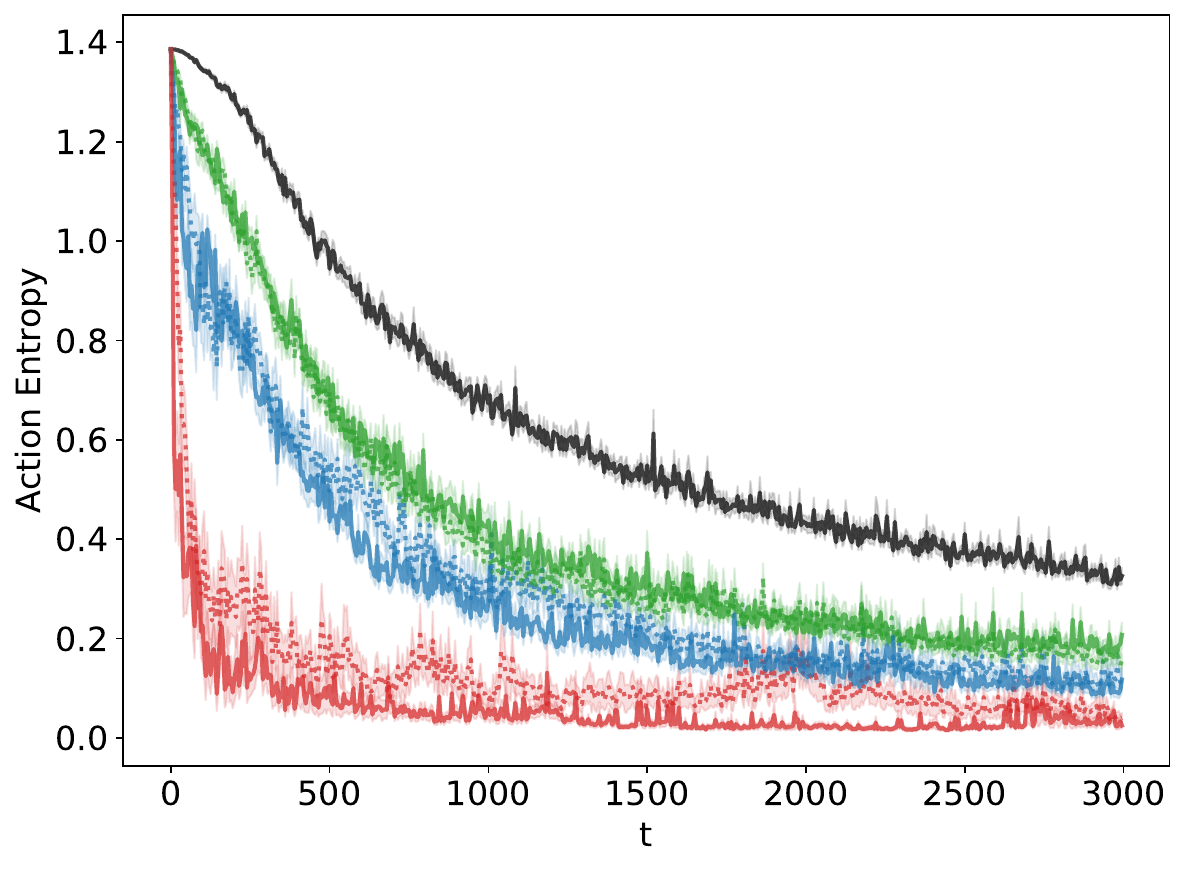}
    \caption{Action entropy}
  \end{subfigure}
  \begin{subfigure}[b]{0.32\linewidth}
    \includegraphics[width=\textwidth]{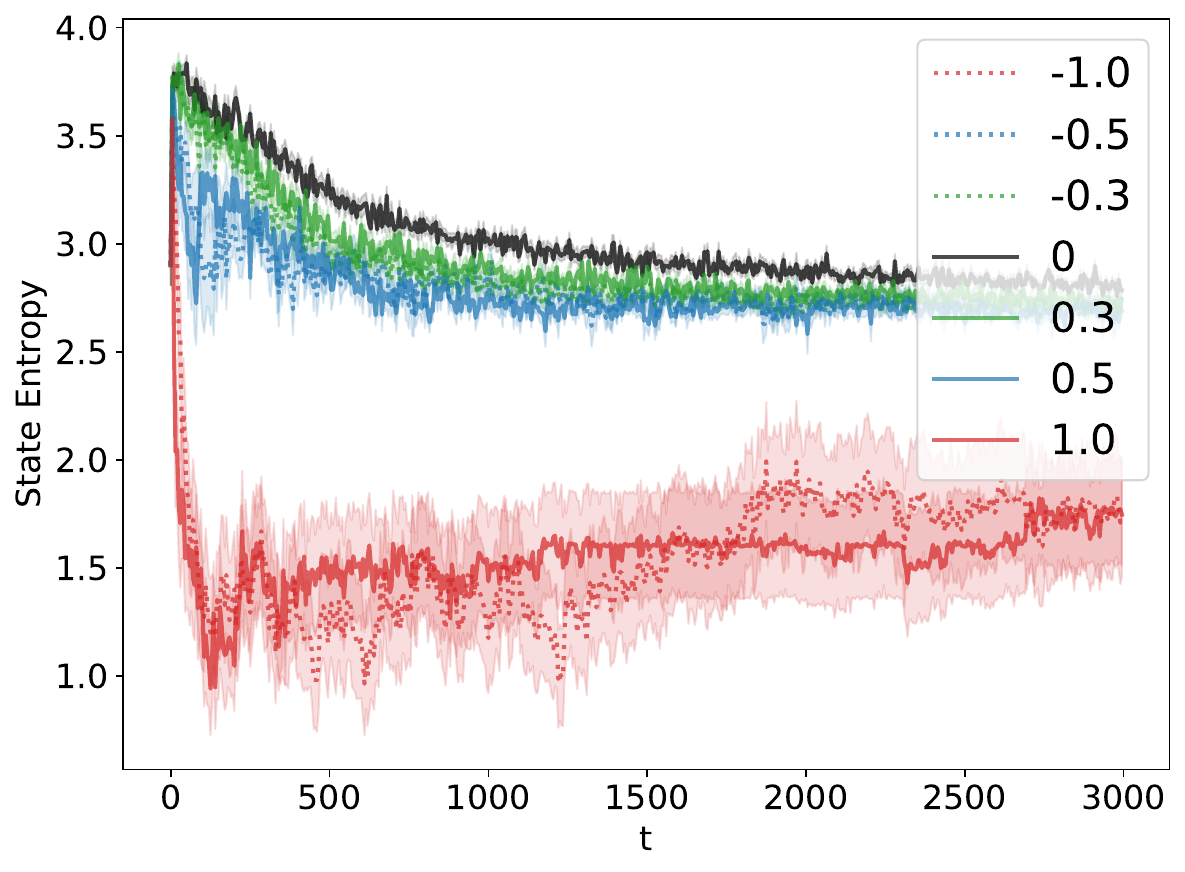}
    \caption{State visitation entropy}
  \end{subfigure}
  \caption{We plot results for vanilla policy gradient with perturbed minimum-variance baselines of the form $b^\ast_\theta + \epsilon$, with $\epsilon$ denoted in the legend. The step size is 0.5 and 20 runs are done. We see smaller differences between positive and negative $\epsilon$ values. ~\label{appfig:4rooms_vpg_relative_perturb}}
\end{figure}

\begin{figure}[!ht]
\centering
  \begin{subfigure}[b]{0.32\linewidth}
    \includegraphics[width=\textwidth]{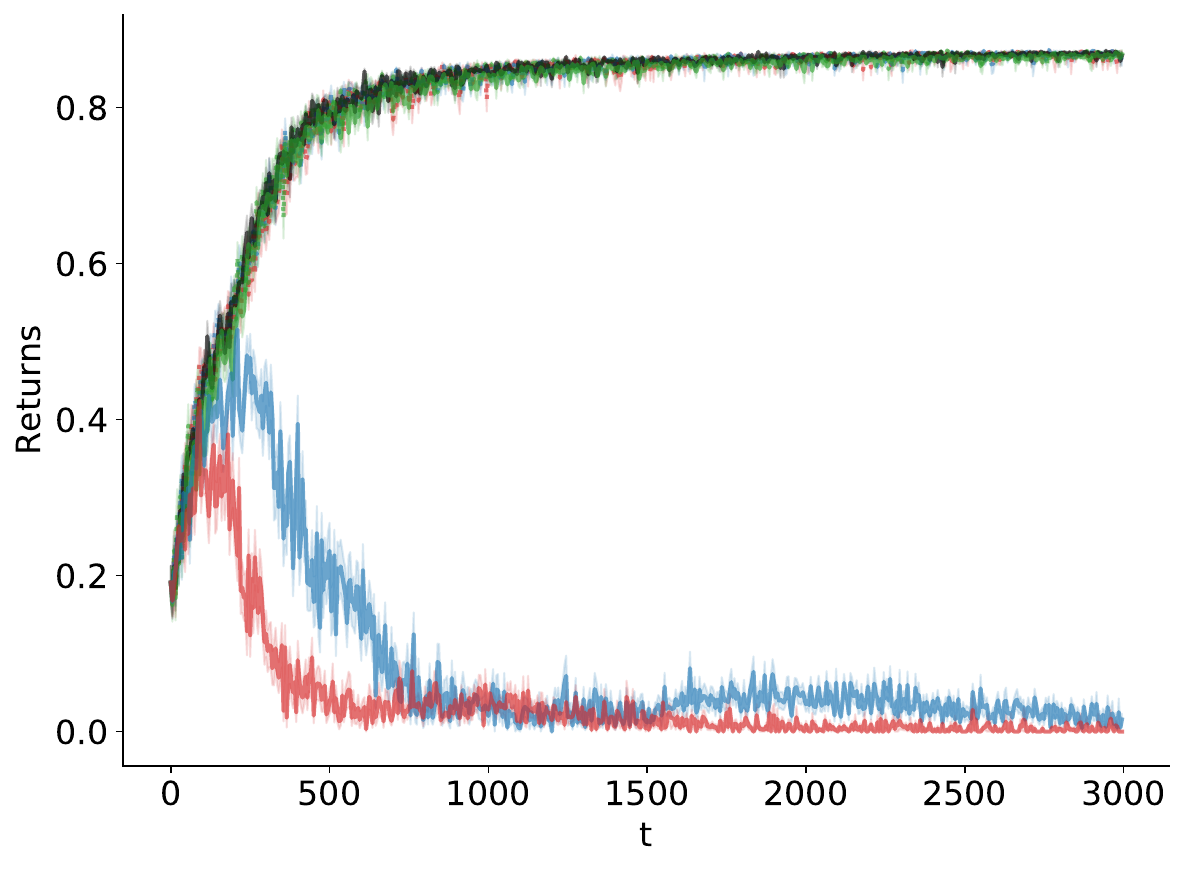}
    \caption{Returns}
  \end{subfigure}
    \begin{subfigure}[b]{0.32\linewidth}
    \includegraphics[width=\textwidth]{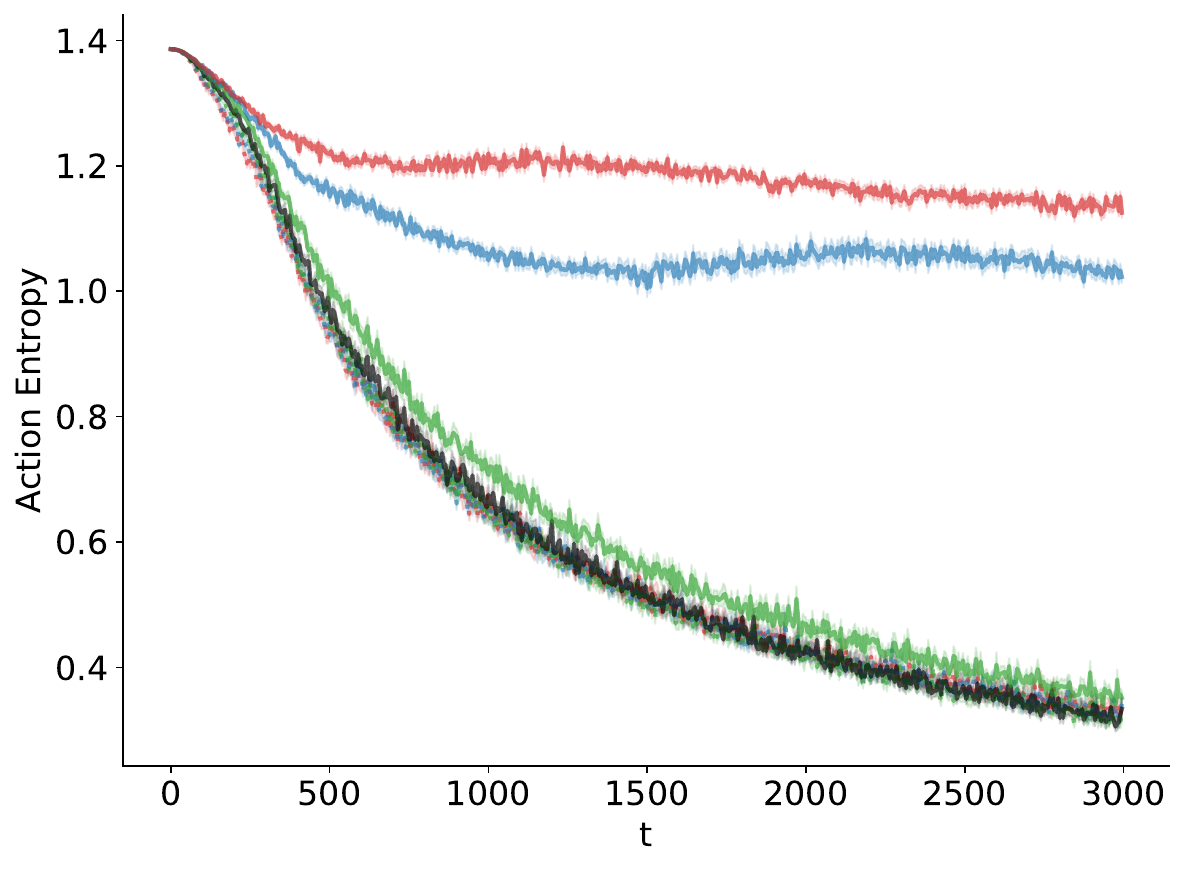}
    \caption{Action entropy}
  \end{subfigure}
  \begin{subfigure}[b]{0.32\linewidth}
    \includegraphics[width=\textwidth]{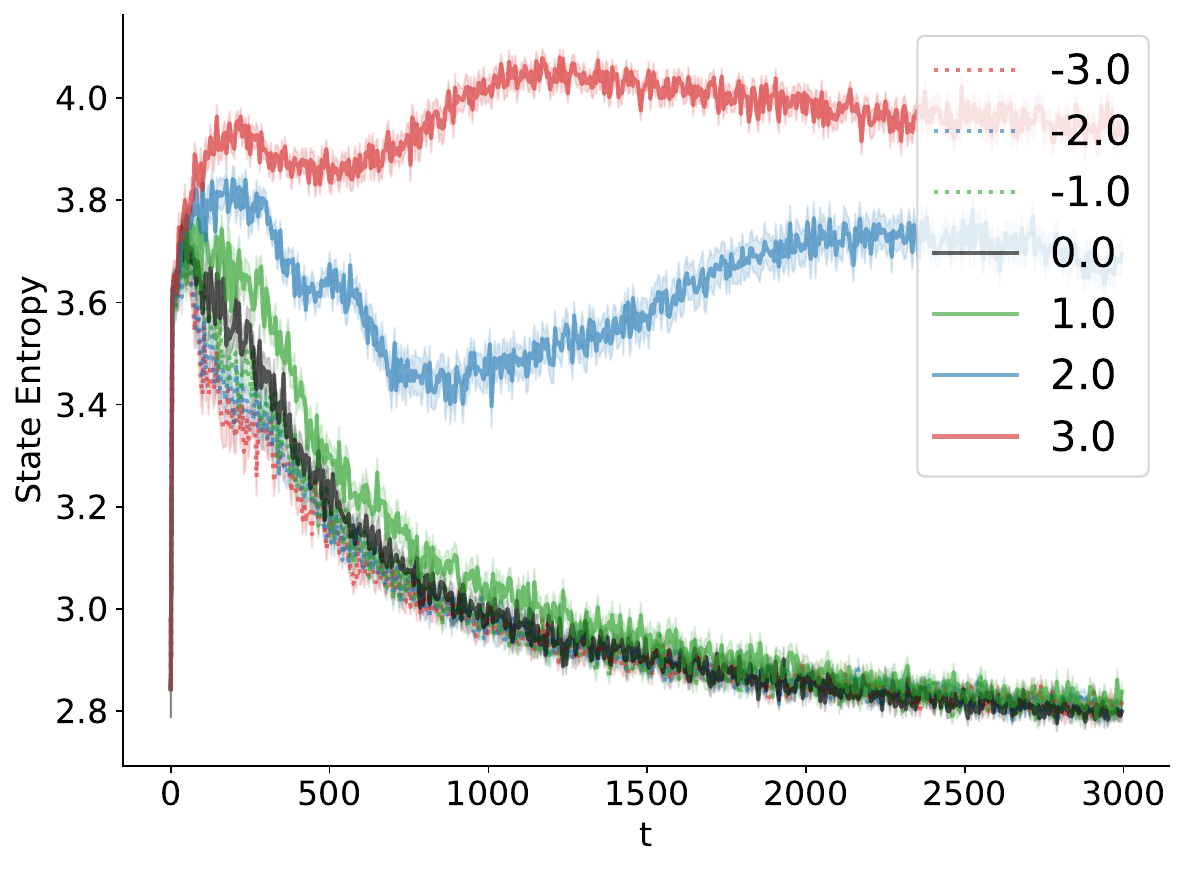}
    \caption{State visitation entropy}
  \end{subfigure}
  \caption{We plot results for vanilla policy gradient with perturbed minimum-variance baselines of the form $b^\ast_\theta + \epsilon$, where $\epsilon = c (\max_a Q_\pi(s_i,a) - b^*_\theta$ and $c$ is denoted in the legend. For a fixed $c$, we can observe a difference between the learning curves for the $+c$ and $-c$ settings. The step size is 0.5 and 50 runs are done. As expected, the action and state entropy for the positive settings of $c$ are larger than for the negative settings. In this case, this increased entropy does not translate to larger returns though and is a detriment to performance,  ~\label{appfig:4rooms_vpg_absolute_perturb}}
\end{figure}

Finally, we also experiment with the vanilla REINFORCE estimator with softmax parameterization where the estimated gradient for a trajectory is $(R(\tau_i) -b ) \nabla \log \pi(\tau_i)$ for $\tau_i$ being a trajectory of state, actions and rewards for an episode. For the REINFORCE estimator, it is difficult to compute the minimum-variance baseline so, instead, we utilize constant baselines.
Although we cannot ensure that the variance of the various baselines are the same, we could still expect to observe committal and non-committal behaviour depending on the sign of $R(\tau_i) -b$.  We use a step size of $0.1$.

\begin{figure}[!ht]
\centering
  \begin{subfigure}[b]{0.32\linewidth}
    \includegraphics[width=\textwidth]{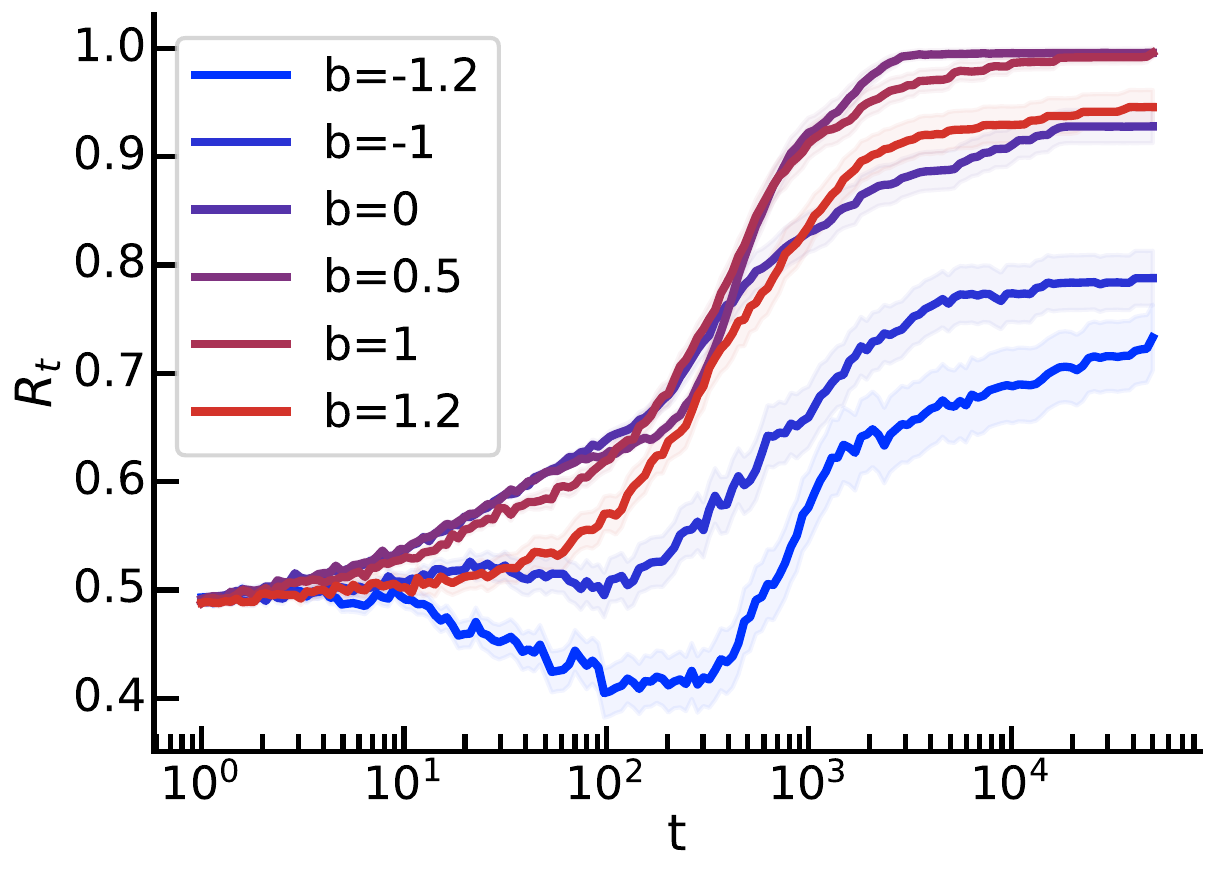}
    \caption{Returns}
  \end{subfigure}
    \begin{subfigure}[b]{0.32\linewidth}
    \includegraphics[width=\textwidth]{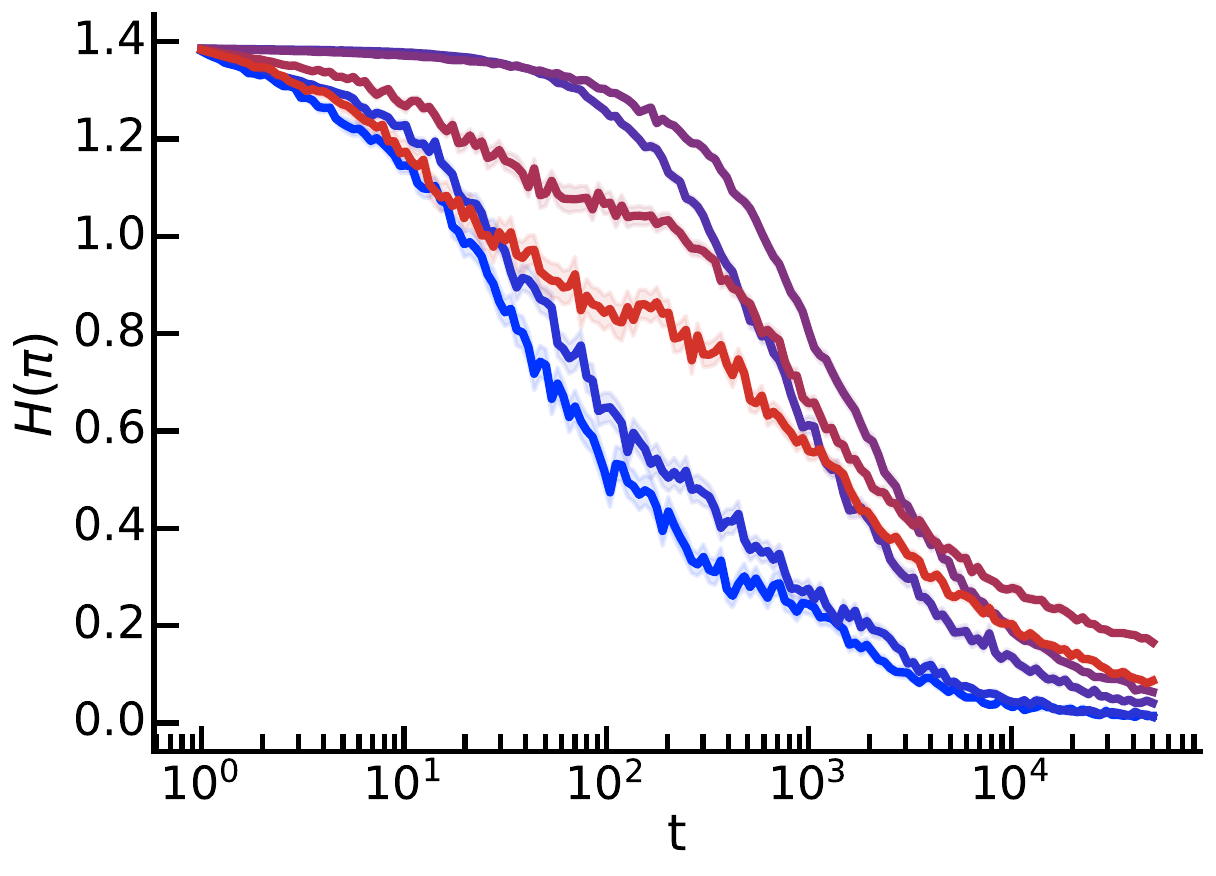}
    \caption{Action entropy}
  \end{subfigure}
  \begin{subfigure}[b]{0.32\linewidth}
    \includegraphics[width=\textwidth]{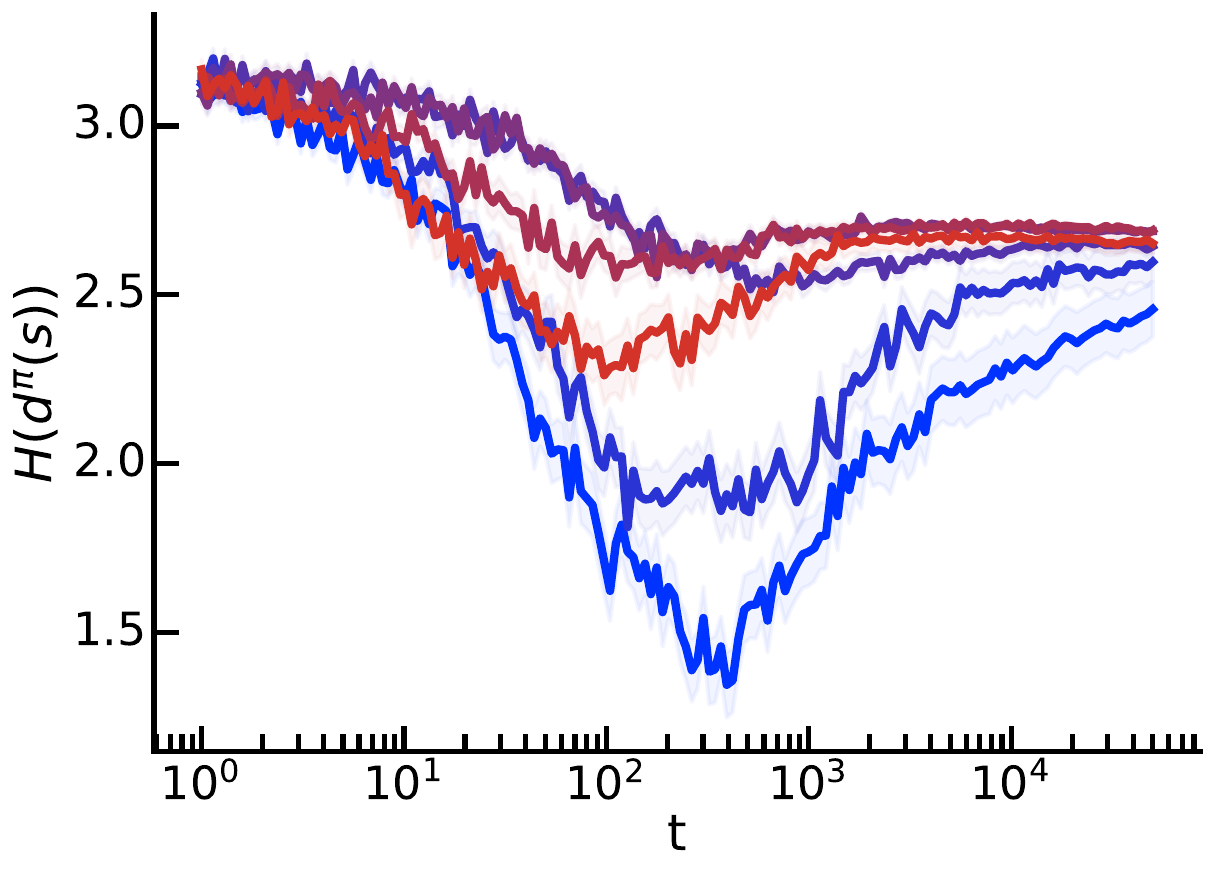}
    \caption{State visitation entropy}
  \end{subfigure}
  \caption{We plot the results for using REINFORCE with constant baselines. Once again, the policy entropy of lower baselines tends to decay faster than for larger baselines, and smaller baselines tend to get stuck on suboptimal policies, as indicated by the returns~plot.~\label{appfig:4rooms_vpg_constant_baselines}}
\end{figure}

We consider an alternative visualization for the experiment of vanilla policy gradient with constant baselines: Figures~\ref{fig:4rooms_m1},~\ref{fig:4rooms_0} and~\ref{fig:4rooms_1}. Each point in the simplex is a policy, and the position is an estimate, computed with $1,000$ Monte-Carlo samples, of the probability of the agent reaching each of the 3 goals. We observe that the starting point of the curve is equidistant to the 2 sub-optimal goals but further from the best goal, which is coherent with the geometry of the MDP. Because we have a discount factor of $\gamma = 0.99$, the agent first learns to reach the best goal in an adjacent room to the starting one, and only then it learns to reach the globally optimal goal fast enough for its reward to be the best one. 

In these plots, we can see differences between $b=-1$ and $b=1$. For the lower baseline, we see that trajectories are much more noisy, with some curves going closer to the bottom-right corner, corresponding to the worst goal. This may suggest that the policies exhibit committal behaviour by moving further towards bad policies. On the other hand, for $b=1$, every trajectory seems to reliably move towards the top corner before converging to the bottom-left, an optimal policy.

\begin{figure}[!ht]
\centering
  \begin{subfigure}[b]{0.32\linewidth}
    \includegraphics[width=\textwidth]{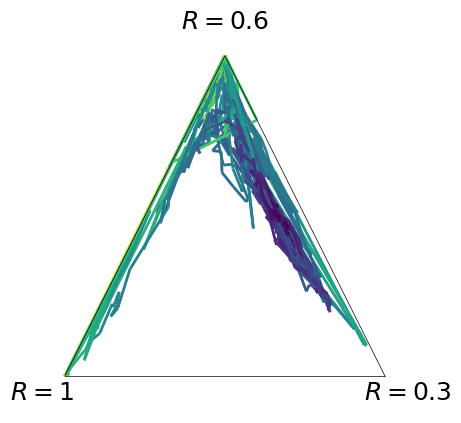}
    \caption{$b=-1$}
    \label{fig:4rooms_m1}
  \end{subfigure}
    \begin{subfigure}[b]{0.32\linewidth}
    \includegraphics[width=\textwidth]{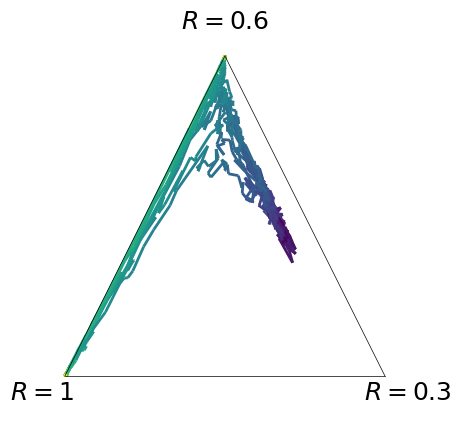}
    \caption{$b=0$}
    \label{fig:4rooms_0}
  \end{subfigure}
  \begin{subfigure}[b]{0.32\linewidth}
    \includegraphics[width=\textwidth]{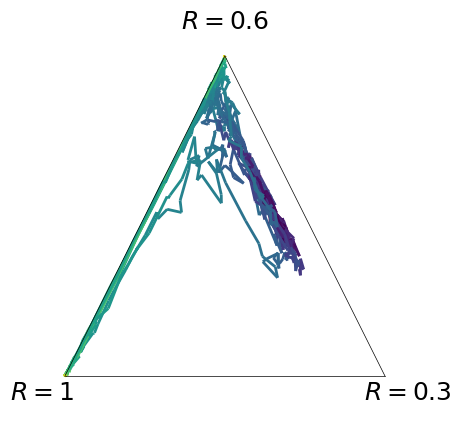}
    \caption{$b=1$}
    \label{fig:4rooms_1}
  \end{subfigure}

  \caption{We plot 10 different trajectories of vanilla policy gradient (REINFORCE) using different constant on a 4 rooms MDP with goal rewards $(1, 0.6, 0.3)$. The color of each trajectory represents time and each point of the simplex represents the probability that a policy reaches one of the 3 goals.~\label{fig:trajectories_4rooms}}
\end{figure}

\section{Two-armed bandit theory}
\label{app:theory_2arm}
In this section, we expand on the results for the two-armed bandit. First, we show that there is some probability of converging to the wrong policy when using natural policy gradient with a constant baseline.
Next, we consider all cases of the perturbed minimum-variance baseline ($b = b^* + \epsilon)$ and show that some cases lead to convergence to the optimal policy with probability 1 while others do not. In particular there is a difference between $\epsilon < -1 $ and $\epsilon > 1$, even though these settings can result in the same variance of the gradient estimates. 
Finally, we prove that the vanilla policy gradient results in convergence in probability to the optimal policy regardless of the baseline, in contrast to the natural policy gradient. 

\paragraph{Notations:}
\begin{itemize}
    \item Our objective is $J(\theta) = \E_{\pi_\theta} [R(\tau)]$, the expected reward for current parameter $\theta$.
    \item $p_t = \sigma(\theta_t)$ is the probability of sampling the optimal arm (arm 1).
    \item $P_1$ is the distribution over rewards than can be obtained from pulling arm 1. Its expected value is $\mu_1 = \E_{r_1 \sim P_1}[r_1]$. Respectively $P_0, \mu_0$ for the suboptimal arm.
    \item $g_t$ is a stochastic unbiased estimate of $\nabla_\theta J(\theta_t)$. It will take different forms depending on whether we use vanilla or natural policy gradient and whether we use importance sampling or not.
    \item For $\{ \alpha_t\}_t$ the sequence of stepsizes, the current parameter $\theta_t$ is a random variable equal to $\theta_t = \sum_{i=1}^t \alpha_i g_i + \theta_0$ where $\theta_0$ is the initial parameter value.
\end{itemize}

For many convergence proofs, we will use the fact that the sequence $\theta_t - \E[\theta_t]$ forms a martingale. In other words, the noise around the expected value is a martingale, which we define below.
\begin{definition}[Martingale] A discrete-time martingale is a stochastic process $\{X_t\}_{t\in \sN}$ such that
\begin{itemize}
    \item $\E[|X_t|]  <  +\infty$
    \item $\E[X_{t+1} | X_t, \dots X_0] = X_t$
\end{itemize}

\end{definition}
\begin{example} For $g_t$ a stochastic estimate of $\nabla J(\theta_t)$ we have
$X_t = \E[\theta_t] - \theta_t$  is a martingale. As $\theta_t = \theta_0 + \sum_i \alpha_i g_i$, $X_t$ can also be rewritten as
$X_t = \E[\theta_t - \theta_0] - (\theta_t-\theta_0) = \sum_{i=0}^{t} \alpha_i \big(\E[ g_i | \theta_0] -  g_i \big)$.%
\end{example}

We will also be making use of Azuma-Hoeffding's inequality to show that the iterates stay within a certain region with high-probability, leading to convergence to the optimal policy.
\begin{lemma}[Azuma-Hoeffding's inequality]
For $\{ X_t \}$ a martingale, if
$|X_t - X_{t-1}| \le c_t$ almost surely, then we have $\forall t, \epsilon \ge 0$
$$\sP(X_t - X_0 \ge \epsilon) \le \exp\Bigg( -\frac{\epsilon^2}{2\sum_{i=1}^t c_i^2}\Bigg)$$
\end{lemma}

\subsection{Convergence to a suboptimal policy with a constant baseline}
\label{app:2arm_constant_baseline_div}
For the proofs in this subsection, we assume that the step size is constant i.e. $\alpha_t = \alpha$ for all $t$ and that the rewards are deterministic. 

\divtwoarms*

\begin{proof}
First, we deal with the case where $\theta_0 < 0$.
\begin{align*}
    1 - \sigma(\theta_0 - \alpha b t) &\ge 1 - \exp (\theta_0 - \alpha b t) \\
\end{align*}

Next, we use the bound $1-x \ge \exp (\frac{-x}{1-x} )$.
This bound can be derived as follows:
\begin{align*}
1 - u &\le e^{-u} \\ 
1 - e^{-u} &\le u \\
1 - \frac{1}{y} &\le \log y, \quad \text{substitute } u = \log y \text { for } y > 0 \\
\frac{-x}{1-x} &\le \log (1-x), \quad \text{substitute } y = 1-x \text{ for } x \in [0, 1) \\
\exp \left( \frac{-x}{1-x} \right) & \le 1 - x.
\end{align*}

Continuing with $x = \exp(\theta_0 - \alpha b t)$, the bound holds when $x \in [0,1)$, which is satisfied assuming $\theta_0 \le 0$.

\begin{align*}
    1 - \sigma(\theta_0 - \alpha b t) &\ge \exp \left(\frac{-1}{e^{-\theta_0 + \alpha b t} - 1} \right)\\
\end{align*}

For now we ignore $t=0$ and we will just multiply it back in at the end.
\begin{align*}
    \prod_{t=1}^\infty [1 - \sigma(\theta_0 - \alpha b t)] &\ge \prod_{t=1}^\infty \exp \left(\frac{-1}{e^{-\theta_0 + \alpha b t} - 1} \right) \\
    &=  \exp \sum_{t=1}^\infty \left(\frac{-1}{e^{-\theta_0 + \alpha b t} - 1} \right) \\
    &\ge \exp \left( - \int_{t=1}^\infty \frac{1}{e^{-\theta_0 + \alpha b t} - 1} dt \right)
\end{align*}
The last line follows by considering the integrand as the right endpoints of rectangles approximating the area above the curve.

Solving this integral by substituting $y = -\theta_0 + \alpha b t$, multiplying the numerator and denominator by $e^y$ and substituting $u = e^y$, we get:
\begin{align*}
    &= \exp \left( \frac{1}{\alpha b} \log (1 - e^{\theta_0-\alpha b}) \right) \\ 
    &= \left(1- e^{\theta_0 -\alpha b} \right)^{\frac{1}{\alpha b}} 
\end{align*}

Finally we have:

\[ P(\text{left forever}) \ge (1-e^{\theta_0}) (1-e^{\theta_0 - \alpha b})^\frac{1}{\alpha b}  \]

If $\theta_0 > 0$, then there is a positive probability of reaching $\theta < 0$ in a finite number of steps since choosing action 2 makes a step of size $\alpha b$ in the left direction and we will reach $\theta_t < 0$ after $m = \frac{\theta_0 - 0}{\alpha b}$ steps leftwards. The probability of making $m$ left steps in a row is positive. 
So, we can simply lower bound the probability of picking left forever by the product of that probability and the derived bound for $\theta_0 \le 0$.
\end{proof}

\begin{corollary} \label{cor:twoarm_linear_regret}
The regret for the previously described two-armed bandit is linear.
\end{corollary}
\begin{proof}
Letting $R_t$ be the reward collected at time $t$,
\begin{align*}
    Regret(T) &= \E \left[ \sum_{t=1}^T (1-b - R_t) \right] \\
    &\ge \sum_{t=1}^T 1 \times Pr(\text{left $T$ times}) \\
    &\ge \sum_{t=1}^T P(\text{left forever}) \\
    &= T \times  P(\text{left forever}).
\end{align*}
The second line follows since choosing the left action at each step incurs a regret of $1$ and this is one term in the entire expectation.
The third line follows since choosing left $T$ times is a subset of the event of choosing left forever.
The last line implies linear regret since we know $Pr(\text{left forever}) > 0$ by the previous theorem.
\end{proof}

\subsection{Analysis of perturbed minimum-variance baseline} \label{sec:appendix_perturbed_minvar}

In this section, we look at perturbations of the minimum-variance baseline in the two-armed bandit, i.e. baselines of the form $b = 1-p_t + \epsilon$.
In summary:
\begin{itemize}
    \item For $\epsilon < -1$, convergence to a suboptimal policy is possible with positive probability.
    \item For $\epsilon \in (-1, 1)$, we have convergence almost surely to the optimal policy.
    \item For $\epsilon \ge 1$, the supremum of the iterates goes to $\infty$ (but we do not have convergence to an optimal policy)
\end{itemize}

It is interesting to note that there is a subtle difference between the case of $\epsilon \in (-1,0)$ and $\epsilon \in (0,1)$, even though both lead to convergence. The main difference is that when $\theta_t$ is large, positive $\epsilon$ leads to both updates being positive and hence improvement is guaranteed at every step. But, when $\epsilon$ is negative, then only one of the actions leads to improvement, the other gives a large negative update. So, in some sense, for $\epsilon \in (-1,0)$, convergence is less stable because a single bad update could be catastrophic.

Also, the case of $\epsilon=-1$ proved to be difficult. Empirically, we found that the agent would incur linear regret and it seemed like some learning curves also got stuck near $p=0$, but we were unable to theoretically show convergence to a suboptimal policy.
\\

\begin{lemma}
\label{lem:prop_epsilon_inf1}
For the two-armed bandit with sigmoid parameterization, natural policy gradient and a perturbed minimum-variance baseline $b = 1-p_t +\epsilon$, with $\epsilon < -1$, there is a positive probability of choosing the suboptimal arm forever and diverging.
\end{lemma}
\begin{proof}

We can reuse the result for the two-armed bandit with constant baseline $b < 0$.
Recall that for the proof to work, we only need $\theta$ to move by at least a constant step $\delta > 0$ in the negative direction at every iteration.

In detail, the update after picking the worst arm is 
$\theta_{t+1} = \theta_t + \alpha (1 + \frac{\epsilon}{1-p_t})$. So, if we choose $\epsilon < -1 - \delta $ for some $\delta > 0$, we get the update step magnitude is $\frac{\delta + p}{1 - p} > \delta$ and hence the previous result applies (replace $\alpha b$ by $\delta$).
\end{proof}

\begin{lemma}
\label{lem:prop_epsilon_10}
For the two-armed bandit with sigmoid parameterization, natural policy gradient and a perturbed minimum-variance baseline $b = 1-p_t +\epsilon$, with $\epsilon \in (-1,0)$, the policy converges to the optimal policy in probability.
\end{lemma}

\begin{proof}

Recall that the possible updates when the parameter is $\theta_t$ are:
\begin{itemize}
    \item $\theta_{t+1} = \theta_t + \alpha(1-\frac{\epsilon}{\sigma(\theta_t)})$ if we choose action 1, with probability $\sigma(\theta_t)$ 
    \item $\theta_{t+1} = \theta_t + \alpha(1+\frac{\epsilon}{1- \sigma(\theta_t)})$ if we choose action 2, with probability $1-\sigma(\theta_t)$.
\end{itemize}

First, we will partition the real line into three regions ($A$, $B$, and $C$ with $a < b < c$ for $a\in A, b \in B, c \in C$), depending on the values of the updates. Then, each region will be analyzed separately.

We give an overview of the argument first.
For region $A$ ($\theta$ very negative), both updates are positive so $\theta_t$ is guaranteed to increase until it reaches region $B$.

For region $C$ ($\theta$ very positive), sampling action 2 leads to the update $\alpha(1+\frac{\epsilon}{1- \sigma(\theta_t)})$, which has large magnitude and results in $\theta_{t+1}$ being back in region $A$. So, once $\theta_t$ is in $C$, the agent needs to sample action 1 forever to stay there and converge to the optimal policy. This will have positive probability (using the same argument as the divergence proof for the two-armed bandit with constant baseline). 

For region $B$, the middle region, updates to $\theta_t$ can make it either increase or decrease and stay in $B$. For this region, we will show that $\theta_t$ will eventually leave $B$ with probability 1 in a finite number of steps, with some lower-bounded probability of reaching $A$. 

Once we've established the behaviours in the three regions, we can argue that for any initial $\theta_0$ there is a positive probability that $\theta_t$ will eventually reach region $C$ and take action 1 forever to converge. In the event that does not occur, then $\theta_t$ will be sent back to $A$ and the agent gets another try at converging. Since we are looking at the behaviour when $t \xrightarrow{} \infty$, the agent effectively gets infinite tries at converging. Since each attempt has some positive probability of succeeding, convergence will eventually happen.

We now give additional details for each region.

To define region $A$, we check when both updates will be positive. The update from action 1 is always positive so we are only concerned with the second update. 
\begin{align*}
    1 + \frac{\epsilon}{1-p} &> 0 \\
    1 - p + \epsilon &> 0 \\
    1 + \epsilon &> p \\
    \sigma^{-1} (1+\epsilon) &> \theta
\end{align*}
Hence, we set $A = (-\infty, \sigma^{-1} (1+\epsilon))$.
Since every update in this region increases $\theta_t$ by at least a constant at every iteration, $\theta_t$ will leave $A$ in a finite number of steps.

For region $C$, we want to define it so that an update in the negative direction from any $\theta \in C$ will land back in $A$. So $C = [c, \infty)$ for some $c \ge \sigma^{-1} (1+\epsilon)$. 
By looking at the update from action 2, $\alpha(1 + \frac{\epsilon}{1-\sigma(\theta)})= \alpha(1 + \epsilon(1+e^\theta)) $, we see that it is equal to 0 at $\theta = \sigma^{-1} (1+\epsilon)$ but it is a decreasing function of $\theta$ and it decreases at an exponential rate. So, eventually for $\theta_t$ sufficiently large, adding this update will make $\theta_{t+1} \in A$.

So let $c = \inf \{\theta: \theta + \alpha \left( 1 - \frac{\epsilon}{1 - \sigma(\theta)} \right), \theta \ge \sigma^{-1} (1+\epsilon) \} $. 
Note that it is possible that $c = \sigma^{-1} (1+\epsilon)$. If this is the case, then region $B$ does not exist.

When $\theta_t \in C$, we know that there is a positive probability of choosing action 1 forever and thus converging (using the same proof as the two-armed bandit with constant baseline).

Finally, for the middle region $B = [a, c)$ ($a = \sigma^{-1} (1 + \epsilon)$), we know that the updates for any $\theta \in B$ are uniformly bounded in magnitude by a constant $u$. 

We define a stopping time $\tau = \inf \{t ; \theta_t \le a \text{ or } \theta_t \ge c \}$. This gives the first time $\theta_t$ exits the region $B$. Let ``$\land$'' denote the min operator.

Since the updates are bounded, we can apply Azuma's inequality to the stopped martingale $\theta_{t \land \tau} - \alpha(t \land \tau) $, for $\lambda \in \mathbb{R}$.
\begin{align*}
    P( \theta_{t \land \tau} - \alpha (t \land \tau) &< \lambda) \le \exp \left( \frac{-\lambda^2}{2 t u} \right) \\
    P(\theta_{t \land \tau} - \alpha (t - (t \land \tau)) \le c) &< \exp \left( -\frac{(c+ \alpha t)^2}{2 t u} \right) 
\end{align*}
The second line follows from substituting $\lambda = -\alpha t + c$. Note that the RHS goes to 0 as $t$ goes to $\infty$.

Next, we continue from the LHS. Let $\theta^*_t = \sup_{0 \le n \le t} \theta_n$
\begin{align*}
 &P(\theta_{t \land \tau} - \alpha (t - (t \land \tau)) < c) \\ 
 &\ge P(\theta_{t \land \tau} - \alpha (t - (t \land \tau)) < c, t \le \tau ) \\ 
 &\quad + P(\theta_{t \land \tau} - \alpha (t - (t \land \tau)) < c, t > \tau), \quad \text{splitting over events} \\ 
 &\ge P(\theta_{t \land \tau} < c, t< \tau), \quad \text{dropping the second term} \\
 &\ge P(\theta_{t} < c, \sup \theta_t < c, \inf \theta_t < a), \quad \text{definition of $\tau$} \\
 &= P(\sup \theta_t < c, \inf \theta_t < a), \quad \text{this event is a subset of the other} \\
 &=  P(\tau > t)
\end{align*}
Hence the probability the stopping time exceeds $t$ goes to $0$ and it is guaranteed to be finite almost surely.

Now, if $\theta_t$ exits $B$, there is some positive probability that it reached $C$. We see this by considering that taking action 1 increases $\theta$ by at least a constant, so the sequence of only taking action $1$ until $\theta_t$ reaches $C$ has positive probability. This is a lower bound on the probability of eventually reaching $C$ given that $\theta_t$ is in $B$.

Finally, we combine the results for all three regions to show that convergence happens with probability 1.
Without loss of generality, suppose $\theta_0 \in A$. If that is not the case, then keep running the process until either $\theta_t$ is in $A$ or convergence occurs. 

Let $E_i$ be the event that $\theta_t$ returns to $A$ after leaving it for the $i$-th time. Then $E_i^\complement$ is the event that $\theta_t \xrightarrow{} \infty$ (convergence occurs).
This is the case because, when $\theta_t \in C$, those are the only two options and, when $\theta_t \in B$ we had shown that the process must exit $B$ with probability 1, either landing in $A$ or $C$.

Next, we note that $P(E_i^\complement) > 0$ since, when $\theta_t$ is in $B$, the process has positive probability of reaching $C$. Finally, when $\theta_t \in C$, the process has positive probability of converging. Hence, $P(E_i^\complement) > 0$. 

To complete the argument, whenever $E_i$ occurs, then $\theta_t$ is back in $A$ and will eventually leave it almost surely. Since the process is Markov and memoryless, $E_{i+1}$ is independent of $E_i$. 
Thus, by considering a geometric distribution with a success being $E^C_i$ occurring, $E_i^C$ will eventually occur with probability 1. In other words, $\theta_t$ goes to $+\infty$.

\end{proof}

\begin{lemma}
\label{lem:prop_epsilon_0}
For the two-armed bandit with sigmoid parameterization, natural policy gradient and a perturbed minimum-variance baseline $b = 1-p_t +\epsilon$, with $\epsilon=0$, the policy converges to the optimal policy with probability 1.
\end{lemma}
\begin{proof}
    By directly writing the updates, we find that both updates are always equal to the expected natural policy gradient, so that $\theta_{t+1} = \theta_t + \alpha$ for any $\theta_t$. 
    Hence $\theta_t \xrightarrow{} \infty$ as $t \xrightarrow{} \infty$ with probability 1. 
\end{proof}

\begin{lemma}
\label{lem:prop_epsilon_01}
For the two-armed bandit with sigmoid parameterization, natural policy gradient and a perturbed minimum-variance baseline $b = 1-p_t +\epsilon$, with $\epsilon \in (0,1)$, the policy converges to the optimal policy in probability.
\end{lemma}

\begin{proof}

The overall idea is to ensure that the updates are always positive for some region $ A = \{\theta: \theta > \theta_A\}$ then show that we reach this region with probability 1. 

Recall that the possible updates when the parameter is $\theta_t$ are:
\begin{itemize}
    \item $\theta_{t+1} = \theta_t + \alpha(1-\frac{\epsilon}{\sigma(\theta_t)})$ if we choose action 1, with probability $\sigma(\theta_t)$ 
    \item $\theta_{t+1} = \theta_t + \alpha(1+\frac{\epsilon}{1- \sigma(\theta_t)})$ if we choose action 2, with probability $1-\sigma(\theta_t)$.
\end{itemize}

First, we observe that the update for action 2 is always positive. As for action 1, it is positive whenever $p \ge \epsilon$, equivalently $\theta \ge \theta_A$, where $\theta_A = \sigma^{-1}(\epsilon)$. Call this region $A = \{ \theta: \theta > \theta_A (= \sigma^{-1}(\epsilon)) \}$. \\
If $\theta_t \in A$, then we can find a $\delta > 0$ such that the update is always greater than $\delta$ in the positive direction, no matter which action is sampled. 
So, using the same argument as for the $\epsilon = 0$ case with steps of $+\delta$, we get convergence to the optimal policy (with only constant regret).

In the next part, we show that the iterates will enter the good region $A$ with probability 1 to complete the proof. We may assume that $\theta_0 < \theta_A$ since if that is not the case, we are already done. 
The overall idea is to create a transformed process which stops once it reaches $A$ and then show that the stopping time is finite with probability 1. This is done using the fact that the expected step is positive ($+\alpha$) along with Markov's inequality to bound the probability of going too far in the negative direction.

We start by considering a process equal to $\theta_t$ except it stops when it lands in $A$. Defining the stopping time $\tau = \inf \{t : \theta_t > \theta_A \}$ and ``$\land$'' by $a \land b = \min(a,b)$ for $a, b \in \mathbb{R}$, the process $\theta_{t \land \tau}$ has the desired property.

Due to the stopping condition, $\theta_{t \land \tau}$ will be bounded above and hence we can shift it in the negative direction to ensure that the values are all nonpositive. So we define $\tilde{\theta}_t = \theta_{t \land \tau} - C$ for all $t$, for some $C$ to be determined. 

Since we only stop the process $\{ \theta_{t \land \tau} \}$ \textit{after} reaching $A$, then we need to compute the largest value $\theta_{t \land \tau}$ can take after making an update which brings us inside the good region. 
In other words, we need to compute $ \sup_\theta \{\theta + \alpha(1+\frac{\epsilon}{1- \sigma(\theta)}) : \theta \in A^\complement\}$.
Fortunately, since the function to maximize is an increasing function of $\theta$, the supremum is easily obtained by choosing the largest possible $\theta$, that is $\theta = \sigma^{-1}(\epsilon)$. 
This gives us that $C = \theta_A + U_A$, where $U_A = \alpha(1 + \frac{\epsilon}{1-\epsilon})$.

All together, we have $\tilde{\theta}_t = \theta_{t \land \tau} - \theta_A - U_A$.
By construction, $\tilde{\theta}_t \le 0$ for all $t$ (note that by assumption, $\theta_0 < \theta_A$ which is equivalent to $\tilde{\theta}_0 < -U_A$ so the process starts at a negative value).

Next, we separate the expected update from the process. 
We form the nonpositive process $Y_t = \tilde{\theta}_t - \alpha (t \land \tau) = \theta_{t \land \tau} - U_A - \theta_A - \alpha (t \land \tau)$.
This is a martingale as it is a stopped version of the martingale $\{\theta_{t} - U_A - \theta_A - \alpha t \}$. 

Applying Markov's inequality, for $\lambda > 0$ we have:
\begin{align*}
P(Y_t \le -\lambda) &\le -\frac{\E[Y_t]}{\lambda} \\
P(Y_t \le -\lambda) &\le -\frac{Y_0}{\lambda}, \quad \text{since $\{Y_t\}$ is a martingale} \\
P(\theta_{\tau \land t} - \alpha (\tau \land t) - \theta_A - U_A \le -\lambda) &\le \frac{\theta_A + U_A -\theta_0}{\lambda} \\
P(\theta_{\tau \land t}  \le \alpha (\tau \land t - t) + \theta_A)  &\le \frac{\theta_A + U_A - \theta_0}{\alpha t + U_A}, \quad \text{choosing $\lambda = \alpha t + U_A$}
\end{align*}

Note that the RHS goes to 0 as $t \xrightarrow{} \infty$.
We then manipulate the LHS to eventually get an upper bound on $P(t \le \tau)$.
\begin{align*}
    & P(\theta_{\tau \land t}  \le \alpha (\tau \land t - t) + \theta_A) \\
    &=  P(\theta_{\tau \land t}  \le \alpha (\tau \land t - t) + \theta_A, t \le \tau) +  P(\theta_{\tau \land t}  \le \alpha (\tau \land t - t) + \theta_A, t > \tau), \quad  \text{splitting over disjoint events} \\
    &\ge  P(\theta_{\tau \land t}  \le \alpha (\tau \land t - t), t \le \tau),  \quad \text{second term is nonnegative} \\
    &=  P(\theta_t  \le  \theta_A, t \le \tau), \quad \text{since $t \le \tau$ in this event} \\
    &= P(\theta_t  \le  \theta_A, \sup_{0\le n \le t} \theta_n \le \theta_A), \quad \text{by definition of $\tau$} \\
    &\ge P(\sup_{0\le n \le t} \theta_n \le \theta_A), \quad \text{this event is a subset of the other }  \\
    &= P(t \le \tau) 
\end{align*}
Since the first line goes to $0$, the last line goes to $0$ and hence we have that $\theta_t$ will enter the good region with probability 1.

\end{proof}

Note that there is no contradiction with the nonconvergence result for $\eps < -1$ as we cannot use Markov's inequality to show that the probability that $\theta_t < c$ ($c > 0$) goes to 0. The argument for the $\epsilon \in (0,1) $ case relies on being able to shift the iterates $\theta_t$ sufficiently left to construct a nonpositive process $\tilde{\theta}_t $. In the case of $\epsilon < 0$, for $\theta < c$ ($c \in \mathbb{R}$), the right update ($1 - \frac{\epsilon}{\sigma(\theta)}$) is unbounded hence we cannot guarantee the process will be nonpositive.
As a sidenote, if we were to additionally clip the right update so that it is $\max(B, 1-\frac{\epsilon}{\sigma(\theta)})$ for some $B>0$ to avoid this problem, this would still not allow this approach to be used because then we would no longer have a submartingale. The expected update would be negative for $\theta$ sufficiently negative.

\begin{lemma}
\label{lem:prop_epsilon_1inf}
For the two-armed bandit with sigmoid parameterization, natural policy gradient and a perturbed minimum-variance baseline $b = 1-p_t +\epsilon$, with $\epsilon \ge 1$, we have that $P(\sup_{0 \le n \le t} \theta_n > C) \xrightarrow{} 1$ as $t \xrightarrow{} \infty$ for any $C \in \mathbb{R}$.
\end{lemma}
\begin{proof}
We follow the same argument as in the $\epsilon \in (0,1)$ case with a stopping time defined as $\tau = \inf \{t : \theta_t > c \} $ and using $\theta_A = c$, to show that
\[ P\left( \sup_{0 \le n \le t} \theta_t \le c \right)  \xrightarrow{} 0\]
\end{proof}

\subsection{Convergence with vanilla policy gradient}
\label{app:2arm_vanilla_pg}

In this section, we show that using vanilla PG on the two-armed bandit converges to the optimal policy in probability. This is shown for on-policy and off-policy sampling with importance sampling corrections.
The idea to show optimality of policy gradient will be to use Azuma's inequality to prove that $\theta_t$ will concentrate around their mean $\E[\theta_t]$, which itself converges to the right arm.

We now proceed to prove the necessary requirements.

\begin{lemma}[Bounded increments for vanilla PG] Assuming bounded rewards and a bounded baseline, the martingale $\{X_t\}$ associated with vanilla policy gradient has bounded increments
$$|X_t - X_{t-1}| \le C \alpha_t$$
\label{proposition:pg_bounded}
\end{lemma}

\begin{proof}
Then, the stochastic gradient estimate 
is \[ g_t = \left\{
                \begin{array}{l}
                  (r_1 -b) (1-p_t), \text{with probability}\ p_t, r_1 \sim P_1 \\
                  - (r_0 - b) p_t, \text{with probability}\ (1-p_t), r_0 \sim P_0
                \end{array}
              \right.
              \]
        Furthermore, $\E[g_t|\theta_0] = \E[\E[g_t | \theta_t] | \theta_0] = \E[\Delta p_t (1-p_t) | \theta_0]$.
        As the rewards are bounded, for $i=0,1$, $\exists R_i >0$ so that $|r_i| \le R_i$
        \begin{eqnarray*}
        |X_t - X_{t-1}| &=&|\sum_{i=1}^t \alpha_i (g_i-\E[g_i])- \sum_{i=1}^{t-1} \alpha_i (g_i-\E[g_i])|\\
        &=& \alpha_t |g_t - \E[\Delta p_t(1-p_t)]|\\
        &\le& \alpha_t\big(|g_t| + |\E[\Delta p_t(1-p_t)]|\big)\\
         &\le& \alpha_t\big( \max(|r_1 - b|, |r_0-b|) + |\E[\Delta p_t(1-p_t)]|\big), \quad r_1 \sim P_1, r_0 \sim P_0\\
        &\le& \alpha_t \big(\max(|R_1| +|b|, |R_0|+|b|) + \frac{\Delta}{4}\big)
        \end{eqnarray*}
        Thus   $|X_t - X_{t-1}| \le C \alpha_t$

\end{proof}
\begin{lemma}[Bounded increments with IS] Assuming bounded rewards and a bounded baseline, the martingale $\{X_t\}$ associated with policy gradient with importance sampling distribution $q$ such that $\min \{q, 1-q\} \ge \epsilon >0$ has bounded increments
$$|X_t - X_{t-1}| \le C \alpha_t$$
\end{lemma}

\begin{proof}
Let us also call $\epsilon>0$ the lowest probability of sampling an arm under $q$.

Then, the stochastic gradient estimate is 
    \[ g_t = \left\{
                \begin{array}{l}
                  \frac{(r_1 -b) p_t (1-p_t)}{q_t}, \text{with probability}\ q_t, r_1 \sim P_1 \\
                  - \frac{(r_0 - b) p_t (1-p_t)}{1-q_t}, \text{with probability}\ (1-q_t), r_0 \sim P_0
                \end{array}
              \right.
    \]
        
        As the rewards are bounded, $\exists R_i >0$ such that $|r_i| \le R_i$ for all $i$
        \begin{eqnarray*}
        |X_t - X_{t-1}| &=&|\sum_{i=1}^t \alpha_i (g_i-\E[g_i])- \sum_{i=1}^{t-1} \alpha_i (g_i-\E[g_i])| \\
        &=& \alpha_t |g_t - \E[\Delta p_t(1-p_t)]|\\
        &\le& \frac{\alpha_t  \big(\max(|R_1|+|b|, |R_0|+|b|) + \Delta \big)}{4\epsilon} \quad \text{as $q_t, 1-q_t$ $\ge \epsilon$}
        \end{eqnarray*}
        Thus   $|X_t - X_{t-1}| \le C \alpha_t$

\end{proof}

We call non-singular importance sampling any importance sampling distribution so that the probability of each action is bounded below by a strictly positive constant.
\begin{lemma} For vanilla policy gradient and policy gradient with nonsingular importance sampling, the expected parameter $\theta_t$ has infinite limit.
i.e. if $\mu_1 \neq \mu_0$,
$$\lim_{t \to +\infty} \E[ \theta_t - \theta_0] = +\infty$$
In other words, the expected parameter value converges to the optimal arm.
\end{lemma}
\begin{proof}

We reason by contradiction. The contradiction stems from the fact that on one hand we know $\theta_t$ will become arbitrarily large with $t$ with high probability as this setting satisfies the convergence conditions of stochastic optimization. On the other hand, because of Azuma's inequality, if the average $\theta_t$ were finite, we can show that $\theta_t$ cannot deviate arbitrarily far from its mean with probability 1. The contradiction will stem from the fact that the expected $\theta_t$ cannot have a finite limit.

We have $\theta_t - \theta_0 = \sum_{i=0}^{t} \alpha_i g_i$. Thus
\begin{eqnarray*}
\E[\theta_t - \theta_0 ] &=& \E[\sum_{i=0}^{t} \alpha_i g_i | \theta_0]\\
&=& \sum_{i=0}^{t} \alpha_i \E[ g_i | \theta_0] \\
&=& \sum_{i=0}^{t} \alpha_i \E[ \E[ g_i |\theta_i]| \theta_0]\quad \text{using the law of total expectations}\\
&=& \sum_{i=0}^{t} \alpha_i \E[ \Delta p_i (1-p_i) | \theta_0] 
\end{eqnarray*}
where $\Delta = \mu_1 - \mu_0 > 0$ the optimality gap between the value of the arms. As it is a sum of positive terms, its limit is either positive and finite or $+ \infty$.
\begin{enumerate} 
\item \textbf{Let us assume that $\lim_{t \to +\infty} \E[\sum_{i=0}^{t} \alpha_i g_i] = \beta > 0$}. 

As $\sum_{i=0}^\infty \alpha_i^2 = \gamma$, using 
Azuma-Hoeffing's inequality

\begin{eqnarray*}
\sP(\theta_t  \ge M) &=& \sP(\theta_t - \theta_0 - \E[\sum_{i=0}^{t} \alpha_i g_i] \ge M - \E[\sum_{i=0}^{t} \alpha_i g_i] - \theta_0)\\
&\le& \exp\big( -\frac{( M - \E[\sum_{i=0}^{t} \alpha_i g_i] - \theta_0)^2}{2\sum_{i=1}^t c_i^2}\big)\\
\end{eqnarray*}
where $c_i = \alpha_i C$ like in the proposition above.
And for $M > |\theta_0| + \beta + 2 C \sqrt{\gamma \log 2}$ we have 
\begin{eqnarray*}
\lim_{t \to +\infty} M - \E[\sum_{i=0}^{t} \alpha_i g_i] - \theta_0 &\ge&  |\theta_0| + \beta + 2 C \sqrt{\gamma \log 2} - \beta - \theta_0\\
&\ge& 2 C \sqrt{\gamma \log 2}\\
\end{eqnarray*}
As $\sum_{i=0}^\infty c_i = \gamma C^2$ , we have
$$\lim_{t \to +\infty} \frac{( M - \E[\sum_{i=0}^{t} \alpha_i g_i] - \theta_0)^2}{2\sum_{i=1}^t c_i^2} = \frac{4 C^2 \gamma \log 2}{2 \gamma C^2 } \ge 2\log 2 = \log 4 $$

Therefore
$$\lim_{t \to +\infty} \sP(\theta_t \ge M) \le \frac{1}{4}$$
By a similar reasoning, we can show that 
$$\lim_{t \to +\infty} \sP(\theta_t \le -M) \le \frac{1}{4}$$
Thus $$\lim_{t \to +\infty} \sP(|\theta_t| \le M) \ge \frac{1}{2}$$ i.e for any $M$ large enough, the probability  that $\{\theta_t\}$ is bounded by $M$ is bigger than a strictly positive constant.

\item
Because policy gradient with diminishing stepsizes satisfies the convergence conditions defined by \cite{bottou2018optimization}, we have that

$$\forall \epsilon > 0, \sP(\|\nabla J(\theta_t)\| \ge \epsilon) \le  \frac{\E[\|\nabla J(\theta_t)\|^2]}{\epsilon^2} \xrightarrow[t \to \infty]{} 0 $$
(see proof of Corollary 4.11 by \cite{bottou2018optimization}).
We also have $\|\nabla J(\theta_t)\|=\| \Delta  \sigma(\theta_t) (1-\sigma(\theta_t)) \| = \Delta  \sigma(\theta_t) (1-\sigma(\theta_t))$ for $\Delta = \mu_1 - \mu_0 > 0$ for $\mu_1$ (resp. $\mu_0$) the expected value of the optimal (res. suboptimal arm). Furthermore, $f: \theta_t \mapsto \Delta  \sigma(\theta_t) (1-\sigma(\theta_t))$ is symmetric, monotonically decreasing on $\R^+$ and takes values in $[0, \Delta  /4]$. Let's call $f^{-1}$ its inverse on $\R^+$.

We have that
$$\forall \epsilon \in [0, \Delta  /4], \ \Delta  \sigma(\theta) (1-\sigma(\theta)) \ge \epsilon \iff |\theta| \le f^{-1}(\epsilon)$$

Thus 
$\forall M > 0,$
\begin{eqnarray*}
\sP(|\theta_t| \le M) &=& \sP(\|\nabla J(\theta_t)\| \ge f(M))\\ 
&\le& \frac{\E[\|\nabla J(\theta_t)\|^2]}{( \Delta  \sigma(M) (1-\sigma(M)))^2}\\
&\xrightarrow[t \to \infty]{}& 0
\end{eqnarray*}

Here we show that $\theta_t$ cannot be bounded by any constant with non-zero probability at $t \to \infty$. This contradicts the previous conclusion.

\end{enumerate}
Therefore $\lim_{t \to +\infty} \E[\theta_t - \theta_0] = +\infty$

\end{proof}

\begin{proposition}[Optimality of stochastic policy gradient on the 2-arm bandit]
\label{prop_vpg_cv}
Policy gradient with stepsizes satisfying the Robbins-Monro conditions ($\sum_t \alpha_t = \infty, \sum_t \alpha_t^2 < \infty$) converges to the optimal arm.
\end{proposition}
Note that this convergence result addresses the stochastic version of policy gradient, which is not covered by standard results for stochastic gradient algorithms due to the nonconvexity of the objective. 
\begin{proof}
We prove the statement using Azuma's inequality again.
We can choose $\epsilon = (1-\beta) \E[\sum_{i=0}^{t}\alpha_i g_i] \ge 0$ for $\beta \in ]0,1[$.

\begin{eqnarray*}
\sP\bigg( \theta_t >  \theta_0 + \beta \E[\sum_{i=0}^{t}\alpha_i g_i]  \bigg) &=&\sP\bigg( \theta_t - \E[\sum_{i=0}^{t}\alpha_i g_i] -\theta_0 >   \beta \E[\sum_{i=0}^{t}\alpha_i g_i] -\E[\sum_{i=0}^{t}\alpha_i g_i] \bigg)  \\ 
&=& 1-\sP \bigg( \theta_t - \theta_0  -\E[\sum_{i=0}^{t}\alpha_i g_i] \le -\epsilon\bigg)\\
&=& 1-\sP \bigg( \underbrace{\theta_0  + \E[\sum_{i=0}^{t}\alpha_i g_i] - \theta_t  }_{\text{Martingale} \ X_t} \ge \epsilon\bigg)\\
&\ge& 1- \exp\bigg( -\frac{(1-\beta)^2\ \E[\sum_{i=0}^{t}\alpha_i g_i]^2}{2\sum_{i=1}^t \alpha_i^2 C^2}\bigg)\\
\end{eqnarray*}
Thus $\lim_{t \to \infty} \sP\bigg( \theta_t >  \theta_0 + \beta \E[\sum_{i=0}^{t}\alpha_i g_i]   \bigg) = 1$, as $\lim_{t\to \infty} \E[\sum_{i=0}^{t}\alpha_i g_i]  = +\infty$ and $\sum_{t=0}^\infty \alpha_t^2 < +\infty$. Therefore $\lim_{t\to\infty} \theta_t = +\infty$ almost surely.
\end{proof}

\section{Multi-armed bandit theory}
\label{app:theory_multiarm}
\threearmedbandit*

\begin{proof}
The example of convergence to a suboptimal policy for the minimum-variance baseline and convergence to the optimal policy for a gap baseline are outlined in the next two subsections.
\end{proof}

\subsection{Convergence issues with the minimum-variance baseline}
\label{app:3arm_minvar_baseline_div}

\begin{proposition}
Consider a three-armed bandit with rewards of 1, 0.7 and 0. Let the policy be parameterized by a softmax ($\pi_i \propto e^{\theta_i}$) and optimized using natural policy gradient paired with the mininum-variance baseline. 
If the policy is initialized to be uniform random, there is a nonzero probability of choosing a suboptimal action forever and converging to a suboptimal policy.
\end{proposition}

\begin{proof}
The policy probabilities are given by $\pi_i = \frac{e^\theta_i}{\sum_j e^\theta_j}$ for $i =1,2,3$. 
Note that this parameterization is invariant to shifting all $\theta_i$ by a constant. 

The natural policy gradient estimate for 

The gradient for sampling arm $i$ is given by $g_i = e_i - \pi$, where $e_i$ is the vector of zeros except for a $1$ in entry $i$. 
The Fisher information matrix can be computed to be $F = diag(\pi) - \pi \pi^T$. \\
Since $F$ is not invertible, then we can instead find the solutions to $Fx = g_i$ to obtain our updates.
Solving this system gives us $x = \lambda e + \frac{1}{\pi_i} e_i$, where $e$ is a vector of ones and $\lambda \in \R$ is a free parameter.

Next, we compute the minimum-variance baseline. Here, we have two main options. We can find the baseline that minimizes the variance of the sampled gradients $g_i$, the ``standard'' choice, or we can instead minimize the variance of the sampled \textit{natural} gradients, $F^{-1}g_i$. 
We analyze both cases separately. 

The minimum-variance baseline for gradients is given by $b^* = \frac{\E [R(\tau) || \nabla \log \pi(\tau)||^2]}{\E [|| \nabla \log \pi(\tau)||^2] } $. 
In this case, $\nabla \log \pi_i = e_i - \pi$, where $e_i$ is the $i$-th standard basis vector and $\pi$ is a vector of policy probabilities. 
Then, $|| \nabla \log \pi_i|| = (1-\pi_i)^2 + \pi_j^2 + \pi_k^2$, where $\pi_j$ and $\pi_k$ are the probabilities for the other two arms.
This gives us
\[ b^* = \frac{\sum_{i=1}^3 r_i w_i }{\sum_{i=1}^3 w_i}  \]
where $w_i = ((1-\pi_i)^2 + \pi_j^2 + \pi_k^2) \pi_i $.

The proof idea is similar to that of the two-armed bandit. Recall that the rewards for the three actions are 1, 0.7 and 0. We will show that this it is possible to choose action 2 (which is suboptimal) forever. 

To do so, it is enough to show that we make updates that increase $\theta_2$ by at least $\delta$ at every step (and leave $\theta_1$ and $\theta_3$ the same). In this way, the probability of choosing action 2 increases sufficiently fast, that we can use the proof for the two-armed bandit to show that the probability of choosing action 2 forever is nonzero.

In more detail, suppose that we have established that, at each step, $\theta_2$ increases by at least $\delta$. 
The policy starts as the uniform distribution so we can choose any initial $\theta$ as long as three components are the same ($\theta_1 = \theta_2 = \theta_3$).
Choosing the initialization $\theta_i = - \log(\nicefrac{1}{2})$ for all $i$, we see that $\pi_2 = \frac{e^{\theta_2}}{\sum_{i=1}^3 \theta_i} = \frac{e^{\theta_2}}{1 + e^{\theta_2}} = \sigma(\theta_2)$ where $\sigma(.)$ is the sigmoid function. 
Since at the $n$-th step, $\theta_2 > \theta_0 + n\delta$, we can reuse the proof for the two-armed bandit to show $Pr(\text{action 2 forever}) > 0$.

To complete the proof, we need to show that the updates are indeed lower bounded by a constant.
Every time we sample action 2, the update is $\theta \xleftarrow{} \theta + \alpha (r_2 - b^*) (\lambda e + \frac{1}{\pi_2} e_2)$. We can choose any value of $\lambda$ since they produce the same policy after an update due to the policy's invariance to a constant shift of all the parameters. We thus choose $\lambda = 0$ for simplicity. 
In summary, an update does $\theta_2 \xleftarrow{} \theta_2 + \alpha (r_2 - b^*) \frac{1}{\pi_2}$ and leaves the other parameters unchanged.

In the next part, we use induction to show the updates are lower bounded at every step. 
For the base case, we need $r_2 - b^* > \delta$ for some $\delta > 0$. Since we initialize the policy to be uniform, we can directly compute the value of $b^* \approx 0.57$, so the condition is satisfied for, say, $\delta = 0.1$.

For the inductive case, we assume that $r_2 - b^* > \delta$ for $\delta > 0$ and we will show that $r_2 - b^*_+ > \delta$ also, where $b^*_+$ is the baseline after an update. 
It suffices to show that $b^*_+ \le b^*$.

To do so, we examine the ratio $\frac{w_2}{w_1}$ in $b^*$ and show that this decreases. Let $\left( \frac{w_2}{w_1}\right)_+$ be the ratio after an update and let $c = r_2 - b^*$.
\begin{align*}
    \left(\frac{w_2}{w_1}\right) &= \frac{2 (\pi_1^2 + \pi_3^2 + \pi_1 \pi_3) \pi_2}{2 (\pi_2^2 + \pi_3^2 + \pi_2 \pi_3) \pi_1 } \\
    &= \frac{( e^{2\theta_1} + e^{2\theta_3} + e^{\theta_1 + \theta_3} ) e^{\theta_2}} 
    {( e^{2\theta_2} + e^{2\theta_3} + e^{\theta_2 + \theta_3} ) e^{\theta_1}} \\
    \left(\frac{w_2}{w_1}\right)_+ &= 
    \frac{( e^{2\theta_1} + e^{2\theta_3} + e^{\theta_1 + \theta_3} ) e^{\theta_2 + \frac{c}{\pi_2}}} 
    {( e^{2\theta_2 + 2\frac{c}{\pi_2}} + e^{2\theta_3} + e^{\theta_2 + \theta_3 + \frac{c}{\pi_2} )} e^{\theta_1}}
\end{align*}
We compare the ratio of these:
\begin{align*}
    \frac{\left(\frac{w_2}{w_1}\right)_+}{\left(\frac{w_2}{w_1}\right)} &= 
    \frac{e^{\theta_2 + \frac{c}{\pi_2}}}{e^{\theta_2}} 
    \frac{e^{2\theta_2} + e^{2\theta_3} + e^{\theta_2 + \theta_3} }
    {e^{2\theta_2 + 2\frac{c}{\pi_2}} + e^{2\theta_3} + e^{\theta_2 + \theta_3 + \frac{c}{\pi_2} }} \\
    &= \frac{e^{2\theta_2} + e^{2\theta_3} + e^{\theta_2 + \theta_3}}
    {e^{2\theta_2 + \frac{c}{\pi_2}} + e^{2\theta_3 - \frac{c}{\pi_2}} + e^{\theta_2 + \theta_3 }} \\
    &< \frac{e^{2\theta_2} + e^{2\theta_3} + e^{\theta_2 + \theta_3}}
    {e^{2\theta_2 + \delta} + e^{2\theta_3 - \delta} + e^{\theta_2 + \theta_3 }} 
\end{align*}

The last line follows by considering the function 
$f(z) = e^{x-z} + e^{y-z}$ for a fixed $x \le y$.
$f'(z) = -e^{x-z} + e^{y+z} > 0$ for all $z$, so $f(z)$ is an increasing function.
By taking $x = 2\theta_2 $ and $y = 2\theta_3$ ($\theta_2 \ge \theta_3$), along with the fact that $\frac{c}{\pi_2} > \delta$ (considering these as $z$ values), then we we see that the denominator has increased in the last line and the inequality holds.

By the same argument, recalling that $\delta > 0$, we have that the last ratio is less than $1$. Hence, $\left(\frac{w_2}{w_1}\right)_+ < \left(\frac{w_2}{w_1}\right)$. 

Returning to the baseline, $b^* = \frac{w_1 r_1 + w_2 r_2 + w3 r_3}{w_1 + w_2 + w_3} $. We see that this is a convex combination of the rewards. 
Focusing on the (normalized) weight of $r_2$:
\begin{align*}
    \frac{w_2}{w_1 + w_2 + w_3} &= \frac{w_2}{2w_1 + w_2} \\
    &= \frac{\nicefrac{w_2}{w_1}}{2 + \nicefrac{w_2}{w_1}}
\end{align*}
The first line follows since $w_1 = w_3$ and the second by dividing the numerator and denominator by $w_1$.
This is an increasing function of $\nicefrac{w_2}{w_1}$ so decreasing the ratio will decrease the normalized weight given to $r_2$.
This, in turn, increases the weight on the other two rewards equally. 
As such, since the value of the baseline is under $r_2 = 0.7$ (recall it started at $b^* \approx 0.57$) and the average of $r_1$ and $r_3$ is $0.5$, the baseline must decrease towards $0.5$.

Thus, we have shown that the gap between $r_2$ and $b^*$ remains at least $\delta$ and this completes the proof for the minimum-variance baseline of the gradients.

Next, we tackle the minimum-variance baseline for the updates.
Recall that the natural gradient updates are of the form 
$x_i = \lambda e + \frac{1}{\pi_i} e_i$ for action $i$ where $e$ is a vector of ones and $e_i$ is the $i$-th standard basis vector.

The minimum-variance baseline for updates is given by
\[b^* = \frac{\E [R_i ||x_i||^2 ] }{\E [||x_i||^2]  }\]

We have that $||x_i||^2 = 2\lambda^2 = (\lambda + \frac{1}{\pi_i})^2$.  
At this point, we have to choose which value of $\lambda$ to use since it will affect the baseline.
The minimum-norm solution is a common choice (corresponding to use of the Moore-Penrose pseudoinverse of the Fisher information instead of the inverse). We also take a look at fixed values of $\lambda$, but we find that this requires an additional assumption $3\lambda^2 < \nicefrac{1}{\pi_1^2}$.

First, we consider the minimum-norm solution. We find that the minimum-norm solution gives $\frac{2}{3\pi_i^2}$ for $\lambda = \frac{-1}{3\pi_i^2}$. 

We will reuse exactly the same argument as for the minimum-variance baseline for the gradients. The only difference is the formula for the baseline, so all we need to check is the that the ratio of the weights of the rewards decreases after one update, which implies that the baseline decreases after an update.

The baseline can be written as:
\begin{align*}
    b^* &= \frac{\sum_{i=1}^3 r_i \frac{2}{3 \pi_i^2} \pi_i}{\sum_{i=1}^3  \frac{2}{3 \pi_i^2}} \\
    &= \frac{\sum_{i=1}^3 r_i \frac{1}{\pi_i} }{\sum_{i=1}^3 \frac{1}{\pi_i} } 
\end{align*}

So we have the weights $w_i = \frac{1}{\pi_i}$ and the ratio is 

\begin{align*}
    \left( \frac{w_2}{w_1} \right) &= \frac{\pi_1}{\pi_2} \\
    &= \frac{e^{\theta_1} }{e^{\theta_2}} \\
    & = e^{\theta_1 - \theta_2}
\end{align*}

So, after an update, we get
\[ \left( \frac{w_2}{w_1} \right)_+ = e^{\theta_1 - \theta_2 - \frac{c}{\pi_2}} \]
for $c = \alpha(r_2 - b^*)$, 
which is less than the initial ratio. 
This completes the case where we use the minimum-norm update.

Finally, we deal with the case where $\lambda \in \mathbb{R}$ is a fixed constant. We don't expect this case to be very important as the minimum-norm solution is almost always chosen (the previous case).
Again, we only need to check the ratio of the weights.

The weights are given by $w_i = (2\lambda^2 + (\lambda + \frac{1}{\pi_i})^2) \pi_i$

\begin{align*}
    \left( \frac{w_2}{w_1} \right) &= \frac{(2\lambda^2 + (\lambda + \frac{1}{\pi_2})^2) \pi_2}{(2\lambda^2 + (\lambda + \frac{1}{\pi_1})^2) \pi_1} \\
    &= \frac{2\lambda^2 \pi_2 + (\lambda + \frac{1}{\pi_2})^2 \pi_2 }{
    2\lambda^2 \pi_1 + (\lambda + \frac{1}{\pi_1})^2 \pi_1 }
\end{align*}
We know that after an update $\pi_2$ will increase and $\pi_1$ will decrease. 
So, we check the partial derivative of the ratio to assess its behaviour after an update.

\begin{align*}
    \frac{d}{d\pi_1} \left( \frac{w_2}{w_1} \right) &= 
    -\frac{2\lambda^2 \pi_2 + (\lambda + \frac{1}{\pi_2})^2 \pi_2 }{
    (2\lambda^2 \pi_1 + (\lambda + \frac{1}{\pi_1})^2 \pi_1 } (3\lambda^2 - \nicefrac{1}{\pi_1^2})
\end{align*}
We need this to be an increasing function in $\pi_1$ so that a decrease in $\pi_1$ implies a decrease in the ratio. This is true when  $3\lambda^2 < \nicefrac{1}{\pi_1^2}$. So, to ensure the ratio decreases after a step, we need an additional assumption on $\lambda$ and $\pi_1$, which is that $3\lambda^2 < \nicefrac{1}{\pi_1^2}$. This is notably always satisfied for $\lambda=0$. 

\end{proof}

\subsection{Convergence with gap baselines}
\label{app:3arm_gap_baseline}
\begin{proposition}
For a three-arm bandit with deterministic rewards, choosing the baseline $b$ so that $r_1 > b > r_2$ where $r_1$ (resp. $r_2$) is the value of the optimal (resp. second best) arm, natural policy gradient converges to the best arm almost surely.
\end{proposition}
\begin{proof}

Let us define $\Delta_i = r_i-b$ which is striclty positive for $i=1$, stricly negative otherwise. Then the gradient on the parameter $\theta^i$ of arm $i$
\begin{eqnarray*}
g^i_t &=& \bm{1}_{\{A_t = i\}} \frac{\Delta_i}{\pi_t(i)}, \ i \sim \pi_t(\cdot)
\end{eqnarray*}
Its expectation is therefore
\begin{eqnarray*}
\E[\theta^i_t] &=& \alpha t \Delta_i  + \theta^i_0\\
\end{eqnarray*}

Also note that there is a nonzero probability of sampling each arm at $t=0$: $\theta_0 \in \R^3$, $\pi_0(i) > 0$.
Furthermore, $\pi_t(1) \ge \pi_0(1)$ as $\theta_1$ is increasing and $\theta_i, i>1$ decreasing because of the choice of our baseline. Indeed, the updates for arm 1 are always positive and negative for other arms.

For the martingale $X_t = \alpha \Delta_1  t + \theta^1_0 - \theta^1_t  $, we have
$$|X_t - X_{t-1}| \le \alpha \frac{\Delta_1}{\pi_0(1)}$$
thus satisfying the \textit{bounded increments} assumption of Azuma's inequality.
We can therefore show

\begin{eqnarray*}
\sP\big(\theta^1_t > \frac{\alpha \Delta_1}{2} t   + \theta^1_0\big) &=& \sP\big(\theta^1_t - \alpha \Delta_1 t  - \theta^1_0> - \frac{\alpha \Delta_1}{2} t \big) \\
&=& \sP\big(X_t < \frac{\alpha \Delta_1}{2} t \big) \\
&=& 1- \sP\big(X_t \ge \frac{\alpha \Delta_1}{2} t \big) \\
&\ge& 1-\exp\big(-\frac{(\frac{\alpha \Delta_1}{2} t)^2 \pi_0(1)^2}{2 t \alpha^2 \Delta_1^2 }\big)\\
&\ge& 1-\exp\big(-\frac{\pi_0(1)^2}{8} t\big)
\end{eqnarray*}

This shows that $\theta^1_t$ converges to $+\infty$ almost surely while the $\theta^i_t, i>1$ remain bounded by $\theta^i_0$, hence we converge to the optimal policy almost surely.

\end{proof}

\subsection{Convergence with off-policy sampling}
\label{app:3arm_off_policy}
We show that using importance sampling with a separate behaviour policy can guarantee convergence to the optimal policy for a three-armed bandit. 

Suppose we have an $n$-armed bandit where the rewards for choosing action $i$ are distributed according to $P_i$, which has finite support and expectation $r_i$. 
Assume at the $t$-th round the behaviour policy selects each action $i$ with probability $\mu_t(i)$.
Then, if we draw action $i$, the stochastic estimator for the natural policy gradient with importance sampling is equal to 
\begin{align*}
    g_t = \frac{R_i - b}{\mu_t (i)} \mathbbm{1}_{\{A_t = i\} } 
\end{align*}
with probability $\mu_t(i)$ and $R_i$ drawn from $P_i$.

We have that $\E[g_t] = r - be$, where $r$ is a vector containing elements $r_i$ and $e$ is a vector of ones. 
We let $\E[g_t] = \Delta$ for notational convenience.

By subtracting the expected updates, we define the multivariate martingale
$X_t = \theta_t - \theta_0 - \alpha \Delta t$. 
Note that the $i$-th dimension $X^{i}_t $ is a martingale for all $i$.

\begin{lemma}[Bounded increments] Suppose we have bounded rewards and a bounded baseline and a behaviour policy selecting all actions with probability at least $\epsilon_t$ at round $t$. Then, the martingale $\{X_t\}$ associated with natural policy gradient with importance sampling has bounded increments
$$|X^{i}_t - X^{i}_{t-1}| \le \frac{C \alpha}{\epsilon_t}$$
for all dimensions $i$ and some fixed constant $C$.
\end{lemma}

\begin{proof}
    The updates and $X_t$ are defined as above. 
    
    Furthermore $\E[g_t|\theta_0] = \E[\E[g_t | \theta_t] | \theta_0] = \Delta $.
    As the rewards are bounded, $\exists R_{max}>0$ such that, for all actions $i$, $|R_i| \le R_{max}$ with probability 1. 
    
    For the $i$-th dimension,
    \begin{eqnarray*}
    |X_t^i - X_{t-1}^i| &=& \alpha |g_t^i - |\Delta_i| |\\
    &\le& \alpha\big(| g_t^i| + |\Delta_i| \big)\\
     &\le& \alpha\big( \frac{|R_{max}-b|}{\epsilon_t} + |\Delta_i| \big) \\
    &\le& \alpha \frac{R_{max} + |b| + |\Delta_i| }{\epsilon_t} \quad \text{as}\ \epsilon_t \le 1\\
    \end{eqnarray*}
    
    Thus   $|X_t^i - X_{t-1}^i| \le \frac{C \alpha}{\epsilon_t}$ for all $i$.
\end{proof}

\is*

\begin{proof}
Let $r_i = \E [R_i]$, the expected reward for choosing action $i$.
Without loss of generality, we order the arms such that $r_1 > r_2 > ... > r_n$. 
Also, let $\Delta_i = r_i - b$, the expected natural gradient for arm $i$. 

Next, we choose $\delta \in (0,1)$ such that $(1-\delta) \Delta_1 > (1+\delta) \Delta_j$.
We apply Azuma's inequality to $X^1_t$, the martingale associated to the optimal action, with $\epsilon = \alpha \delta \Delta_i t$.

\begin{align*}
    \sP(\theta_t^1 \le \theta_0^1 + \alpha (1-\delta) \Delta_1 t  )
    &= \sP(\theta_t^1 - \theta_0^1 - \alpha \Delta_1 t  \le -\alpha \delta \Delta_1 t  )\\
    &\le \exp \left(- \frac{(\alpha \delta \Delta_1 t)^2 \epsilon_t^2}{2t \alpha^2 C^2}  \right) \\ 
    &= \exp \left(- \frac{ \delta^2 \Delta_1^2 }{2 C^2} t \epsilon_t^2 \right) \\
\end{align*}

Similarly, we can apply Azuma's inequality to actions $i \neq 1$ and obtain
\begin{align*}
    \sP(\theta_t^i \ge \theta_0^i + \alpha (1 + \delta) \Delta_i t  ) 
    &= \sP(\theta_t^i - \theta_0^i - \alpha \Delta_i t  \ge \alpha \delta \Delta_i t  )\\ 
    &\le \exp \left(- \frac{ \delta^2 \Delta_i^2 }{2 C^2} t \epsilon_t^2 \right)
\end{align*}

Letting $A$ be the event $\theta_t^1 \le \theta_0^1 + \alpha (1-\delta) \Delta_1 t$ and $B_i$ be the event that $\theta_t^i - \theta_0^i \ge \alpha (1 + \delta) \Delta_i t $ for $i \neq 1$, we can apply the union bound to get 
\[ \sP (A \cup B_1 \cup ... \cup B_n) \le \sum_{i=1}^n  \exp \left(- \frac{ \delta^2 \Delta_i^2 }{2 C^2} t \epsilon_t^2 \right) \]
The RHS goes to $0$ when $\sum_{t\ge 0} t \epsilon_t^2 = \infty$. 

Notice that $A^\complement$ is the event $\theta_t^1 > \theta_0^1 + \alpha (1-\delta) \Delta_1 t$ and $B^\complement$ is the event $\theta_t^i < \theta_0^i + \alpha (1 + \delta) \Delta_i t$. 
Then, inspecting the difference between $\theta^1_t$ and $\theta^i_t$, we have
\begin{align*}
    \theta_t^1 - \theta_t^i &> \theta_0^1 + \alpha (1-\delta) \Delta_1 t - (\theta_0^i + \alpha (1 + \delta) \Delta_i t) \\
    &= \theta_0^1 - \theta_0^i  + \alpha ( (1-\delta) \Delta_1 - (1 + \delta) \Delta_i) t 
\end{align*}
By our assumption on $\delta$, the term within the parenthesis is positive and hence the difference grows to infinity as $t \xrightarrow{} \infty$.
Taken together with the above probability bound, we have convergence to the optimal policy in probability.

\end{proof}

\section{Other results}
\label{app:other_results}
\subsection{Minimum-variance baselines}
\label{app:optimal_baseline}
For completeness, we include a derivation of the minimum-variance baseline for the trajectory policy gradient estimate (REINFORCE) and the state-action policy gradient estimator (with the true state-action values).

\textbf{Trajectory estimator (REINFORCE)} \\
We have that $\nabla J(\theta) = \E_{\tau \sim \pi} [R(\tau) \nabla \log \pi(\tau)] = \E_{\tau \sim \pi} [(R(\tau) - b) \nabla \log \pi(\tau)] $ and our estimator is $g = (R(\tau) -b) \nabla \log \pi (\tau)$ for a sampled $\tau$ for any fixed $b$.
Then we would like to minimize the variance: 
\begin{align*}
    Var(g) &= \E[ \|g\|^2_2] - \| \E[g] \|^2_2 \\ 
    &= \E[ \|g\|^2_2] - \| \E[(R(\tau) -b) \nabla \log \pi (\tau)] \|^2_2  \\
    &= \E[ \|g\|^2_2] - \| \E[R(\tau) \nabla \log \pi (\tau)] \|^2_2 
\end{align*}
The second equality follows since the baseline doesn't affect the bias of the estimator.
Thus, since the second term does not contain $b$, we only need to optimize the first term.

Taking the derivative with respect to $b$, we have:
\begin{align*}
    \frac{\partial}{\partial b} \E[ \|g\|^2_2] &= \frac{\partial}{\partial b} \E[ \| R(\tau) \nabla \log \pi (\tau)\|^2 - 2\cdot R(\tau) b \|\nabla \log \pi (\tau)]\|^2 + b^2  \|\nabla \log \pi (\tau)]\|^2]\\
    &= 2\ \big( b \cdot \E [\|\nabla \log \pi (\tau)]\|^2] - \E [ R(\tau) \|  \nabla \log \pi (\tau)]\|^2] \big)\\
\end{align*}
The minimum of the variance can then be obtained by finding the baseline $b^\ast$ for which the gradient is $0$, i.e

$$b^\ast = \frac{ \E [R(\tau) \| \nabla \log \pi (\tau)]\|^2] }{ \E [\| \nabla \log \pi (\tau)]\|^2] }$$

\textbf{State-action estimator (actor-critic)} \\
In this setting we assume access to the $Q$-value for each state-action pair $Q^\pi(s,a)$, in that case the update rule is $\nabla J(\theta) = \E_{s,a \sim d^\pi} [Q^\pi(s,a) \nabla \log \pi(a|s)]  = \E_{s,a \sim d^\pi} [(Q^\pi(s,a) - b(s)) \nabla \log \pi(a|s)] $ and our estimator is $g = (Q^\pi(s,a) - b(s)) \nabla \log \pi(a|s)$ for a sampled $s,a$.
We will now derive the best baseline for a given state $s$ in the same manner as above

\begin{align*}
Var(g | s) &= \E_{a \sim \pi} [\|g\|^2] - \| \E_{a \sim \pi} [g]\|^2\\
 &= \E_{a \sim \pi} [\|g\|^2] - \| \E_{a \sim \pi} [Q^\pi(s,a) \nabla \log \pi(a|s)]\|^2\\
\end{align*}
So that we only need to take into account the first term.

\begin{align*}
    \frac{\partial}{\partial b}  \E_{a \sim \pi} [\|g\|^2]&= \frac{\partial}{\partial b} \E_{a \sim \pi} [ \| Q^\pi(s,a)\nabla \log \pi (a|s))\|^2 - 2\cdot Q^\pi(s,a) b(s) \|\nabla \log \pi (a|s)]\|^2 + b(s)^2  \|\nabla \log \pi (a|s)]\|^2]\\
    &= 2\ \big( b(s) \cdot \E [\|\nabla \log \pi (a|s)]\|^2] - \E [ Q^\pi(s,a) \|  \nabla \log \pi (a|s)]\|^2] \big)\\
\end{align*}
Therefore the baseline that minimizes the variance for each state is 
$$b^\ast(s) =\frac{ \E [ Q^\pi(s,a) \|  \nabla \log \pi (a|s)]\|^2] }{ \E [ \|  \nabla \log \pi (a|s)]\|^2] \big)} $$

Note that for the natural policy gradient, the exact same derivation holds and we obtain that 
$$b^\ast(s) =\frac{ \E [ Q^\pi(s,a) \|  F^{-1}_s   \nabla \log \pi (a|s)]\|^2] }{ \E [ \| F^{-1}_s  \nabla \log \pi (a|s)]\|^2] \big)} $$ 
where $F^{-1}_s = \E_{a \sim \pi(\cdot,s)} [\nabla \log \pi (a|s) \nabla \log \pi(a|s)^\top]$

\subsection{Natural policy gradient for softmax policy in bandits}
\label{app:npg_softmax_bandit}
We derive the natural policy gradient estimator for the multi-armed bandit with softmax parameterization.

The gradient for sampling arm $i$ is given by $g_i = e_i - \pi$, where $e_i$ is the vector of zeros except for a $1$ in entry $i$. 
The Fisher information matrix can be computed to be $F = diag(\pi) - \pi \pi^T$, where $diag(\pi)$ is a diagonal matrix containing $\pi_i$ as the $i$-th diagonal entry. \\
Since $F$ is not invertible, then we can instead find the solutions to $Fx = g_i$ to obtain our updates.
Solving this system gives us $x = \lambda e + \frac{1}{\pi_i} e_i$, where $e$ is a vector of ones and $\lambda \in \R$ is a free parameter.
Since the softmax policy is invariant to the addition of a constant to all the parameters, we can choose any value for $\lambda$.

\subsection{Link between minimum variance baseline and value function}
\label{app:value_minvar_baseline}
We show here a simple link between the minimum variance baseline and the value function. While we prove this for the REINFORCE estimator, a similar relation holds for the state-action value estimator.
\begin{align*}
    b^\ast &= \frac{ \E [R(\tau) \| \nabla \log \pi (\tau)]\|^2] }{ \E [\| \nabla \log \pi (\tau)]\|^2] }\\
     &= \frac{ \E [R(\tau) \| \nabla \log \pi (\tau)]\|^2] }{ \E [\| \nabla \log \pi (\tau)]\|^2] } - V^\pi + V^\pi\\
     &=  \frac{ \E [R(\tau) \| \nabla \log \pi (\tau)]\|^2]  - \E[R(\tau)] \E [\| \nabla \log \pi (\tau)]\|^2 }{ \E [\| \nabla \log \pi (\tau)]\|^2] } + V^\pi\\
     &=  \frac{\text{Cov}\big(R(\tau\big), \| \nabla \log \pi (\tau)]\|^2)}{ \E [\| \nabla \log \pi (\tau)]\|^2] } + V^\pi\\
\end{align*}

\subsection{Variance of perturbed minimum-variance baselines} \label{app:var_perturbed_baseline}
Here, we show that the variance of the policy gradient estimator is equal for baselines $b_+ = b^* + \epsilon$ and $b_- = b^* - \epsilon$, where $\epsilon > 0$ and $b^*$ is the minimum-variance baseline.
We will use the trajectory estimator here but the same argument applies for the state-action estimator. 

We have $g = R(\tau) -b) \nabla \log \pi (\tau) $ and the variance is given by
\begin{align*}
    Var(g) &= \E[ \|g\|^2_2] - \| \E[g] \|^2_2 \\ 
    &= \E[ \|g\|^2_2] - \| \E[(R(\tau) -b) \nabla \log \pi (\tau)] \|^2_2  \\
    &= \E[ \|g\|^2_2] - \| \E[R(\tau) \nabla \log \pi (\tau)] \|^2_2 
\end{align*}
where the third line follows since the baseline does not affect the bias of the policy gradient.

Focusing on the first term:
\begin{align*}
    \E [||g||^2_2 ||] &= \E [R(\tau) -b) \nabla \log \pi (\tau) ] \\
    &= \E[ (R(\tau) -b)^2 ||\nabla \log \pi (\tau) ||^2_2 ] \\
    &= \sum_\tau  (R(\tau) -b)^2 ||\nabla \log \pi (\tau) ||^2_2 \pi (\tau) 
\end{align*}
Since $(R(\tau) -b)^2$ is a convex quadratic in $b$ and $||\nabla \log \pi (\tau)||^2_2 \pi (\tau)$ is a positive constant for a fixed $\tau$, the sum of these terms is also a convex quadratic in $b$.
Hence, it can be rewritten in vertex form $\E [||g||^2_2 ||] = a (b-b_0)^2 + k$ for some $a > 0$, $b_0, k \in \mathbb{R}$.

We see that the minimum is achieved at $b^* = b_0$ (in fact, $b_0$ is equal to the previously-derived expression for the minimum-variance baseline).
Thus, choosing baselines $b_+ = b^* + \epsilon$ or  $b_- = b^* - \epsilon$ result in identical expressions $\E [||g||^2_2 ||] = a \epsilon^2 + k$ and therefore yield identical variance. 

Note this derivation also applies for the natural policy gradient. The only change would be the substitution of $\nabla \log \pi(\tau)$ by $F^{-1} \nabla \log \pi (\tau)$ where $F = \E_{s_t \sim d_\pi, a_t \sim \pi} [\nabla \log \pi(a_t|s_t) \nabla \log \pi(a_t|s_t)^\top]$

\subsection{Baseline for natural policy gradient and softmax policies}
We show that introducing a baseline does not affect the bias of the stochastic estimate of the natural policy gradient. 
The estimator is given by 
$g = (R_i - b) F^{-1} \nabla \log \pi (a_i)$, where $F^{-1} = \E_{a \sim \pi} [\nabla \log \pi(a) \nabla \log \pi(a)^\top]$.

For a softmax policy, this is:
$ g = (R_i - b) (\frac{1}{\pi_\theta(i)}e_i + \lambda e)$, where $e_i$ is a vector containing a 1 at position $i $ and 0 otherwise, $e$ is a vector of all one and $\lambda$ is an arbitrary constant.
Checking the expectation, we see that
\begin{align*}
    \E[g] &= \E[(R_i - b) \left( \frac{1}{\pi_\theta(a_i)}e_i + \lambda e \right)] \\
    &= \E[R_i\left( \frac{1}{\pi_\theta(a_i)}e_i + \lambda e \right)] - b  \E[ \left( \frac{1}{\pi_\theta(a_i)}e_i + \lambda e \right)] \\ 
    &= \E[R_i\left( \frac{1}{\pi_\theta(a_i)}e_i + \lambda e \right)] - b (e + \lambda e) \\
\end{align*}
So the baseline only causes a constant shift in all the parameters. But for the softmax parameterization, adding a constant to all the parameters does not affect the policy, so the updates remained unbiased. In other words, we can always add a constant vector to the update to ensure the expected update to $\theta$ does not change, without changing the policy obtained after an update.

\subsection{Natural policy gradient estimator for MDPs} \label{app:npg_mdp_estimate}
In this section, we provide a detailed derivation of the natural policy gradient with $Q$-values estimate used in the MDP experiments. 

Suppose we have a policy $\pi_\theta$. Then, the (true) natural policy gradient is given by 
$u = F^{-1}(\theta) \nabla J(\theta)$
where $F(\theta) = \E_{s \sim d_{\pi_\theta}} [F_s(\theta)]$ and $F_s(\theta) = \E_{a \sim \pi} [\nabla \log \pi (a|s) \nabla \log \pi(a|s)^\top]$.
We want to approximate these quantities with trajectories gathered with the current policy.
Assuming that we have a tabular representation for the policy (one parameter for every state-action pair), our estimators for a single trajectory of experience $(s_0, a_0, r_0, ..., s_{T-1}, a_{T-1}, r_{T-1}, s_T)$ are as follows:
$\hat{F} = \frac{1}{T} \sum_{i=0}^{T-1} F(s_i)$ and $\widehat{\nabla J} = \frac{1}{T} \sum_{i=0}^{T-1} (Q_\pi (s_i,a_i) - b(s)) \nabla \log \pi (a_i|s_i)$.

Together, our estimate of the policy gradient is
\begin{align*}
    \hat{F}^{-1} \widehat{\nabla J} &= \left( \frac{1}{T} \sum_{i=0}^{T-1} F(s_i) \right)^{-1} 
    \left( \frac{1}{T} \sum_{i=0}^{T-1} (Q_\pi (s_i,a_i) - b(s)) \nabla \log \pi (a_i|s_i) \right) \\
    &= \left( \sum_{i=0}^{T-1} F(s_i) \right)^{-1} 
    \left( \sum_{i=0}^{T-1} (Q_\pi (s_i,a_i) - b(s)) \nabla \log \pi (a_i|s_i) \right)
\end{align*}
Since we have a tabular representation, $F(s_i)$ is a block diagonal matrix where each block corresponds to one state and $F(s_i)$ contains nonzero entries only for the block corresponding to state $s_i$. Hence, the sum is a block diagonal matrix with nonzero entries corresponding to the blocks of states $s_0, ..., s_{T-1}$ and we can invert the sum by inverting the blocks. 
It follows that the inverse of the sum is the sum of the inverses. 
\begin{align*}
    &= \left( \sum_{i=0}^{T-1} F(s_i)^{-1} \right)
    \left( \sum_{i=0}^{T-1} (Q_\pi (s_i,a_i) - b(s)) \nabla \log \pi (a_i|s_i) \right) \\
    &= 
    \sum_{i=0}^{T-1} (Q_\pi (s_i,a_i) - b(s))  \left( \sum_{j=0}^{T-1} F(s_j)^{-1} \right) \nabla \log \pi (a_i|s_i)
\end{align*}
Finally, we notice that $\nabla \log \pi(a_i|s_i)$ is a vector of zeros except for the entries corresponding to state $s_i$. So, $F(s_j)^{-1} \nabla \log \pi(a_i|s_i)$ is nonzero only if $i = j$ giving us our final estimator
\begin{align*}
        \hat{u} = \sum_{i=0}^{T-1} (Q_\pi (s_i,a_i) - b(s))  F(s_i)^{-1} \nabla \log \pi (a_i|s_i) .
\end{align*}

Note that this is the same as applying the natural gradient update for bandits at each sampled state $s$, where the rewards for each action is given by $Q_\pi(s,a)$.